\documentclass[11pt,reqno]{amsart}

\usepackage{amsmath,amssymb,amsthm}
\usepackage{amsaddr}
\usepackage{bm}
\usepackage{url}
\usepackage{natbib}
\usepackage{graphicx,here,comment}
\usepackage{xcolor}
\usepackage{fullpage}
\usepackage{booktabs}
\usepackage{wrapfig}
\usepackage{mathtools}
\usepackage{multirow}

\usepackage{algpseudocode}
\usepackage{algorithm}
\usepackage{bbding}

\newtheorem{theorem}{Theorem}

\newtheorem{lemma}[theorem]{Lemma}
\newtheorem{proposition}[theorem]{Proposition}
\newtheorem{assumption}{Assumption}
\theoremstyle{definition}
\newtheorem{definition}{Definition}
\newtheorem{remark}{Remark}
\newtheorem{example}{Example}

\newcommand{\R}{\mathbb{R}}
\newcommand{\N}{\mathbb{N}}

\newcommand{\Ep}{\mathbb{E}}

\renewcommand{\hat}{\widehat}
\renewcommand{\tilde}{\widetilde}

\newcommand{\argmin}{\operatornamewithlimits{argmin}}

\newcommand{\mone}{\textbf{1}}

\mathtoolsset{showonlyrefs}

\newcommand{\Z}{\mathbb{Z}}

\newcommand{\LASET}{\Lambda_\varphi(A,B)}

\usepackage{color}

\usepackage{newtxtext,newtxmath}

\title{Federated Learning with Relative Fairness}

\author{Shogo Nakakita$^1$, Tatsuya Kaneko$^1$, Shinya Takamaeda-Yamazaki$^1$, \\ Masaaki Imaizumi$^{1,2}$}

\address{$^1$The University of Tokyo, $^2$RIKEN Advanced Intelligence Project}

\date{ \today, \textit{Mail}: \url{nakakita@g.ecc.u-tokyo.ac.jp}}

\allowdisplaybreaks

\begin{document}
\maketitle

\begin{abstract}
This paper proposes a federated learning framework designed to achieve \textit{relative fairness} for clients. Traditional federated learning frameworks typically ensure absolute fairness by guaranteeing minimum performance across all client subgroups. However, this approach overlooks disparities in model performance between subgroups. The proposed framework uses a minimax problem approach to minimize relative unfairness, extending previous methods in distributionally robust optimization (DRO). A novel fairness index, based on the ratio between large and small losses among clients, is introduced, allowing the framework to assess and improve the relative fairness of trained models. Theoretical guarantees demonstrate that the framework consistently reduces unfairness. We also develop an algorithm, named \textsc{Scaff-PD-IA}, which balances communication and computational efficiency while maintaining minimax-optimal convergence rates. Empirical evaluations on real-world datasets confirm its effectiveness in maintaining model performance while reducing disparity. 
\end{abstract}

\section{Introduction}\label{sec:intro}
Federated learning is a regime of machine learning to train models with privacy preservation of data of clients \citep{konevcny2016afederated,konevcny2016bfederated,mcmahan2017communication,yang2020federated,kairouz2021advances}.
In the scheme of federated learning, models are trained just with communications regarding their weights, and thus the local data of each client are not collected into a centralized server.
This decentralized regime has advantages over the centralized one, such as privacy preservation and reduction of communication costs.
These advantages make federated learning applicable in many different fields, for example, the Internet of Things \citep{nguyen2021federated}, healthcare \citep{xu2021federated}, and finance \citep{long2020federated}.

\begin{figure}[htbp]
 \centering
 \begin{minipage}[b]{0.24\textwidth}
     \includegraphics[width=\textwidth]{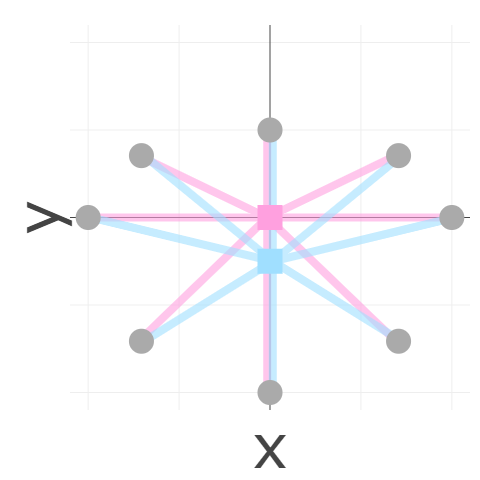}
 \end{minipage}
 \begin{minipage}[b]{0.48\textwidth}
     \includegraphics[width=\textwidth]{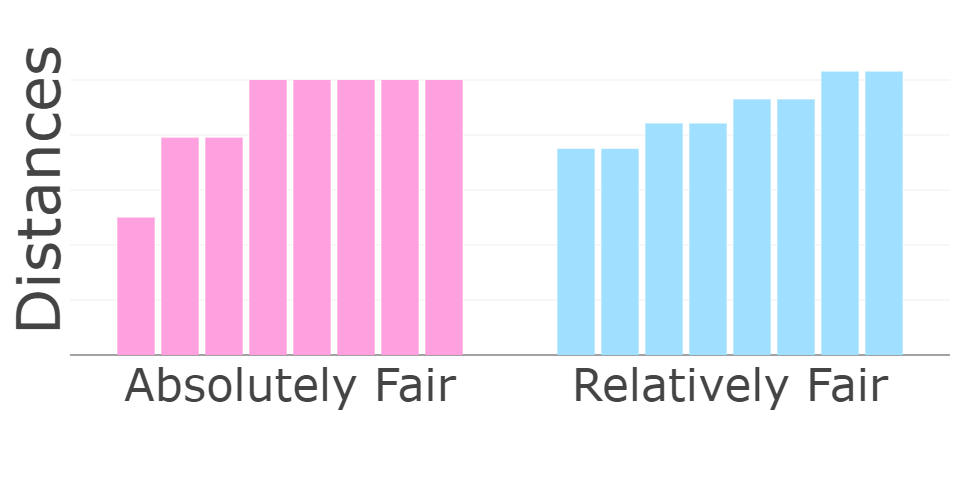}
 \end{minipage}

    \caption{
    Illustration of absolute fairness and relative fairness.
    In the left panel, distances from the eight gray points to the pink/blue center point are measured, respectively, and the right figure shows the distribution of those distances.
    While the pink distribution achieves absolute fairness which minimizes of the maximum distance, the blue distribution achieves relative fairness which minimizes the discrepancy among the distances.
    Section \ref{sec:figure_explain} shows more details.\label{fig:arc}} 
    \vskip-0.1in
\end{figure}

The fairness of trained models has received special attention in federated learning due to distributed structures.
While fairness in machine learning is multifaceted \citep{mehrabi2021survey,pessach2022review,caton2023fairness,barocas2023fairness}, some studies in federated learning take particular notes of fairness in the enjoyed performance of trained models among subgroups of clients.
This fairness is significant regardless of regime, but its achievement in federated learning is more difficult than in centralized one.
A cause of the difficulty is that the distribution of training attendees can be biased against that of the population of interest (\textit{participation gap}), which results in favoring majorities of attendees \citep{yuan2022we}.
Another cause is that empirical risks, typical loss functions in federated learning, reflect the preference of the majority even if no participation gap exists.

A significant challenge in fairness in federated learning is the control of \textit{relative fairness}.
Fairness in these previous studies means guarantees for minimal performance in any subpopulation; therefore, it is an \textit{absolute} fairness.
These guarantees are not sufficient to ensure relative fairness among subpopulations (e.g., differences or ratios of performances).
See Figure \ref{fig:arc} for the illustration of the difference.
One of the worst potential consequences of models lacking relative fairness is the digital divide, a typical failure of technology in society; therefore, methods for measuring and realizing this type of fairness are highly demanded.
While relative fairness has been the subject of studies on general machine learning \citep[for example,][]{pessach2022review,barocas2023fairness} and federated learning \citep{li2020fair}, it is still a developing issue.
A key approach to address these issues is distributionally robust optimization (DRO) as several studies \citep{mohri2019agnostic,deng2020distributionally,yu2023scaff}; however, the usual DRO is insufficient to consider \textit{relative} fairness among subpopulations.

In this study, we develop a federated learning framework to achieve relative fairness. This framework utilizes a minimax problem, which can be regarded as an extension of DRO, where the parameters are learned to minimize relative unfairness. The model trained under this framework consistently reduces relative unfairness compared to conventional methods. Additionally, we develop an algorithm for this learning process, which, through several innovations, balances stable learning with communication and computational efficiency. Moreover, the algorithm maintains a minimax-optimal convergence rate with respect to the number of updates.

Our learning framework consists of the following steps. First, we develop a generalized relative unfairness index, which is defined by the ratio of large losses to small losses within the set of clients. This index can reproduce several well-known fairness measures, such as the Atkinson index \citep{atkinson1970measurement}. Second, we develop an approximation of this relative unfairness index that can be efficiently computed by machine learning algorithms, reducing it to a type of DRO problem. Third, we develop an algorithm, \textsc{Scaff-PD-IA}, which extends the gradient update algorithms for existing DRO problems.

We present both theoretical and experimental evidence to verify the superiority of the developed framework and algorithm. First, the proposed learning framework consistently reduces relative unfairness. This reduction is demonstrated by changes in the hyperparameters involved in the algorithm. The reduction in unfairness is supported by both theoretical and experimental results. Second, the developed algorithm \textsc{Scaff-PD-IA} performs sufficiently well. Theoretically, this is shown by the the minimax optimal convergence rate of the algorithm. Experimentally, we demonstrate that the developed framework and algorithm reduce unfairness while maintaining performance through experiments using several real-world datasets.
Furthermore, we analyze the generalization error of the trained model and clarify the impact of learning through the proposed framework on the prediction performance for unseen data.

The contributions of this study are summarized as follows.

\textbf{Framework}
We propose the learning frame of DRO for handling the relative fairness by developing a measure of the relative (un)fairness. This measure is based on the relative unfairness index, which connects to several commonly used fairness measure.

\textbf{Algorithm}
We propose the algorithm \textsc{Scaff-PD-IA}, which efficiently solves the proposed framework of DRO. This algorithm employs several techniques to make the learning result stable and effective.

\textbf{Theoretical Guarantee}
There several theoretical results:
(i) we shows that our proposed framework always reduces the relative unfairness, (ii) we shows that the proposed algorithm \textsc{Scaff-PD-IA} achieves the minimax optimal convergence rate in terms of a number of updates, and (iii) we examine generalization errors in our setting combined with a supervised learning setting.

\textbf{Empirical Evaluation}
We finally analyze the empirical performance of several algorithms including \textsc{Stochastic}-AFL, DRFA, and \textsc{Scaff-PD} in fairness under several settings of problems.

\subsection{Related Works}
Distributionally robust optimization (DRO) is a classical topic in optimization theory and also gathers the interest of the statistics and machine learning communities; see \citet{rahimian2019distributionally}, \citet{shapiro2021chapter}, \citet{lin2022distributionally}, and references therein.
Although our study constrains ambiguity sets and does not introduce regularization terms, previous studies on DRO have a wide variety of settings in constraints and/or regularization on the ambiguity of distributions. 
\citet{duchi2021learning} study statistical properties of general DRO based on empirical measures and show the minimax optimality of the empirical risks under DRO problems.
Since recent datasets in machine learning are quite enormous, several studies \citep[e.g.,][]{levy2020large,qi2021online} consider large-scale optimization of DRO problems and present the convergence of algorithms theoretically and numerically.
Note that DRO plays an important role as a computational approach for out-of-distribution generalization \citep{shen2021towards}.
A bilevel federated optimization is a generalization of DRO. 
\citet{tarzanagh2022fednest} propose the \textsc{FedNest} algorithms for general bilevel optimization in the federated learning regime. 
\citet{xing2022big} consider bilevel optimization under the setting where clients are on a network graph.
\citet{qu2022generalized} and \citet{sun2023dynamic} discuss sharpness-aware optimization and its realization in federated learning.

Fairness in federated learning is multifaceted and thus has been studied via several approaches.
We first review works on fairness in the performance of models trained via federated learning among clients, which is called \textit{client-based fairness} in previous studies \citep[e.g.,][]{ezzeldin2023fairfed}.
\citet{mohri2019agnostic} study a federated minimax optimization problem with fairness and the stochastic agnostic federated learning (\textsc{Stochastic}-AFL) algorithm for the problem.
\citet{deng2020distributionally} propose the distributionally robust federated averaging (DRFA) algorithm for federated learning with DRO, which achieves a better communication efficiency than \citet{mohri2019agnostic} (they do not discuss federated learning with fairness, but their work can be easily applied).
\citet{zecchin2023communication} consider a gradient-descent-ascent algorithm with compressed communications in federated optimization to improve the complexities of communication.
\citet{yu2023scaff} propose the \textsc{Scaff-PD} algorithm for federated learning with minimal performance guarantees, which employs the idea of controlled averaging \citep{karimireddy2020scaffold} to control the heterogeneity of data distributions among clients.
While these studies are based on general DRO problems and their efficient optimization, \citet{li2020fair} discuss the $q$-Fair Federated Learning ($q$-FFL), extending a special case of \citet{mohri2019agnostic}, and show theoretical improvement of the fairness and the performance of modified algorithms ($q$-\textsc{FedSGD} and $q$-\textsc{FedAvg}).
We also note studies on other types of fairness in federated learning.
\citet{lyu2020collaborative} consider fairness in the performances of local models against the contributions of training attendees and propose a novel training scheme Collaborative Fair Federated Learning.
\citet{wang2021federated} discuss fairness in the influence on the updates of global models under the norms of the gradients of clients being heterogeneous at each step and develop the \textsc{FedFV} algorithm.
\citet{ezzeldin2023fairfed} study fairness in model outputs against inputs with demographic information whose achievement in federated learning is more difficult than in centralized learning and consider a new algorithm \textsc{FedFair}.

\subsection{Notation}
For all $m\in\N$, $[m]:=\{1,\ldots,m\}$, $[m]_{0}:=\{0,\ldots,m\}$. 
For all $x\in\R$, $\lfloor x\rfloor :=\max\{z\in\Z;x\ge z\}$ and $\lceil x\rceil :=\min\{z\in\Z;x\le z\}$.
For all $x\in\R^{m}$ with $m\in\N$, $x_{i}$ with $i\in[m]$ is the $i$-th coordinate of $x$.
For all $x,y\in\R^{m}$ with $m\in\N$, $\langle x,y\rangle =\sum_{i=1}^{n}x_{i}y_{i}$.
For any finite set $A$, $\mathrm{card}(A)$ denotes the cardinality of $A$.
For arbitrary normed space $(N,\|\cdot\|)$, for any set $A\subset N$ and positive number $\epsilon>0$, $\mathcal{N}(A,\|\cdot\|,\epsilon)$ denotes the minimum $\epsilon$-covering number, that is, $\mathcal{N}(A,\|\cdot\|,\epsilon)=\min\{\mathrm{card}(B);A\subset\bigcup_{x\in B}\{y\in B;\|x-y\|\le \epsilon\}\}$; we let $\mathcal{N}(A,\|\cdot\|,\epsilon)=\infty$ if there is no finite set $B$ with $A\subset\bigcup_{x\in B}\{y\in B;\|x-y\|\le \epsilon\}$.
For any set $A\subset \R^{m}$ with $m\in\N$, $\mathrm{int}(A)$ denotes the interior of $A$.
For any compact $C\subset\R^{m}$ with $m\in\N$ and $p\ge1$, $\|C\|_{p}:=\max_{c\in C}(\sum_{i=1}^{m}|c_{i}|^{p})^{1/p}$.
We also use the notation $D(x,y)=2^{-1}\|x-y\|_{2}^{2}$ for any $x,y\in\R^{m}$, $m\in\N$.
$\Delta^{n-1}:=\{c\in\R^{n};c_{i}\ge0,\sum_{i}c_{i}=1\}$ denotes the $(n-1)$-simplex.
$\mathcal{C}(A;[0,\infty))$ for $A\subset\R^{m}$ with $m\in\N$ is a collection of $[0,\infty)$-valued continuous functions on $A$, and $\mathcal{C}^{p}(A;[0,\infty))$ for open $A\subset\R^{m}$ is a collection of $[0,\infty)$-valued $p$-times continuously differentiable functions on $A$.
$\mone_{m}\in\R^{m}$ with $m\in\N$ is the vector whose all elements are $1$.

\section{Preliminary}

\subsection{Distribution Robust Optimization for Fair Federated Learning}
We present a distribution robust optimization (DRO) problem \citep{mohri2019agnostic,deng2020distributionally,yu2023scaff}, which is designed for federated learning. 
Let $n \in \N$ be a number of clients.
With a parameter space $\Theta\subset\R^{d}$ and a compact nonempty set $A\subset\Delta^{n-1}$ named an \textit{ambiguity set}, the DRO problem is defined as the following minimax problem:
\begin{equation}\label{eq:problem:previous}
    \min_{\theta\in\Theta}\max_{a\in A}\sum_{i=1}^{n}a_{i}f_{i}(\theta),
\end{equation}
where $\{f_{i}:i\in[n]\}\subset\mathcal{C}(\Theta;[0,\infty))$ is a sequence of loss functions of the $n$ clients.
The maximization step on the ambiguity set $A$ has a role to capture a subgroup of the $n$ clients who have a larger loss.
In other words, a situation where some clients face significant losses is negatively evaluated in the minimax form \eqref{eq:problem:previous}. 
In this sense, we can regard this optimization problem to improving the fairness among the clients.

We present typical examples of an ambiguity set $A$ designed for some specific purposes.

\begin{example}[Value-at-Risk] \label{ex:var}
    We can design a class of ambiguity sets to define conditional Value-at-Risk losses \citep{rockafellar2000optimization,rockafellar2002conditional} used in studies for fairness and DRO \citep{williamson2019fairness,duchi2021learning,pillutla2023federated,yu2023scaff}. 
Recall that  $\Delta^{n-1}:=\{c\in\R^{n};c_{i}\ge0,\sum_{i}c_{i}=1\}$ is the $(n-1)$-simplex.
With $\alpha\in(0,1]$, we can set  
\begin{align}\label{eq:ambiguity:CVaR}
    A= \Delta_{\alpha}^{n-1}:=\left\{a\in\Delta^{n-1};\ a_{i}\le ({\alpha n})^{-1},\ ^{\forall}i\in[n]\right\},
\end{align}
then the maximized value $\max_{a\in A}\sum_{i=1}^{n}a_{i}f_{i}(\theta)$ in \eqref{eq:problem:previous} approximates an average of the largest losses of $\lceil\alpha n\rceil$ clients at $\theta$.
Note that the case in which $\alpha=1/n$ gives $\max_{a\in A}\sum_{i=1}^{n}a_{i}f_{i}(\theta) = \max_{i\in[n]}f_{i}(\theta)$ is the case with $A=\Delta^{n-1}$, which is the largest loss among the $n$ clients.
Figure \ref{fig:ambiguities} contains an image of $\Delta_{0.5}^2$.
\end{example}

\subsection{Learning Algorithm for DRO}

To solve the minimax problem \eqref{eq:problem:previous}, several celebrated algorithms  have been developed \citep{nemirovskij1983problem,nemirovski2004prox,lin2020near}.
A typical algorithm is the primal-dual algorithm, which iteratively updates the parameter $\theta \in \Theta$ and the coefficient $a \in A$ with the round $r \in \N$ from an initialization $\theta^{0} \in \Theta, a^{0} \in A$ as
\begin{align}\label{eq:algorithm_dro}
    &\theta^{r+1}=\argmin_{\theta\in\Theta}\left\{\left\langle\nabla_{\theta}\sum_{i=1}^{n}a^{r}_{i}f_{i}(\theta^{r}),\theta\right\rangle +\frac{1}{2\eta_{r}}\|\theta-\theta^{r}\|_{2}^{2}\right\},\\
    &a^{r+1}=\argmin_{a\in A}\left\{\left\langle-\nabla_{a}\sum_{i=1}^{n}a^{r}_{i}f_{i}(\theta^{r}),a\right\rangle +\frac{1}{2\sigma_{r}}\|a-a^{r}\|_{2}^{2}\right\}
\end{align}
where $\eta_{r},\sigma_{r} > 0$ are learning rates.
In federated learning, however, derivation of gradients $\nabla_{\theta}f(\theta)$ is prohibitive due to computational complexities and privacy concerns.
Previous studies thus have adopted various settings of stochastic gradients instead of batch gradients.
The \textsc{Stochastic-AFL} algorithm \citep{mohri2019agnostic}, the DRFA algorithm \citep{deng2020distributionally}, and the \textsc{Scaff-PD} algorithm \citep{yu2023scaff} utilize the construction of stochastic gradients in \textsc{FedSGD}, \textsc{FedAvg}, and \textsc{Scaffold} respectively.

\section{Proposal: Learning Framework with Relative (Un)Fairness}\label{sec:proposal}

We propose a federated learning framework with relative fairness. 
As shown in Figure \ref{fig:arc}, relative fairness is a measure of fairness evaluated by the overall distribution of losses among clients. 
In this section, we first present the measure for relative fairness, and then propose a learning framework that incorporates an approximation of this measure.

\subsection{Relative Unfairness Index}

We first introduce a measure of relative fairness, named a \textit{relative unfairness index}. Although various (un)fairness measures are already known \citep{atkinson1970measurement,cobham2013all}, we consider the following generalized index. Here, both $A,B \subset \R^n$ are ambiguity sets for this measure, which will be specified later.
\begin{definition}[Relative unfairness index] \label{def:rf_index}
For loss functions $f_1(\theta),...,f_n(\theta)$ and ambiguity sets $A,B \subset \Delta^{n-1}$, a \textit{relative fairness index} is defined as follows:
\begin{equation}\label{eq:fairness}
    \mathfrak{R}_{A,B}(\theta):=\frac{\max_{a\in A}\sum_{i=1}^{n}a_{i}f_{i}(\theta)}{\min_{b\in B}\sum_{i=1}^{n}b_{i}f_{i}(\theta)}.
\end{equation}
\end{definition}
The relative unfairness index typically takes values in $[0,\infty]$. Additionally, when the ambiguity sets $A$ and $B$ are identical, it takes values in $[1,\infty]$.
The intuition behind this metric is that it calculates the ratio between large and small losses among clients. Using the Value-at-Risk (VaR) setup in Example \ref{ex:var}, the numerator $\max_{a\in A}\sum_{i=1}^{n}a_{i}f_{i}(\theta)$ is the average loss of the top 10\% of clients with the largest losses, and the denominator $\min_{b\in B}\sum_{i=1}^{n}b_{i}f_{i}(\theta)$ is the average loss of the bottom 10\% of clients with the smallest losses. Since this metric takes the form of a ratio, it has the advantage of being unaffected by the scale of the losses $\{f_i(\theta)\}_i$.

The relative unfairness index can represent several commonly used (un)fairness measures in welfare economics and social choice theory. When \(A\) and \(B\) are set as \(\Delta_{0.2}^{n-1}\), the relative unfairness index coincides with the $20:20$ index \citep{cobham2013all}, which is defined as the ratio of the average losses of the top 20\% and bottom 20\%. Additionally, when we set \(A=\Delta_{0.1}^{n-1}\) and \(B=\Delta_{0.4}^{n-1}\), the relative unfairness index corresponds to the Palma index \citep{cobham2013all}, which is the ratio of the top 10\% and bottom 40\% of losses. 
When we set \(A=\{a \in \{0,1\}^n, \|a\|=1\}\) and \(B=\{(1/n,...,1/n)\}\), the Atkinson measure $1- ({\min_{i}f_{i}(\theta)}) / ({\frac{1}{n} \sum_{i=1}^{n}f_{i}(\theta)})$ \citep{atkinson1970measurement}, which is calculated from the ratio of the minimum loss to the average, can be reproduced.
We put some more details in Section \ref{sec:compare_fairness}.

\begin{example}[Relative fairness index with VaR]\label{ex:proposal:design:pair}
We consider $\Delta_{\alpha}^{n-1}$ again and a generalization of conditional Value-at-Risk losses as Example \ref{ex:var}, which gives a clear motivation to our problem setting.
If $A=\Delta_{\alpha}^{n-1}$ and $B=\Delta_{\beta}^{n-1}$ with $\alpha,\beta\in(0,1)$, the relative fair index \eqref{eq:fairness} is approximately equivalent to the following value:
\begin{equation*}
    \max_{\sigma\in S_{n}}\frac{\frac{1}{\lfloor \alpha n\rfloor}\sum_{i=1}^{\lfloor \alpha n\rfloor}f_{\sigma(i)}\left(\theta\right)}{\frac{1}{\lfloor \beta n\rfloor}\sum_{i=n-\lfloor \beta n\rfloor+1}^{n}f_{\sigma(i)}\left(\theta\right)},
\end{equation*}
where $S_{n}$ is the symmetric group on $[n]$.
A permutation $\sigma\in S_{n}$ sorting losses $\{f_{i}(\theta);i\in[n]\}$ in descending order attains the maximum value. 
Therefore, this loss function is the difference between the empirical mean of $\lfloor \alpha n\rfloor$ clients with the larger losses and $\lfloor \beta n\rfloor$ clients with small losses.
We easily notice that $\mathfrak{R}_{A,B}$ with $A=\Delta_{\alpha}^{n-1},B=\Delta_{\beta}^{n-1}$ recovers the 20:20 ratio, the Palma ratio, and the Atkinson inequality measure with infinite aversion if $(\alpha,\beta)=(0.20,0.20)$, $(\alpha,\beta)=(0.40,0.10)$, and $(\alpha,\beta)=(1,1/n)$ respectively.
\end{example}

\subsubsection{Relative Fairness Discrepancy by Majorization Approximation}\label{sec:proposal:scalarized}

We extend the aforementioned relative fairness index in Definition \ref{def:rf_index} and develop a measure of relative fairness expressed in the form of a difference. 
The original relative fairness index is written in the form of a ratio, which introduces computational instability when optimizing it in learning algorithms. 
To address this issue is crucial to enhance computational usefulness of the relative fairness in practical use.

To solve this problem from the ratio form, we perform an approximation using the majorization technique \citep{marshall1979inequalities}. Specifically, through straightforward calculations (see Proposition \ref{prop:ratiobounds} in Appendix), we obtain the following inequalities for any $\varphi\in[0,1)$ and $\theta \in \Theta$:
\begin{align*}
    &\varphi+\frac{1}{C_b}\left(\max_{a\in A}\sum_{i=1}^{n}a_{i}f_{i}(\theta)-\varphi\min_{b\in B}\sum_{i=1}^{n}b_{i}f_{i}(\theta)\right) \le \mathfrak{R}_{A,B}(\theta)\le \varphi+\frac{1}{C_b'}\left(\max_{a\in A}\sum_{i=1}^{n}a_{i}f_{i}(\theta)-\varphi\min_{b\in B}\sum_{i=1}^{n}b_{i}f_{i}(\theta)\right),
\end{align*}
where $C_b := \max_{\theta\in\Theta}\min_{b\in B}\sum_{i=1}^{n}b_{i}f_{i}(\theta)$ and $C_b' := \min_{\theta\in\Theta,b\in B}\min_{b\in B}\sum_{i=1}^{n}b_{i}f_{i}(\theta)$ are constants which are assumed to be positive.
In words, the ratio-form relative fairness index is equivalent to the difference between the denominator and the numerator, up to a constant.
Here, $\varphi$ is thus a coefficient to control the tightness of the approximation above, the lower bound (resp.~the upper bound) is non-decreasing (resp.~the lower bound) in $\varphi$ for each $\theta$; see Proposition \ref{prop:ratiobounds} in Appendix.

Using this majorization approximation, we define the following discrepancy:
\begin{definition}[Relative fairness discrepancy] \label{def:rf_discrepancy}
For loss functions $f_1(\theta),...,f_n(\theta)$, ambiguity sets $A,B \subset \R^n$, and a coefficient $\varphi \in [0,1)$, a \textit{relative fairness discrepancy} is defined as follows:
\begin{align}
    \max_{a\in A}\sum_{i=1}^{n}a_{i}f_{i}(\theta)-\varphi \min_{b\in B}\sum_{i=1}^{n}b_{i}f_{i}(\theta) \label{eq:rf_discrepancy}
\end{align}
\end{definition}

This discrepancy, being in the form of a difference, allows for easy computation by many machine learning algorithms, such as gradient descent. This enables the construction of a useful learning framework.

\subsection{Proposed Framework} \label{sec:proposed_framework}

We use the proposed discrepancy \eqref{eq:rf_discrepancy} in Definition \ref{def:rf_discrepancy} to employ a learning framework for relative fairness. 
Since the relative fairness discrepancy consists of a maximization problem over $A$ and a minimization problem over $B$, this can be uniformly represented as a maximization problem over a certain set. 
Specifically, we develop an \textit{integrated ambiguity set} $\LASET\subset\R^{n}$:
\begin{align}
    \LASET :=\left\{\lambda\in\R^{n};\ ^{\exists}(a,b)\in A\times B,\ \lambda=\frac{1}{1-\varphi}\left(a-\varphi b\right)\right\}. \label{eq:large_a_set}
\end{align}
$\LASET$ is a subset of the hyperplane $\{h\in\R^{n}:\langle h,\mone_n\rangle=1\}$, but not necessarily a subset of the simplex $\Delta^{n-1}$; in this sense, $\LASET$ is possibly irregular in comparison to regular ambiguity sets being subsets of $\Delta^{n-1}$.
The term $\frac{1}{1-\varphi}$ is for normalizing its scale.

\begin{figure}[htbp]
    \centering
    \includegraphics[width=.49\linewidth]{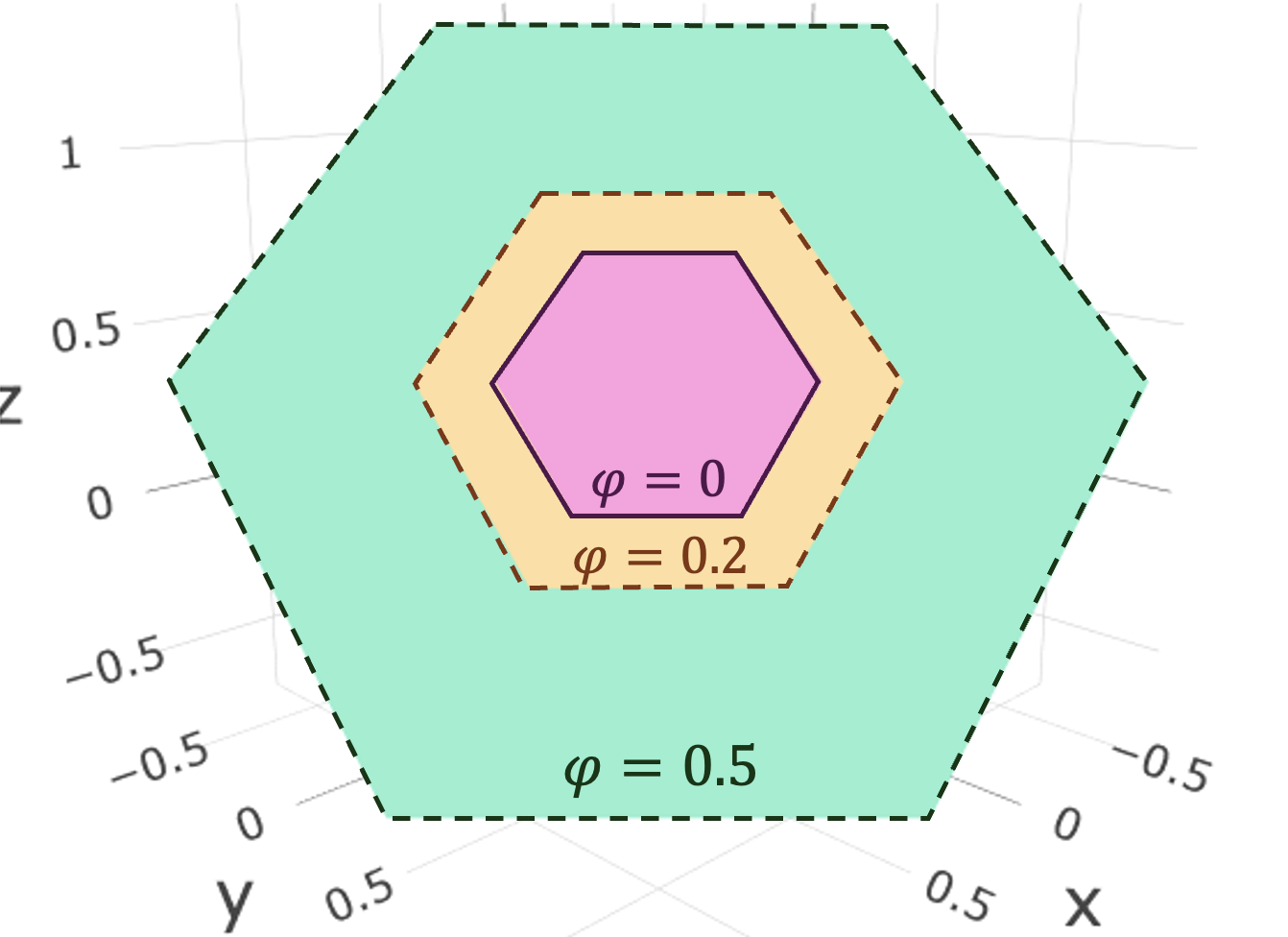}
    \caption{
    Plots of $\Lambda_{\varphi}(\Delta_{0.5}^{2},\Delta_{0.5}^{2})$, with $\varphi= 0$  (pink), $\varphi = 0.2$ (orange), and $\varphi = 0.5$ (green).  The set with solid lines is a regular ambiguity set ($\Lambda_{0}(\Delta_{0.5}^{2},\Delta_{0.5}^{2})=\Delta_{0.5}^{2}$), and with dashed lines is an integrated set.}    
    \label{fig:ambiguities}
\end{figure}

Finally, we propose our framework for \textit{federated learning with relative fairness}, using the relative fairness discrepancy \eqref{eq:rf_discrepancy} associated with the unified ambiguity set $\LASET$ in \eqref{eq:large_a_set}.
For ambiguity sets $A,B\subset\Delta^{n-1}$ and the coefficient $\varphi\in[0,1)$, we consider the following minimax problem:
\begin{equation}\label{eq:problem:scalarized}
    \min_{\theta\in\Theta}\max_{\lambda \in \LASET} \sum_{i=1}^{n} \lambda_i f_{i}(\theta).
\end{equation}
We observe that the problem \eqref{eq:problem:scalarized} is an extension of the DRO problem \eqref{eq:problem:previous} by replacing the ambiguity set $A$ with its united version $\LASET$.
Especially, if $\varphi=0$ holds, these problems are identical.
The ambiguity set $B$ is used to capture the effect of a subset of clients with smaller loss values, and $\varphi$ controls the intensity of the discrepancy between clients with larger losses and smaller losses.

We introduce new notation. 
We let $(\theta_\varphi^{\star},\lambda_\varphi^{\star})\in\Theta\times\LASET$ denote an arbitrary \textit{solution} of the minimax problem \eqref{eq:problem:scalarized}, that is, an element satisfying (i) $\sum_{i}\lambda_{\varphi,i}^{\star}f_{i}(\theta_\varphi^{\star})=\min_{\theta\in\Theta}\max_{\lambda\in\LASET}\sum_{i=1}^{n}\lambda_{\varphi,i}f_{i}(\theta)$ and (ii) for any $\theta\in\Theta,\lambda\in\LASET$, $\sum_{i}\lambda_{i}f_{i}(\theta^{\star})\le \sum_{i}\lambda_{\varphi,i}^{\star}f_{i}(\theta_\varphi^{\star})\le \sum_{i}\lambda_{\varphi,i}^{\star}f_{i}(\theta)$.
Note that they exist due to the compactness of $\Theta$ and $\LASET$ and the continuity of $f_{i}$ for all $i\in[n]$ (if necessary, see Lemma \ref{lem:existence} in Appendix \ref{appendix:technical}).
We also define a tuple $(\theta_{\varphi}^{\star},a_{\varphi}^{\star},b_{\varphi}^{\star})\in\Theta\times A\times B$ such that
        $\sum_{i=1}^{n}a_{i,\varphi}^{\star}f_{i}(\theta_{\varphi}^{\star})-\varphi\sum_{i=1}^{n}b_{i,\varphi}^{\star}f_{i}(\theta_{\varphi}^{\star})=\min_{\theta\in\Theta}\max_{a\in A,b\in B}\left(\sum_{i=1}^{n}a_{i}f_{i}(\theta)-\varphi\sum_{i=1}^{n}b_{i}f_{i}(\theta)\right)$ hold.

\section{Theory on Fairness Improvement}\label{sec:mono}

In this section, we show that the solutions of our optimization problem \eqref{eq:problem:scalarized} with $\varphi>0$ improves the relative fairness from the solutions of the original DRO problem \eqref{eq:problem:previous}, specifically, it makes the relative fairness index $\mathfrak{R}_{A,B}(\cdot)$ smaller.
Precisely speaking, we present that $\mathfrak{R}_{A,B}(\cdot)$ is non-increasing in $\varphi$ under minimal assumptions (Section \ref{sec:mono:nonstrict}) and strictly decreasing in $\varphi$ under additional mild conditions (Section \ref{sec:mono:strict}).
Recall that $\varphi = 0$ makes the original DRO problem and our problem \eqref{eq:problem:scalarized} identical.

\subsection{Weak Improvement}\label{sec:mono:nonstrict}

We first demonstrate that weak improvement, specifically, increasing \(\varphi\) does not lead to an increase in relative fairness index $\mathfrak{R}_{A,B}(\theta_{\varphi}^{\star})$.
To avoid division by zero, let us set the following assumption.
\begin{assumption}\label{assump:nozerodiv}
    For all $\theta\in\Theta$ and $b\in B$, $\sum_{i=1}^{n}b_{i}f_{i}(\theta)>0$. 
\end{assumption}
The following theorem obtained just by the compactness of $A$ and $B$ and the continuity of $f_{i}$.
\begin{theorem}\label{thm:mono:nonstrict}
    Under Assumption \ref{assump:nozerodiv}, $\mathfrak{R}_{A,B}(\theta_{\varphi}^{\star})$ is non-increasing in $\varphi\in[0,1)$, that is, we have $\mathfrak{R}_{A,B}(\theta_{\varphi}^{\star})\le \mathfrak{R}_{A,B}(\theta_{0}^{\star})$ for any $\varphi\in (0,1)$.
\end{theorem}

From this result, we have the motivation to use large $\varphi$ to let $\mathfrak{R}_{A,B}$ small owing to the monotonicity. 
    Note that $\mathfrak{R}_{A,B}(\theta_{\varphi}^{\star})$ is not strictly decreasing in $\varphi$ necessarily; we need further assumptions to derive it.
    A trivial edge case is where $\Theta$ is a singleton, which is not excluded by our assumptions on $\Theta$ (nonemptiness and compactness).
    A more intuitive edge case is where $f_{i}=f_{1}$ for all $i\in[n]$ and then $\mathfrak{R}_{A,B}(\theta_{\varphi}^{\star})=1$ for any $\varphi\in[0,1)$.
    These edge cases motivate us to consider more nontrivial structures for analysis on fairness improvement below.

\subsection{Strong Improvement}\label{sec:mono:strict}
In this section, we demonstrate the stronger result that the weak improvement: the strict monotonicity of $\mathfrak{R}_{A,B}(\theta_{\varphi}^{\star})$ in $\varphi$ similar to \citet{li2020fair}.
We exhibit a sufficient condition such that $\mathfrak{R}_{A,B}(\theta_{\varphi}^{\star})$ is strictly decreasing in $\varphi\in[0,
\overline{\varphi})$ for some $\overline{\varphi}\in(0,1]$.
We set the following ``imbalance'' condition on all the solutions of the problem \eqref{eq:problem:scalarized}.
\begin{assumption}\label{assump:imbalance}
    For some $\overline{\varphi}\in(0,1]$, for any $\varphi\in[0,\overline{\varphi})$, the following hold: (i) $\mathrm{int}(\Theta)\neq\varnothing$ and $\theta_{\varphi}^{\star}\in\mathrm{int}(\Theta)$ hold, (ii) $f_{i}\in \mathcal{C}(\Theta;[0,\infty))\cap\mathcal{C}^{1}(\mathrm{int}(\Theta);[0,\infty))$ holds, and (iii) at least one of the following inequalities holds:
    \begin{equation*}
        \sum_{i=1}^{n}a_{i,\varphi}^{\star}\left.\nabla f_{i}(\theta)\right|_{\theta=\theta_{\varphi}^{\star}}\neq\mathbf{0},\mbox{~~or~~}
        \sum_{i=1}^{n}b_{i,\varphi}^{\star}\left.\nabla f_{i}(\theta)\right|_{\theta=\theta_{\varphi}^{\star}}\neq\mathbf{0}.
    \end{equation*}
\end{assumption}
Assumption \ref{assump:imbalance} excludes cases in which there exist excessively stable solutions; for example, it successfully excludes both of the edge cases where $\Theta$ is a singleton and $f_{i}=f_{1}$ for all $i\in[n]$ in which $\mathfrak{R}_{A,B}(\theta_{\varphi}^{\star})=1$ for all $\varphi\in[0,1)$.
If $n=2$, $A=B=\Delta^{n-1}=\Delta^{1}$, and $f_{i}$ are strongly convex functions whose global minimum points are in $\mathrm{int}(\Theta)$, then Assumption \ref{assump:imbalance} indicates that the global minimum points of $f_{1},f_{2}$ are never identical.
For the case in which $n=2m$ with $m\in\N$, $A=B=\Delta_{.50}^{n-1}$, $f_{i}$ are strongly convex, and their global minimum points are in $\mathrm{int}(\Theta)$, Assumption \ref{assump:imbalance} requires that the global minimum point of the empirical risk $\sum_{i=1}^{n}f_{i}(\theta)$ is not a solution of the problem \eqref{eq:problem:scalarized} for any $\varphi\in[0,\overline{\varphi})$.
$\theta_{\varphi}^{\star}\in\mathrm{int}(\Theta)$ should hold if we consider only small $\overline{\varphi}$, which lets optimization algorithms converge.

\begin{theorem}\label{thm:mono:strict}
    Under { Assumptions \ref{assump:nozerodiv} and \ref{assump:imbalance}},
    $\mathfrak{R}_{A,B}(\theta_{\varphi}^{\star})$ is strictly decreasing in $\varphi\in[0,\overline{\varphi})$.
\end{theorem}

Therefore, if the heterogeneity (e.g., non-i.i.d.) of loss functions $\{f_{i}\}$ is sufficiently large (so that Assumption \ref{assump:imbalance} holds), then we only need to let $\varphi$ be positive for better fairness than $\theta_{0}^{\star}$ in previous studies.

\subsection{Simulation Study}\label{sec:mono:sim}

\begin{figure}[htbp]
    \centering
    \includegraphics[width=0.5\hsize]{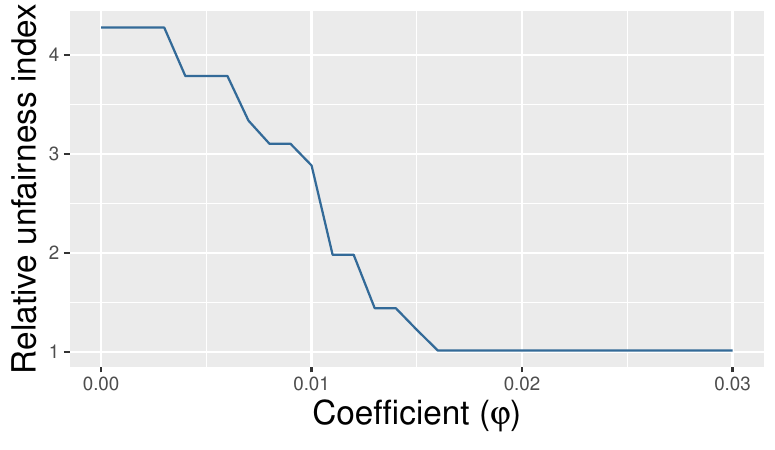}
    \caption{Relative unfairness indices in linear regression for the \texttt{penguins} data with three clients defined by \texttt{Adelie}, \texttt{Chinstrap}, and \texttt{Gentoo}.
    \label{fig:penguin:fairness}}
\end{figure}

We validate the fairness improvement by simulation.
We give a detailed explanation of the setting in Appendix; see Section \ref{sec:mono:sim:detail}.
We analyze a simple regression model with three clients and evaluate three quadratic loss functions of the regression parameter $\theta$ for datasets separated according to the three classes:
$f_{i}(\theta)=\sum_{j=10(i-1)+1}^{10i}(y_{j}-\theta_{1}-\theta_{2}(x_{j})_{(1)}+\theta_{3}(x_{j})_{2})^{2}$ with $i=1,2,3$ ($f_{1}$, $f_{2}$, and $f_{3}$ represent the losses of clients).
Setting $\Delta^{3}=A=B$, we exactly obtain the solutions of \eqref{eq:problem:scalarized} for all $\varphi\in\{0,0.01,\ldots,0.05\}$ and Figure \ref{fig:penguin:fairness} represents the unfairness measure \eqref{eq:fairness} of those solutions.
We observe that the unfairness measure decreases as $\varphi$ increases, which coincides with the conclusion of our theoretical analysis.
Note that the estimator of previous studies on fairness via DRO corresponds to the leftmost point of the figure; therefore, our proposal indeed gives better relative fairness in this case.

\section{Algorithm}\label{sec:convergence}

To solve the proposed problem \eqref{eq:problem:scalarized}, we propose an algorithm, named \textsc{Scaff-PD-IA{$(A,B,\varphi)$}} (\textsc{Scaff-PD} with integrated ambiguity; {we refer to it as \textsc{Scaff-PD-IA} in short when $A$, $B$, and $\varphi$ are obvious from contexts}), which is an extension of the algorithm \textsc{Scaff-PD} by \citet{yu2023scaff} (we do not introduce regularization of $\lambda$ that appears in the work of \citealp{yu2023scaff}; however, in the derivation, we employ it to let the proof be easy to compare to the original proof of \citealp{yu2023scaff}).
\textsc{Scaff-PD-IA} efficiently solves the DRO problem with an integrated ambiguity set  $\LASET$. 
While \textsc{Scaff-PD} allows nonnegative weights on client losses, \textsc{Scaff-PD-IA} deals with possibly negative weights and we give a convergence guarantee for such nontrivial settings.

\subsection{Algorithm Design}

While \textsc{Scaff-PD-IA} is basically an algorithm analogous to the gradient-descent-ascent method displayed in \eqref{eq:algorithm_dro}, a very typical solver for minimax problems, it has two representative features:
(i) the controlled averaging to stabilize the  optimization, and (ii) the local stochastic update approach to make communication and computation be efficient.
We show the pseudo-codes of \textsc{Scaff-PD-IA} and the local update in Algorithms \ref{alg:scaffpd} and \ref{alg:local}.

\paragraph{Controlled Averaging.}
To stabilize optimization, we adopt the controlled averaging strategy of \textsc{Scaffold} \citep{karimireddy2020scaffold} and \textsc{Scaff-PD} \citep{yu2023scaff}.
Controlled averaging is a variance reduction method such that for each round $r \in \N$, we derive  the globally desirable direction 
\begin{align}
    c^{r}=\sum_{i=1}^{n}\lambda_{i}^{r+1}\nabla f_{i}(\theta^{r}),
\end{align}
where $\theta^{r}$ and $\lambda^{r}$ are primal and dual variables in \textsc{Scaff-PD-IA} at $r$-th round.
This direction $c^{r}$ is used to de-bias client-level local updates using $\nabla f_{i}(\theta^{r})$.
This is effective in addressing the heterogeneity of client losses.
The variance reduction is essential in solving our optimization problem, since some $\lambda \in \LASET$ does not belong to $\Delta^{n-1}$ and it makes the gradient-based algorithms without unstable.

\paragraph{Local Stochastic Updates.}
Another feature of \textsc{Scaff-PD-IA} is local stochastic updates for each client.
This approach has been used in  \textsc{FedAvg}, and \textsc{Scaffold} and \textsc{Scaff-PD} as well.
In federated learning, we frequently need to reduce communication costs to make the optimization algorithm feasible.
Frequent access to all the available data should be very prohibitive for some clients, and stochastic local updates (e.g., mini-batch local updates) can address this problem.

We present the local stochastic update with some additional notations.
Specifically, we consider the following concrete structures of the loss function $f_{i}$: $f_{i}\in\mathcal{C}^{1}(\R^{d};[0,\infty))$ for all $i\in[n]$; and for a measurable space $(\mathbb{X},\mathcal{X})$, for some $f_{i}:(\R^{d},\mathcal{B}(\R^{d}))\times (\mathbb{X},\mathcal{X}) \to([0,\infty),\mathcal{B}([0,\infty))) $ and a random variable $\xi_{i}$ on $(\mathbb{X},\mathcal{X})$, we assume to have the form
\begin{equation*}
    f_{i}(\theta)=\Ep[f_{i}(\theta,\xi_{i})],~\mbox{and}~ \nabla f_{i}(\theta)=\Ep[\nabla_{\theta} f_{i}(\theta,\xi_{i})].
\end{equation*}
Moreover, we set the number of local steps $J\ge 2$ and a triple array of random variables $\{\xi_{i,j}^{r}:i\in[n],j\in[J]_{0},r\in\N\}$ and assume that $\{\xi_{i,j}^{r}:j\in[J]_{0},r\in\N\}\sim^{\mathrm{i.i.d.}}\mathcal{L}(\xi_{i})$ and it is independent of $\{\xi_{i',j}^{r}:j\in[J]_{0},r\in\N\}$ if $i\neq i'$.
Using the setup model, the local stochastic update adjusts the parameters in the form described in Algorithm \ref{alg:local}.

\begin{algorithm}
\caption{\textsc{Scaff-PD-IA}$(A,B,\varphi)$}\label{alg:scaffpd}
        \begin{algorithmic}
        \Require Set model parameters $(\theta^{1},\lambda^{1})=(\theta^{0},\lambda^{0})$
        \For{$r \gets 1$ to $R$}
            \State Set hyperparameters $\{\tau_{r},\sigma_{r},\varsigma_{r},\eta_{r}\}$ 
            \For{$i \gets 1$ to $n$}
                \State $L_{i}^{r}\gets f_{i}(\theta^{r})$, $c_{i}^{r}\gets\nabla f_{i}(\theta^{r}, \xi_{i,0})$, send $(L_{i}^{r},c_{i}^{r})$ to the server
            \EndFor
            \State $s^{r}\gets(1+\varsigma_{r})(L_{1}^{r},\ldots,L_{n}^{r})-\varsigma_{r}(L_{1}^{r-1},\ldots,L_{n}^{r-1})$
            \State $(a^{r+1},b^{r+1})\gets\argmin_{(a,b)\in A\times B}\{-\langle s^{r},\frac{1}{1-\varphi}(a-\varphi b)\rangle +\frac{1}{2\sigma_{r}}\|\frac{1}{1-\varphi}(a-\varphi b)-\lambda^{r}\|_{2}^{2}\}$ 
            \State $\lambda^{r+1}\gets\frac{1}{1-\varphi}(a^{r+1}-\varphi b^{r+1})$ \Comment{Update the dual variable $\lambda$}
            \State $c^{r}\gets\sum_{i=1}^{n}\lambda_{i}^{r+1}c_{i}^{r}$, send $c^{r}$ to each client \Comment{Compute the controlled averaging}
            \For{$i \gets 1$ to $n$}
                \State $\Delta u_{i}^{r}\gets\textsc{Local-stochastic-update}(\theta^{r},c_{i}^{r},c^{r},r)$, send $\Delta u_{i}^{r}$ to the server
            \EndFor
            \State $\theta^{r+1}\gets\argmin_{\theta\in\Theta}\{\langle\sum_{i=1}^{N}\lambda_{i}^{r+1}\Delta u_{i}^{r},\theta\rangle+\frac{1}{2\tau_{r}}\|\theta-\theta^{r}\|_{2}^{2}\}$ \Comment{Update the primal variable $\theta$}
        \EndFor
        \Ensure $(\theta^{R+1},\lambda^{R+1})$
    \end{algorithmic}
    \end{algorithm}
    
\begin{algorithm}
    \caption{\textsc{Local-stochastic-update}$(\theta,c_i,c,r)$ \citep{yu2023scaff}}\label{alg:local}
    \begin{algorithmic}
        \Require Set optimization hyperparameter $(\eta_{r},J)$ and model parameters $(\theta,c_{i},c)$
        \State $u_{i,0}=\theta$
        \For{$j \gets 1$ to $J$}
            \State $u_{i,j}\gets u_{i,j-1}-\eta_{r}(\nabla f_{i}(u_{i,j-1},\xi_{i,j})-c_{i}+c)$ \Comment{SGD with controlled averaging}
        \EndFor
        \State $\Delta u_{i}\gets (\theta-u_{i,J})/(\eta_{r}J)$
        \Ensure $\Delta u_{i}$
    \end{algorithmic}

\end{algorithm}

\subsection{Theoretical Analysis}

We perform a theoretical analysis of the convergence of the proposed algorithm \textsc{Scaff-PD-IA}. In particular, we study the convergence speed of the parameters that converge to the optimal value in terms of the number of updates. For the theoretical analysis, we follow the settings of existing studies \citep{yu2023scaff,deng2020distributionally} and consider the case where the loss function is strongly convex.

\subsubsection{Assumptions}
Let us set the following assumptions on the loss functions $f_i(\theta)$ and its stochastic variant $f_i(\theta,\xi_{i,j})$.
\begin{assumption}\label{assump:sets}
    $\Theta\subset\R^{d}$ and $\LASET\subset\{h\in\R^{n};\langle h,\mone_n\rangle=1\}$ are convex.
\end{assumption}

\begin{assumption}[uniform smoothness and convexity of $\{f_{i}\}$]\label{assump:risk}
$\{f_{i}:i\in[n]\}\subset\mathcal{C}^{1}(\R^{d};[0,\infty))$. Moreover,
    for some $M_{f},m_{f}>0$, for any $i\in[n]$, $x,y\in\R^{d}$, it holds that
    \begin{equation*}
        \frac{m_{f}}{2}\left\|x-y\right\|_{2}^{2}\le f_{i}\left(x\right)-f_{i}\left(y\right)-\left\langle \nabla f_{i}\left(y\right),x-y\right\rangle\le\frac{M_{f}}{2}\left\|x-y\right\|_{2}^{2}.
    \end{equation*}
\end{assumption}

\begin{assumption}[stochastic gradients]\label{assump:sg}
    For some $\delta_{g}\ge 0$, for all $i\in[n]$, $j\in [J]_{0}$, $r\in\N$, and $\theta\in\R^{d}$,
    \begin{equation*}
        \Ep\left[\left\|\nabla_{\theta} f_{i}\left(\theta,\xi_{i,j}^{r}\right)-\Ep\left[\nabla_{\theta} f_{i}\left(\theta,\xi_{i,j}^{r}\right)\right]\right\|_{2}^{2}\right]\le \delta_{g}^{2}, 
    \end{equation*}
\end{assumption}

\begin{assumption}\label{assump:smoothInBoth}
    For some $L_{\lambda\theta}>0$, for all $\theta,\theta'\in\Theta$, $\|(f_{1}(\theta),\ldots,f_{n}(\theta))-(f_{1}(\theta'),\ldots,f_{n}(\theta'))\|_{2}\le L_{\lambda\theta}\|\theta-\theta'\|_{2}$ holds.
\end{assumption}

\begin{assumption}[large parameter space]\label{assump:noprojection}
    $\theta^{r}-\tau_{r}\sum_{i=1}^{n}\lambda_{i}^{r+1}\Delta u_{i}^{r}\in\Theta$ holds, almost surely.
\end{assumption}

Note that Assumption \ref{assump:smoothInBoth} is always satisfied by $L_{\lambda\theta}\le \sqrt{n}M_{f}\mathrm{diam}(\Theta)$ under Assumption \ref{assump:risk}.
Though Assumption \ref{assump:noprojection} is not a weak condition, we can easily give reasonable sufficient conditions for it such as uniform dissipativity and smoothness; see Appendix \ref{appendix:asbounded}.
Note that we need Assumption \ref{assump:noprojection} not because of our $\LASET$ setting; it is to justify the proof of Lemma A.3 of \citet{yu2023scaff}, which is used in our proof as well.

\subsubsection{Convergence Statement}

We give the following convergence guarantee for \textsc{Scaff-PD-IA}. 
which is the same as that of \citet{yu2023scaff} with $\LASET\subset\Delta^{n-1}$.
As a preparation, we consider a parameter \( \theta^\star \in \Theta \), which is the solution to the minimax optimization problem \eqref{eq:problem:scalarized}. 
The details of \( \theta^\star \), as well as its uniqueness, are demonstrated in Proposition \ref{prop:unique:general} in Section \ref{appendix:uniqune}.

\begin{theorem}[Convergence of \textsc{Scaff-PD-IA}]\label{thm:scaffpd}
    Suppose that Assumptions \ref{assump:sets}--\ref{assump:noprojection} and $\|\LASET\|_{1}<1+2m_{f}/M_{f}$ hold.
    Then, there exists a unique element $\theta^{\star}\in\Theta$ and nonempty set $\LASET^{\star}\subset\LASET$ such that $\{(\theta,\lambda)\in\Theta\times\LASET;(\theta,\lambda)\text { is a solution of the problem \eqref{eq:problem:scalarized}}\}=\{\theta^{\star}\}\times\LASET^{\star}$.
    Moreover, for some sequences of positive hyperparameters $\{\sigma_{r}\}_{r\in[R]_{0}},\{\tau_{r}\}_{r\in[R]_{0}},\{\varsigma_{r}\}_{r\in[R]},\{\eta_{r}\}_{r\in[R]}$, and positive constants $C_{1},C_{2}\ge1$, for all $R\in\N$, $\theta^{0}\in\Theta$, $\lambda^{0}\in\LASET$, and $\lambda^{\star}\in\LASET^{\star}$, we have
    \begin{equation}
        \Ep\left[\left\|\theta^{R+1}-\theta^{\star}\right\|_{2}^{2}\right]\le \frac{C_{1}\delta_{g}^{2}}{R}+\frac{C_{2}}{R^{2}}\left(D\left(\theta^{\star},\theta^{0}\right)+D\left(\lambda^{\star},\lambda^{0}\right)+F\left(\theta^{0},\lambda^{\star}\right)-F\left(\theta^{\star},\lambda^{0}\right)\right).
    \end{equation}
\end{theorem}
This theorem recovers the convergence guarantee of \textsc{Scaff-PD} by \citet{yu2023scaff} even if $\LASET$ is not a subset of $\Delta^{n-1}$.
The rate $\mathcal{O}(1/R)$ by Theorem \ref{thm:scaffpd} is optimal in a minimax sense, that is, there is no stochastic first-order algorithm whose rate of convergence is $o(1/R)$ for all the problems included in a certain class.
Our minimax optimization problem \eqref{eq:problem:scalarized} can be reduced to a minimization problem for strongly convex functions by setting $A=B=\{(1,0,\ldots,0)\}$.
\citet{agarwal2012information} show that the minimax rate of convergence of first-order stochastic algorithms for minimization of strongly convex functions is $\Omega(1/R)$.
Therefore, we immediately notice that the rate $\mathcal{O}(1/R)$ cannot be improved in the minimax sense.

In addition, the rate $\mathcal{O}(1/R^{2})$ by Theorem \ref{thm:scaffpd} under $\delta_{g}=0$ (that is, a deterministic first-order optimization setting) matches the optimal rate of deterministic first-order algorithms for strongly-convex-concave minimax optimization problems given by \citet{ouyang2021lower}.
Theorem 2.2 of \citet{ouyang2021lower} gives the lower complexity bound for deterministic first-order algorithms in our problem setting via rescaling the dual variable in its proof\footnote{Precisely speaking, \textsc{Scaff-PD-IA} has local steps for each client as $J\ge 2$, which are not represented by the oracle setting of \citet{ouyang2021lower}; the detailed investigation is left for a future work.}.

\subsection{Selection of Coefficient $\varphi$} \label{sec:selection_varphi}
We propose an adaptive approach for the selection of $\varphi\in[0,1)$ which minimizes a negative utility function of absolute unfairness and relative unfairness.
Its detailed derivation process will be presented in Section \ref{sec:selection_varphi_appendix}, and we give only its outline here.

We introduce the negative utility function, which balances the parameters for the relative and absolute unfairness as
\begin{align}
    u(\theta,A,B)&:=\left(\max_{a\in A}\sum_{i=1}^{n}a_{i}f_{i}(\theta)\right)^{1/2}\left(\frac{\max_{a\in A}\sum_{i=1}^{n}a_{i}f_{i}(\theta)}{\min_{b\in B}\sum_{i=1}^{n}b_{i}f_{i}(\theta)}\right)^{1/2}=\left(\max_{a\in A}\sum_{i=1}^{n}a_{i}f_{i}(\theta)\right)\left(\min_{b\in B}\sum_{i=1}^{n}b_{i}f_{i}(\theta)\right)^{-1/2}. \label{def:negative_utility}
\end{align}
Next, we approximate this utility function. 
With the optimizer $(\theta^\star_0, a^\star_0, b^\star_0)$ with setting $\varphi = 0$ and associated notation $\mathfrak{c}=\langle (\sum_{i=1}^{n}a_{0,i}^{\star}\nabla^{2}f_{i}(\theta_{0}^{\star}))^{-1}\sum_{i=1}^{n}b_{0,i}^{\star}\nabla f_{i}(\theta_{0}^{\star}),\sum_{i=1}^{n}b_{0,i}^{\star}\nabla f_{i}(\theta_{0}^{\star})\rangle$, $\mathfrak{a}=\sum_{i=1}^{n}a_{0,i}^{\star}f_{i}(\theta_{0}^{\star})$, and 
$\mathfrak{b}=\sum_{i=1}^{n}b_{0,i}^{\star}f_{i}(\theta_{0}^{\star})$, 
, 
we obtain a simplified approximation for the negative utility function:
\begin{equation}
    u(\hat{\theta}_{\varphi}^{\star},A,B)\approx \left(\mathfrak{a}+\frac{\varphi^{2}}{2}\mathfrak{c}\right)\left(\mathfrak{b}+\varphi \mathfrak{c}\right)^{-1/2}.
\end{equation}
Then, the right-hand side of the negative utility is minimized by
\begin{equation}
    \varphi=\frac{\sqrt{\mathfrak{b}^{2}+2(1-0.25)\mathfrak{a}\mathfrak{c}}-\mathfrak{b}}{\mathfrak{c} (1+0.5)}. \label{eq:phi_selected}
\end{equation}
This coefficient \eqref{eq:phi_selected} represents a value that appropriately balances absolute and relative unfairness. It is worth noting that the right-hand side of \eqref{eq:phi_selected} does not depend on $\varphi$, hence we only need to compute can compute $(\theta^\star_0, a^\star_0, b^\star_0)$ using only existing methods, such as \citet{yu2023scaff}, for the case where $\varphi=0$.

\section{Generalization Analysis}\label{sec:generalization}

We analyze the generalization error of a learner trained in the proposed learning problem \eqref{eq:problem:scalarized}, specifically the error of the learner with respect to new data observed independently from the training data. As discussed in \citet{mohri2019agnostic} and \citet{li2020fair}, generalization error plays a crucial role in problems where prediction is of primary importance.

Assume that the optimization objective $L_{\lambda}:\mathcal{H}\to\R$ and its empirical version $L_{\lambda,Nn}:\mathcal{H}\to\R$ with $\lambda\in\LASET$, an integrated ambiguity set $\LASET\subset\{h\in\R^{n}:\langle h,(1,\ldots,1)\rangle =1\}$, and $N\in\N$ such that
\begin{equation*}
    L_{\lambda}(h)=\sum_{i=1}^{n}\lambda_{i}\Ep_{(x,y)\sim P_{i}}[\ell(h(x),y)],\quad L_{\lambda,Nn}(h)=\sum_{i=1}^{n}\frac{\lambda_{i}}{N}\sum_{j=1}^{N}\ell(h(x_{i,j}),y_{i,j}),
\end{equation*}
where $P_{i}$ is the data distribution of the $i$-th client on a measurable space $(\mathbb{X}\times\mathbb{Y},\mathcal{X}\times\mathcal{Y})$, $\mathcal{H}$ is a collection of hypotheses $h:(\mathbb{X},\mathcal{X})\to(\mathbb{X}',\mathcal{X}')$ with a measurable space $(\mathbb{X}',\mathcal{X}')$, $\ell:(\mathbb{X}'\times\mathbb{Y},\mathcal{X}'\times\mathcal{Y})\to([0,M],\mathcal{B}(\R)|_{[0,M]})$ is a loss function with $M>0$, and $\{(x_{i,j},y_{i,j});i\in[n],j\in[N]\}$ is an array of i.i.d.~random variables such that $(x_{i,j},y_{i,j})\sim P_{i}$.
Note that we set $\ell$ is $\{-1,+1\}$-valued and fix the sample size $N$ for all the training attendees for brevity; it is easy to generalize it and let the sample sizes dependent on attendees as \citet{mohri2019agnostic}.

In order to analyze the generalization error, we introduce several technical notions.
We define $\mathcal{G}$ as the family of loss functions associated with $\mathcal{H}$ such as $\mathcal{G}=\{(x,y)\mapsto\ell(h(x),y);h\in\mathcal{H}\}$.
Let us also define the skewness of $\lambda\in\LASET$ and $\LASET$ itself as 
\begin{equation*}
    \mathfrak{s}(\lambda)=n\sum_{i=1}^{n}\left|\lambda_{i}-\frac{1}{n}\right|^{2}+1,\quad\mathfrak{s}(\LASET)=\max_{\lambda\in\LASET}\mathfrak{s}(\lambda).
\end{equation*}
We, then, present an upper bound on the generalization error:
\begin{theorem}[bounds for generalization error]\label{thm:generalization}
    Suppose that $\mathcal{G}$ admits its VC-dimension $\mathrm{VC}(\mathcal{G})<Nn$. Then, for fixed $\epsilon>0$ and $N\in\N$, for any $\delta>0$, with probability at least $1-\delta$, for all $h\in\mathcal{H}$ and $\lambda\in\LASET$, we obtain
    \begin{equation*}
        L_{\lambda}(h)\le L_{\lambda,Nn}(h)+M\epsilon+\sqrt{\frac{\mathfrak{s}(\LASET)}{2Nn}}\left(4\sqrt{\mathrm{VC}(\mathcal{G})}\sqrt{1+\log\left(\frac{Nn}{\mathrm{VC}(\mathcal{G})}\right)}+M\sqrt{\log\left(\frac{\mathcal{N}(\LASET,\|\cdot\|_{1},\epsilon)}{\delta}\right)}\right).
    \end{equation*}
\end{theorem}
This result shows that the upper bound of the generalization error is composed of the VC dimension $\mathrm{VC}(\mathcal{G})$ of the hypothesis set and the maximum skewness $\mathfrak{s}(\LASET)$ of the ambiguity set $\LASET$. 
In addition to the general upper bound based on the VC dimension, the result shows that solving the minimax problem reveals the influence of the size of $\LASET$  on the generalization error. Note that this derivation is almost equivalent to that of \citet{mohri2019agnostic} and \citet{li2020fair}.

In order to gain more specific insights, we conduct a more detailed analysis of the skewness $\mathfrak{s}(\LASET)$ under particular conditions. Specifically, we consider the typical setting where $A = B = \Delta_\alpha^{n-1}$ and examine the upper bound of the skewness $\mathfrak{s}(\LASET)$.
\begin{proposition}\label{prop:skewness}
    Set $A=B=\Delta_{\alpha}^{n-1}$ with $\alpha\in[1/n,1]$ and $\varphi\in[0,1)$.
    Then the skewness of the corresponding $\LASET$ has the following asymptotic representation:
    \begin{align*}
        \mathfrak{s}(\LASET)&= 
        \alpha\left[\left(\frac{1}{(1-\varphi)\alpha}-1\right)^{2}+\left(\frac{\varphi}{(1-\varphi)\alpha}+1\right)^{2}\right]+1+\mathcal{O}(1/n).
    \end{align*}
\end{proposition}
Hence $\mathfrak{s}(\LASET)$ for this $\LASET$ with fixed $\alpha$ and $\varphi$ is independent of $n$ and then $N$ or $n$ has the same effect to decrease the minimax Rademacher complexity as the result of \citet{mohri2019agnostic}.

The skewness parameter $\mathfrak{s}(\LASET)$ is an important feature of the bounds given by Theorem \ref{thm:generalization}.
For fixed $\alpha\in(0,1)$, as $\varphi$ increases, the skewness also increases, which results in the upper bound of the generalization error in Theorem \ref{thm:generalization} also increasing. This comes from the fact that the integrated ambiguity set $\LASET$ gets larger as $\varphi$ increases, which is explained in Section \ref{sec:proposal}.
This result suggests that considering the integrated ambiguity set to achieve relative fairness affects the generalization error because of its high degree of freedom.

\section{Experiments}\label{sec:experiments}

We compare our \textsc{Scaff-PD-IA} algorithm with the \textsc{FedAvg}, \textsc{Scaffold}, \textsc{Stochastic-AFL} (we refer to it as \textsc{SAFL} below in short), \textsc{DRFA}, and \textsc{Scaff-PD} algorithms. 
We display the results of two numerical experiments: the classification of the MNIST dataset and that of the CIFAR10 dataset with neural networks.
In both the tasks, the number of clients $n$ is fixed as $100$.
We generate the heterogeneity of the datasets of clients in the manner same as \citet{yurochkin2019bayesian} (see also \citealp{li2022federated}, \citealp{yu2022tct}, and \citealp{yu2023scaff}).
For more details of the experiments, see Appendix \ref{appendix:experiments}.

\begin{table}[]
    \centering
    \begin{tabular}{ll||ccc|cc}\hline
        & & \multicolumn{3}{c|}{Average Accuracy} & \multicolumn{2}{c}{Unfairness} \\
        Task & Algorithm &  All  & Worst-20\% & Best-20\% & $\mathfrak{R}_{A,B}$ & Gini \\\hline\hline
        \multirow{6}{*}{MNIST} & \textsc{FedAvg}  &.7080 & .4239 & .9376 & 7.966 & .3496 \\
        & \textsc{Scaffold}  &.7425 & .6026 & .8980 & 3.154 & .2047\\
        & \textsc{SAFL} &.8976 & .8295 & .9538 & 2.992 & .2084 \\
        & \textsc{DRFA} &.8969 & .8265 & \textbf{.9613} & 4.378 & .2520 \\
        & \textsc{Scaff-PD} &.8837 & .8204 & .9436 & 2.567 & .1870 \\
        & \textsc{Scaff-PD-IA} &\textbf{.9005} & \textbf{.8483} & .9523 & \textbf{2.483} & \textbf{.1802} \\ \hline
        \multirow{6}{*}{CIFAR10} & \textsc{FedAvg} & .1704 & .0006 & .3937 & 2.135 & .1587\\
        & \textsc{Scaffold} & .1365 & .0000 & .3453 & 1.612 & .1013 \\
        & \textsc{SAFL} & .3560 & \textbf{.2532} & .4711 & 1.187 & .0345 \\
        & \textsc{DRFA} & .3402 & .0615 & \textbf{.5761} & 2.277 & .1671 \\
        & \textsc{Scaff-PD} & .3682 & .2452 & .4792 & 1.187 & .0345 \\
        & \textsc{Scaff-PD-IA} & \textbf{.3693} & .2504 & .4752 & \textbf{1.167} & \textbf{.0307} \\\hline
    \end{tabular}
    \caption{Summary of the experiments for the MNIST dataset and CIFAR10 dataset with 2-layer neural nets.
    The column `Average Accuracy' summarizes the average accuracy of the population or sub-populations; the column ``All'' reports the average validation accuracy of the population and the column ``Worst-20\%'' (resp.~``Best-20\%'') exhibits the average validation accuracy of the worst-20\% (resp.~best-20\%) clients in accuracy.
    The column ``$\mathfrak{R}_{A,B}$'' displays the relative unfairness indices $\mathfrak{R}_{A,B}(\theta)$.
    The column ``Gini' exhibits the Gini coefficients for losses.
    The \textbf{bold} numbers in the accuracy columns  (resp.~unfairness column) represent the largest (resp.~smallest) values within the experiments.}
    \label{tab:experiment:summary}
\end{table}

Table \ref{tab:experiment:summary} summarizes the experiments; we measure the performance of the algorithms by (i) the average validation accuracy of all clients, (ii) the average validation accuracy of the worst-20\% clients in accuracy, (iii) the average validation accuracy of the best-20\% clients in accuracy, (iv) the relative unfairness index (the ratio of the average validation loss of the worst-20\% clients in losses to that of the best-20\% clients), and (v) the Gini coefficient for losses.
Each client splits their data into training data and validation data with the proportion 80:20 and computes its validation accuracy and loss.
The performance (i) is essentially an objective of ordinary federated learning algorithms such as \textsc{FedAvg} and \textsc{Scaffold}.
The performance (ii) is an aim of previous fair federated learning algorithms \textsc{SAFL}, \textsc{DRFA}, and \textsc{Scaff-PD}.
The performance (iv) indicates relative fairness and thus the aim of \textsc{Scaff-PD-IA}.
The Gini coefficient, a very typical unfairness measurement in welfare economics, is not our direct objective, but we introduce it as another unfairness measurement.

We see that \textsc{Scaff-PD-IA} achieves not only the smallest relative unfairness indices but also the best average accuracy of all the clients.
Though the average accuracy of population is not an objective of \textsc{Scaff-PD-IA}, the experiments indicate that the algorithm does not greatly sacrifice the average accuracy of population.
In the MNIST experiment, notably, \textsc{Scaff-PD-IA} performs better in the average accuracy of the worst-20\% than \textsc{SAFL} and \textsc{DRFA}, while it performs worse in the average accuracy of the best-20\% than them.
In total, the remarkable performance in the average accuracy of population of \textsc{Scaff-PD-IA} is not caused by subgroups enjoying good local performance. 
In the CIFAR10 experiment, \textsc{Scaff-PD-IA} dominates \textsc{DRFA} and \textsc{Scaff-PD} in the average accuracy of all and worst-20\% clients and spread between the average accuracy of worst-20\% and that of best-20\%.
On \textsc{SAFL} and \textsc{Scaff-PD-IA}, their performances in best-20\% accuracy are very close, while \textsc{Scaff-PD-IA}'s performance in the average accuracy of all clients is non-negligibly better than that \textsc{SAFL}.
\begin{figure}[htbp]
    \centering
    \includegraphics[width=0.5\linewidth]{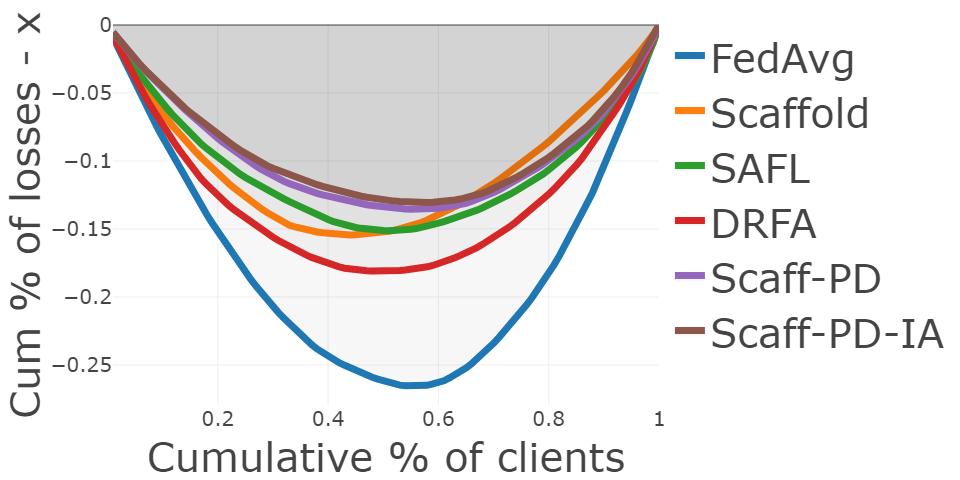}
    \caption{(Transformed) Lorenz curves of losses in the MNIST experiment.
    The smaller the area above the curve, the smaller unfairness it achieves.
    }
    \label{fig:experiment:lorenz}
\end{figure}

We also display the (transformed) Lorenz curves of the losses of the MNIST experiment in Figure \ref{fig:experiment:lorenz}.
The ordinary Lorenz curve \( \ell(x)\) is the cumulative proportion of losses earned by the bottom \( x \) fraction of the clients, and we transform it into the transformed one as $\ell'(x)=\ell(x)-x/2$.
The curve determines the Gini coefficient \( G = 1 - 2 \int_0^1 \ell' (x) \, dx \), which corresponds to a volume of the area above the curve.  
We see that the region above the the Lorenz curve for \textsc{Scaff-PD-IA} (brown) is smaller than most the others.
Although the Lorenz curve for \textsc{Scaffold} (orange) is above the curve for \textsc{Scaff-PD-IA} when $x$ is close to $1$, it is far below when $x$ is close to $0$.
The values of the Gini coefficient determined by the Lorentz curve are displayed in Table \ref{tab:experiment:summary}; \textsc{Scaff-PD-IA} achieves the smallest unfairness in the sense of not only our index but also the Gini coefficient.

\section{Conclusion}

We proposed a novel federated learning framework that prioritizes relative fairness among client subgroups. While the conventional federated learning frameworks focus on absolute fairness, we introduced a new relative fairness index and formulated a minimax optimization problem aimed at minimizing relative unfairness. Our proposed algorithm, Scaff-PD-IA, was designed to balance communication efficiency with minimax-optimal convergence rates, ensuring stable learning outcomes.
Through both theoretical analysis and empirical evaluations, we demonstrated that our framework consistently reduces relative unfairness without sacrificing model performance. 

\appendix

\section{Details of Figure \ref{fig:arc}} \label{sec:figure_explain}
    We give some details of Figure \ref{fig:arc}.
    The left panel of Figure \ref{fig:arc} is a scatter plot of 8 data points (gray balls)
    $D:=\{(\pm 1,0),(\pm1,1/2)/\sqrt{2}, (0,1/2), (\pm1,-1)/\sqrt{2}, (0,-1)\}$ 
    , the solution of $\min_{\theta\in\R^{2}}\max_{v\in D}\|v-\theta\|_{2}$ (pink square), and that of $\min_{\theta\in\R^{2}}(\max_{v\in D}\|v-\theta\|_{2}-\min_{v'\in D}\|v'-\theta\|_{2})$ (blue square). 
    The middle panel of Figure \ref{fig:arc} is a bar plot of the $\ell^{2}$-distances between the gray data points and the pink square point.
    The right panel of Figure \ref{fig:arc} is a bar plot of the $\ell^{2}$-distances between the gray data points and the blue square point.
    We can observe that the blue point distributes the distances more fairly than the pink point (the sums of the distances are 7.081 (pink) and 7.105 (blue) and thus almost the same).
    Absolute fairness merely minimization of the maximum loss, but relative fairness considers minimization of the discrepancy between losses.

\section{Other (Un)Fairness Measures} \label{sec:compare_fairness}

We present some (un)fairness measures related to the relative unfairness index.
For a given sequence of $n$ utilities $\{u_{i}\}_{i\in [n]}\subset(0,\infty)^{n}$, we can represent the 20:20 ratio, the Palma ratio \citep{cobham2013all}, and the Atkinson inequality measure with infinite aversion as 
\begin{align*}
    \text{(20:20 ratio) } &\frac{\frac{1}{U}\max_{\sigma\in S_{n}}\sum_{i=1}^{\lfloor .20\times n\rfloor}u_{\sigma(i)}}{\frac{1}{U}\min_{\sigma\in S_{n}}\sum_{i=1}^{\lfloor .20\times n\rfloor}u_{\sigma(i)}}=\max_{\sigma\in S_{n}}\frac{\frac{1}{\lfloor .20\times n\rfloor}\sum_{i=1}^{\lfloor .20\times n\rfloor}u_{\sigma(i)}}{\frac{1}{\lfloor .20\times n\rfloor}\sum_{i=1}^{\lfloor .20\times n\rfloor}u_{\sigma(n-i+1)}},\\
    \text{(The Palma ratio) } &\frac{\frac{1}{U}\max_{\sigma\in S_{n}}\sum_{i=1}^{\lfloor .10\times n\rfloor}u_{\sigma(i)}}{\frac{1}{U} \min_{\sigma\in S_{n}}\sum_{i=1}^{\lfloor .40\times n\rfloor}u_{\sigma(i)}}\approx\frac{1}{4}\max_{\sigma\in S_{n}}\frac{\frac{1}{\lfloor .10\times n\rfloor}\sum_{i=1}^{\lfloor .10\times n\rfloor}u_{\sigma(i)}}{\frac{1}{\lfloor .40\times n\rfloor}\sum_{i=1}^{\lfloor .40\times n\rfloor}u_{\sigma(n-i+1)}},\\
    \text{(An Atkinson measure) } &1-\frac{\min_{i}u_{i}}{\frac{1}{n} \sum_{i=1}^{n}u_{i}},
\end{align*}
where $U=\sum_{i}u_{i}$ is the total utility.
As losses are opposite to utilities, we can naturally define their versions for losses by taking their reciprocals, and the relative fairness index \eqref{def:rf_index} can represent them for every $\theta$. 

\section{Details of Selection of Coefficient $\varphi$} \label{sec:selection_varphi_appendix}

We provide the details of the selection process of the coefficient $\varphi$, presented in Section \ref{sec:selection_varphi}.

First, we introduce the negative utility function, which balances the parameters for the relative and absolute unfairness.
Second, we obtain a parameter for control the absolute unfairness by training the parameter with $\varphi=0$.
Third, we obtain the approximator of the trained parameter with $\varphi > 0$, by minimizing the second-order Taylor approximation of the target value of the optimization.
Finally, we derive an analytic form of the optimal $\varphi$ which minimize the negative utility.

\paragraph{Step i.}
We introduce the following negative utility function balancing absolute unfairness and relative unfairness as \eqref{def:negative_utility}.
This is a monomial (or, more specifically, geometric mean) of absolute fairness and relative fairness.
Monomial functions are a typical and fundamental form of cost functions, and its optimization itself has been of interest \citep{duffin1962cost,duffin1967geometric,boyd2007tutorial}.
This objective is design to achieve a solution with better relative fairness to ones with similar absolute fairness performance.
In what follows, we select the coefficient $\varphi$ to minimize the negative utility function.

\paragraph{Step ii.}
To study the minimization of the negative utility function $u(\theta,A,B)$, we consider tuples $(\theta^\star_\varphi, a^\star_\varphi, b^\star_\varphi)$ and $(\theta^\star_0, a^\star_0, b^\star_0)$.
These notations follow the optimized parameters defined in Section \ref{sec:proposed_framework}.
Note that the first tuple $(\theta^\star_\varphi, a^\star_\varphi, b^\star_\varphi)$ with $\varphi > 0$ is obtained by our framework and the algorithm, and the second tuple $(\theta^\star_0, a^\star_0, b^\star_0)$ can be obtained by the previous DRO studies such as \citet{yu2023scaff}.

\paragraph{Step iii.}
We approximate the value of $\sum_{i=1}^{n}a_{\varphi,i}^{\star}f_{i}(\hat{\theta}_{\varphi}^{\star})$ and $\sum_{i=1}^{n}b_{\varphi,i}^{\star}f_{i}(\hat{\theta}_{\varphi}^{\star})$ under $\varphi > 0$ by a function of $\theta^\star_0$ without $\varphi$. To achieve this, we perform a second-order Taylor approximation of the objective function \eqref{eq:problem:scalarized} around $\theta^\star_0$.
In other words, under smoothness conditions and assumptions that solutions $(a_{\varphi}^{\star},b_{\varphi}^{\star})$ is close to $(a_{0}^{\star},b_{0}^{\star})$,
\begin{align}
    \sum_{i=1}^{n}a_{\varphi,i}^{\star}f_{i}(\theta)&\approx \sum_{i=1}^{n}a_{0,i}^{\star}f_{i}(\theta_{0}^{\star})+\frac{1}{2}\sum_{i=1}^{n}a_{0,i}^{\star}\left\langle \nabla^{2}f_{i}(\theta_{0}^{\star})(\theta-\theta_{0}^{\star}),\theta-\theta_{0}^{\star}\right\rangle,\\
    \sum_{i=1}^{n}b_{\varphi,i}^{\star}f_{i}(\theta)&\approx \sum_{i=1}^{n}b_{0,i}^{\star}f_{i}(\theta_{0}^{\star})+\sum_{i=1}^{n}b_{0,i}^{\star}\left\langle \nabla f_{i}(\theta_{0}^{\star}),\theta-\theta_{0}^{\star}\right\rangle.
\end{align}
By some calculation, we consider the following approximate optimization problem
\begin{equation}
    \min_{\theta\in\R^{d}}\left(\frac{1}{2}\sum_{i=1}^{n}a_{0,i}^{\star}\left\langle \nabla^{2}f_{i}(\theta_{0}^{\star})(\theta-\theta_{0}^{\star}),\theta-\theta_{0}^{\star}\right\rangle-\varphi\sum_{i=1}^{n}b_{0,i}^{\star}\left\langle \nabla f_{i}(\theta_{0}^{\star}),\theta-\theta_{0}^{\star}\right\rangle\right).
\end{equation}
Then, the solution of this problem is written as  
\(
\hat{\theta}_{\varphi}^{\star} =\theta_{0}^{\star}+\varphi(\sum_{i=1}^{n}a_{0,i}^{\star}\nabla^{2}f_{i}(\theta_{0}^{\star}))^{-1}\sum_{i=1}^{n}b_{0,i}^{\star}\nabla f_{i}(\theta_{0}^{\star}) 
\)
with invertible $\sum_{i=1}^{n}a_{0,i}^{\star}\nabla^{2}f_{i}(\theta_{0}^{\star})$.
As a result, we obtain the following approximation results
\begin{equation}
    \sum_{i=1}^{n}a_{\varphi,i}^{\star}f_{i}(\hat{\theta}_{\varphi}^{\star})\approx \mathfrak{a}+\frac{\varphi^{2}}{2}\mathfrak{c},\mbox{~and~} \sum_{i=1}^{n}b_{\varphi,i}^{\star}f_{i}(\hat{\theta}_{\varphi}^{\star})\approx \mathfrak{b}+\varphi \mathfrak{c}, \label{eq:approximation}
\end{equation}
where
$\mathfrak{c}=\langle (\sum_{i=1}^{n}a_{0,i}^{\star}\nabla^{2}f_{i}(\theta_{0}^{\star}))^{-1}\sum_{i=1}^{n}b_{0,i}^{\star}\nabla f_{i}(\theta_{0}^{\star}),\sum_{i=1}^{n}b_{0,i}^{\star}\nabla f_{i}(\theta_{0}^{\star})\rangle$, $\mathfrak{a}=\sum_{i=1}^{n}a_{0,i}^{\star}f_{i}(\theta_{0}^{\star})$, and
$\mathfrak{b}=\sum_{i=1}^{n}b_{0,i}^{\star}f_{i}(\theta_{0}^{\star})$.
.

\paragraph{Step iv.}
By substituting the approximation results \eqref{eq:approximation} into the negative utility function \eqref{def:negative_utility}, we obtain a simplified expression for the negative utility:
\begin{equation}
    u(\hat{\theta}_{\varphi}^{\star},A,B)\approx \left(\mathfrak{a}+\frac{\varphi^{2}}{2}\mathfrak{c}\right)\left(\mathfrak{b}+\varphi \mathfrak{c}\right)^{-1/2}.
\end{equation}
Then, the right-hand side of the negative utility is minimized by
\begin{equation}
    \varphi=\frac{\sqrt{\mathfrak{b}^{2}+2(1-0.25)\mathfrak{a}\mathfrak{c}}-\mathfrak{b}}{\mathfrak{c} (1+0.5)}.
\end{equation}

\section{Technical Results}\label{appendix:technical}
We show a proof of the following elementary result for clarity and self-containedness.
\begin{lemma}\label{lem:existence}
    For $m,n\in\N$, let $\Theta\subset\R^{m}$ and $\Lambda\subset\R^{n}$ be compact sets, $\phi:\Theta\to\R^{n}$ be a continuous function, and $\psi:\Lambda\to\R$ be a lower semicontinuous function.
    There exist some $\theta^{\star}\in \Theta$ and $\lambda^{\star}\in\Lambda$ such that (i) it holds that
    \begin{equation*}
        \langle \phi(\theta^{\star}),\lambda^{\star}\rangle -\psi(\lambda^{\star})=\min_{\theta\in \Theta}\max_{\lambda\in\Lambda}\left(\langle \phi(\theta^{\star}),\lambda\rangle -\psi(\lambda)\right),
    \end{equation*}
    and (ii) for any $\theta\in\Theta$ and $\lambda\in\Lambda$, 
    \begin{equation*}
        \langle \phi(\theta^{\star}),\lambda\rangle -\psi(\lambda)\le \langle \phi(\theta^{\star}),\lambda^{\star}\rangle -\psi(\lambda^{\star})\le \langle \phi(\theta),\lambda^{\star}\rangle -\psi(\lambda^{\star}).
    \end{equation*}
\end{lemma}
\begin{proof}
    Define $V:=\{\phi(\theta);\theta\in\Theta\}\subset\R^{n}$, which is a compact set in $\R^{n}$ due to the continuity of $\phi$ and $\Theta$.
    For any $v\in V$, $\langle v,\lambda\rangle -\psi(\lambda)$ attains its maximum in $\lambda\in\Lambda$ owing to the lower semicontinuity of $-\langle v,\cdot\rangle $ and $\psi(\cdot)$ and the compactness of $\Lambda$.
    Define a function $\rho:V\to\R$ such that $\rho(v)=\max_{\lambda\in\Lambda}(\langle v,\lambda\rangle -\psi(\lambda))$ and then this $\rho$ is a Lipschitz continuous function due to the compactness of $\Lambda$, which attains its minimum in $v\in V$ owing to its continuity and the compactness of $V$.
    By the definition of $V$, there exist $\theta\in\Theta$ such that $\phi(\theta)$ minimizes $\rho(\phi(\theta))$.
    This is the desired conclusion.
\end{proof}
\begin{remark}
Based on this result, the existence of solutions is guaranteed for all the minimax problems in this paper and will not be mentioned below.
\end{remark}

\begin{lemma}\label{lem:max:sign:general}
    For any compact $\Lambda\subset\{\lambda\in\R^{n}:\sum_{i=1}^{n}\lambda_{i}=1\}$, it holds that
    \begin{align*}
       \text{(i)}& \max_{\lambda\in\Lambda}\sum_{i=1}^{n}\lambda_{i}\mathbb{I}_{[0,\infty)}\left(\lambda_{i}\right)=\frac{1}{2}\left(\left\|\Lambda\right\|_{1}+1\right),\\
       \text{(ii)}&\max_{\lambda\in\Lambda}\sum_{i=1}^{n}|\lambda_{i}|\mathbb{I}_{(-\infty,0)}\left(\lambda_{i}\right)=\frac{1}{2}\left(\left\|\Lambda\right\|_{1}-1\right).
    \end{align*}
\end{lemma}

\begin{proof}
    Note that for any $\lambda\in\R^{n}$ with $\sum_{i=1}^{n}\lambda_{i}=1$,
    \begin{align*}
        &\sum_{i=1}^{n}\lambda_{i}\mathbb{I}_{[0,\infty)}\left(\lambda_{i}\right)=\frac{1}{2}\sum_{i=1}^{n}\left(\left|\lambda_{i}\right|+\lambda_{i}\right)=\frac{1}{2}\sum_{i=1}^{n}\left|\lambda_{i}\right|+\frac{1}{2},\\
        &\sum_{i=1}^{n}|\lambda_{i}|\mathbb{I}_{(-\infty,0)}\left(\lambda_{i}\right)=\frac{1}{2}\sum_{i=1}^{n}\left(\left|\lambda_{i}\right|-\lambda_{i}\right)=\frac{1}{2}\sum_{i=1}^{n}\left|\lambda_{i}\right|-\frac{1}{2}.
    \end{align*}
    Hence the statement holds.
\end{proof}

\begin{lemma}\label{lem:min:weightedMean:general}
    For all compact $\Lambda\subset\{\lambda\in\R^{n}:\sum_{i=1}^{n}\lambda_{i}=1\}$, $M,m\in\R$ with $M\ge m$, and $n\in\N$, it holds that
    \begin{align*}
        \min_{\lambda\in\Lambda,x\in[m,M]^{n}}\sum_{i=1}^{n}\lambda_{i}x_{i}= m-\frac{M-m}{2}\left(\left\|\Lambda\right\|_{1}-1\right).
    \end{align*}
\end{lemma}

\begin{proof}
    We just consider $M,m\in\R$ with $M>m$ since the statement under $M=m$ is trivial.
    
    First, we see that the argument of the minimum on the left-hand side of the statement is bounded below by the right-hand side.
    By letting $z_{i}:=(M-m)^{-1}(x_{i}-m)$, we obtain
    \begin{align*}
        \sum_{i=1}^{n}\lambda_{i}x_{i}\ge m-\frac{M-m}{2}\left(\left\|\Lambda\right\|_{1}-1\right) &\iff \sum_{i=1}^{n}\lambda_{i}(x_{i}-m)\ge -\frac{M-m}{2}\left(\left\|\Lambda\right\|_{1}-1\right)\\
        &\iff \sum_{i=1}^{n}\lambda_{i}z_{i}\ge -\frac{1}{2}\left(\left\|\Lambda\right\|_{1}-1\right).
    \end{align*}
    Hence it suffices to see that the last inequality holds for any $z_{i}\in[0,1]$.
    Lemma \ref{lem:max:sign:general} gives that for any $\lambda\in\Lambda$,
    \begin{align*}
        \sum_{i=1}^{n}\lambda_{i}\mathbb{I}_{(-\infty,0)}(\lambda_{i})\ge -\frac{1}{2}\left(\left\|\Lambda\right\|_{1}-1\right).
    \end{align*}
    Since for any $a\in\R$ and $b\in[0,1]$, $ab\ge a\mathbb{I}_{(-\infty,0)}(a)$, it holds that
    \begin{align*}
        \sum_{i=1}^{n}\lambda_{i}z_{i}\ge -\frac{1}{2}\left(\left\|\Lambda\right\|_{1}-1\right).
    \end{align*}

    The lower bound can be achieved by taking $\lambda\in\Lambda$ such that $\|\lambda\|_{1}=\|\Lambda\|_{1}$ and $x_{i}=m+(M-m)\mathbb{I}_{(-\infty,0)}(\lambda_{i})$,
    \begin{align*}
        \sum_{i=1}^{n}\lambda_{i}x_{i}=m+(M-m)\sum_{\lambda_{i}<0}\lambda_{i}\mathbb{I}_{(-\infty,0)}(\lambda_{i})=m-\frac{M-m}{2}\left(\left\|\Lambda\right\|_{1}-1\right).
    \end{align*}
    Therefore, we obtain the conclusion.
\end{proof}

\begin{proposition}\label{prop:ratiobounds}
    Fix a set $X$ and real functions $f:X\to[0,\infty)$ and $g:X\to[m_{g},M_{g}]$ with some $m_{g},M_{g}>0$. Then, for all $\varphi\in[0,1)$ and $x\in X$,
    \begin{equation*}
        \varphi+\frac{1}{M_{g}}\left(f(x)-\varphi g(x)\right)\le \frac{f(x)}{g(x)}\le \varphi+\frac{1}{m_{g}}\left(f(x)-\varphi g(x)\right).
    \end{equation*}
    Moreover, the lower bound is non-decreasing in $\varphi$ for each $x\in X$, and the upper bounds are non-increasing in $\varphi$ for each $x\in X$.
\end{proposition}

\begin{proof}
    We only consider the upper bound since parallel discussion immediately holds for the lower bound.
    We see that
    \begin{equation*}
        \frac{f(x)}{g(x)}=\varphi+\frac{1}{g(x)}\left(f(x)-\varphi g(x)\right)\le \varphi+\frac{1}{m_{g}}\left(f(x)-\varphi g(x)\right).
    \end{equation*}
    Rearranging the right-hand side, we have that
    \begin{equation*}
        \varphi+\frac{1}{m_{g}}\left(f(x)-\varphi g(x)\right)=\frac{f(x)}{m_{g}}+\varphi\underbrace{\left(1-\frac{g(x)}{m_{g}}\right)}_{\le 0}.
    \end{equation*}
    This is the desired conclusion.
\end{proof}

\section{Theoretical Analysis of Relative Unfairness Indices}

\subsection{Proof of Theorem \ref{thm:mono:nonstrict}}
For convenience, we introduce the following notation within this section: $F_{A,\max}:\Theta\to\R$ and $F_{B,\min}:\Theta\to\R$ such that
\begin{align}
    F_{A,\max}(\theta)=\max_{a\in A}\sum_{i=1}^{n}a_{i}f_{i}(\theta),\ F_{B,\min}(\theta)=\min_{b\in B}\sum_{i=1}^{n}b_{i}f_{i}(\theta).
\end{align}
\begin{proof}[Proof of Theorem \ref{thm:mono:nonstrict}]
    Fix $\varphi,\tilde{\varphi}\in[0,1)$ with $\varphi<\tilde{\varphi}$. 
    By the definition of $\theta_{\varphi}^{\star},\theta_{\tilde{\varphi}}^{\star}$, the following system of inequalities holds:
    \begin{align*}
        F_{A,\max}(\theta_{\varphi}^{\star})-\varphi F_{B,\min}(\theta_{\varphi}^{\star})&\le F_{A,\max}(\theta_{\tilde{\varphi}}^{\star})-\varphi F_{B,\min}(\theta_{\tilde{\varphi}}^{\star}),\\
        F_{A,\max}(\theta_{\tilde{\varphi}}^{\star})-\tilde{\varphi}F_{B,\min}(\theta_{\tilde{\varphi}}^{\star})&\le F_{A,\max}(\theta_{\varphi}^{\star})-\tilde{\varphi}F_{B,\min}(\theta_{\varphi}^{\star}).
    \end{align*}
    We easily derive 
    \begin{equation}
        F_{A,\max}(\theta_{\varphi}^{\star})\le F_{A,\max}(\theta_{\tilde{\varphi}}^{\star}).\label{eq:mono:proof:1}
    \end{equation}
    Noticing that the system is equivalent to
    \begin{align*}
        (1-\varphi)F_{A,\max}(\theta_{\varphi}^{\star})+\varphi (F_{A,\max}-F_{B,\min})(\theta_{\varphi}^{\star})&\le (1-\varphi)F_{A,\max}(\theta_{\tilde{\varphi}}^{\star})+\varphi (F_{A,\max}-F_{B,\min})(\theta_{\tilde{\varphi}}^{\star}),\\
        (1-\tilde{\varphi})F_{A,\max}(\theta_{\tilde{\varphi}}^{\star})+\tilde{\varphi}(F_{A,\max}-F_{B,\min})(\theta_{\tilde{\varphi}}^{\star})&\le (1-\tilde{\varphi})F_{A,\max}(\theta_{\varphi}^{\star})+\tilde{\varphi}(F_{A,\max}-F_{B,\min})(\theta_{\varphi}^{\star}),
    \end{align*}
    we also obtain
    \begin{align*}
        F_{A,\max}(\theta_{\varphi}^{\star})+\frac{\varphi}{1-\varphi} (F_{A,\max}-F_{B,\min})(\theta_{\varphi}^{\star})&\le F_{A,\max}(\theta_{\tilde{\varphi}}^{\star})+\frac{\varphi}{1-\varphi} (F_{A,\max}-F_{B,\min})(\theta_{\tilde{\varphi}}^{\star}),\\
        F_{A,\max}(\theta_{\tilde{\varphi}}^{\star})+\frac{\tilde{\varphi}}{1-\tilde{\varphi}}(F_{A,\max}-F_{B,\min})(\theta_{\tilde{\varphi}}^{\star})&\le F_{A,\max}(\theta_{\varphi}^{\star})+\frac{\tilde{\varphi}}{1-\tilde{\varphi}}(F_{A,\max}-F_{B,\min})(\theta_{\varphi}^{\star}),
    \end{align*}
    and thus
    \begin{equation}
        F_{A,\max}(\theta_{\tilde{\varphi}}^{\star})-F_{B,\min}(\theta_{\tilde{\varphi}}^{\star})\le F_{A,\max}(\theta_{\varphi}^{\star})-F_{B,\min}(\theta_{\varphi}^{\star})\label{eq:mono:proof:2}.
    \end{equation}
    Therefore,
    \begin{equation*}
        \frac{F_{A,\max}(\theta_{\tilde{\varphi}}^{\star})-F_{B,\min}(\theta_{\tilde{\varphi}}^{\star})}{F_{A,\max}(\theta_{\tilde{\varphi}}^{\star})}\le \frac{F_{A,\max}(\theta_{\varphi}^{\star})-F_{B,\min}(\theta_{\varphi}^{\star})}{F_{A,\max}(\theta_{\varphi}^{\star})},
    \end{equation*}
    which gives
    \begin{equation*}
        \frac{F_{A,\max}(\theta_{\tilde{\varphi}}^{\star})}{F_{B,\min}(\theta_{\tilde{\varphi}}^{\star})}\le \frac{F_{A,\max}(\theta_{\varphi}^{\star})}{F_{B,\min}(\theta_{\varphi}^{\star})}.
    \end{equation*}
    This is the desired conclusion.
\end{proof}

\subsection{Proof of Theorem \ref{thm:mono:strict}}

\begin{proof}[Proof of Theorem \ref{thm:mono:strict}]
    Let $\varphi,\tilde{\varphi}\in[0,\overline{\varphi})$ with $\varphi<\tilde{\varphi}$.
    
    (i) We show that all $(\theta_{\varphi}^{\star},a_{\varphi}^{\star},b_{\varphi}^{\star})$ never solve the minimax problem
    \begin{equation*}
        \min_{\theta\in\Theta}\max_{a\in A,b\in B}\left(\sum_{i=1}^{n}a_{i}f_{i}(\theta)-\tilde{\varphi}\sum_{i=1}^{n}b_{i}f_{i}(\theta)\right)
    \end{equation*}
    by contradiction.
    Assume that for some $(\theta_{\varphi}^{\star},a_{\varphi}^{\star},b_{\varphi}^{\star})$, (a) it holds that 
    \begin{equation*}
        \sum_{i=1}^{n}a_{i,\varphi}^{\star}f_{i}(\theta_{\varphi}^{\star})-\tilde{\varphi}\sum_{i=1}^{n}b_{i,\varphi}^{\star}f_{i}(\theta_{\varphi}^{\star})=\min_{\theta\in\Theta}\max_{a\in A,b\in B}\left(\sum_{i=1}^{n}a_{i}f_{i}(\theta)-\tilde{\varphi}\sum_{i=1}^{n}b_{i}f_{i}(\theta)\right)
    \end{equation*}
    and (b) for any $(\theta,a,b)\in \Theta\times A\times B$,
    \begin{equation*}
        \sum_{i=1}^{n}a_{i}f_{i}(\theta_{\varphi}^{\star})-\tilde{\varphi}\sum_{i=1}^{n}b_{i}f_{i}(\theta_{\varphi}^{\star})\le \sum_{i=1}^{n}a_{i,\varphi}^{\star}f_{i}(\theta_{\varphi}^{\star})-\tilde{\varphi}\sum_{i=1}^{n}b_{i,\varphi}^{\star}f_{i}(\theta_{\varphi}^{\star})\le \sum_{i=1}^{n}a_{i,\varphi}^{\star}f_{i}(\theta)-\tilde{\varphi}\sum_{i=1}^{n}b_{i,\varphi}^{\star}f_{i}(\theta).
    \end{equation*}
    Then it yields that for any $\theta\in\Theta$,
    \begin{align*}
        \sum_{i=1}^{n}a_{i,\varphi}^{\star}f_{i}(\theta_{\varphi}^{\star})-\varphi\sum_{i=1}^{n}b_{i,\varphi}^{\star}f_{i}(\theta_{\varphi}^{\star})&\le \sum_{i=1}^{n}a_{i,\varphi}^{\star}f_{i}(\theta)-\varphi\sum_{i=1}^{n}b_{i,\varphi}^{\star}f_{i}(\theta)\\
        \sum_{i=1}^{n}a_{i,\varphi}^{\star}f_{i}(\theta_{\varphi}^{\star})-\tilde{\varphi}\sum_{i=1}^{n}b_{i,\varphi}^{\star}f_{i}(\theta_{\varphi}^{\star})&\le \sum_{i=1}^{n}a_{i,\varphi}^{\star}f_{i}(\theta)-\tilde{\varphi}\sum_{i=1}^{n}b_{i,\varphi}^{\star}f_{i}(\theta).
    \end{align*}
    On the other hand, for any $\theta\in\Theta$,
    \begin{equation*}
        \sum_{i=1}^{n}a_{i,\varphi}^{\star}f_{i}(\theta)-\tilde{\varphi}\sum_{i=1}^{n}b_{i,\varphi}^{\star}f_{i}(\theta)=\sum_{i=1}^{n}a_{i,\varphi}^{\star}f_{i}(\theta)-\varphi\sum_{i=1}^{n}b_{i,\varphi}^{\star}f_{i}(\theta)+(\varphi-\tilde{\varphi})\sum_{i=1}^{n}b_{i,\varphi}^{\star}f_{i}(\theta)
    \end{equation*}
    and thus
    \begin{equation*}
        \sum_{i=1}^{n}a_{i,\varphi}^{\star}\nabla f_{i}(\theta)-\tilde{\varphi}\sum_{i=1}^{n}b_{i,\varphi}^{\star}\nabla f_{i}(\theta)=\sum_{i=1}^{n}a_{i,\varphi}^{\star}\nabla f_{i}(\theta)-\varphi\sum_{i=1}^{n}b_{i,\varphi}^{\star}\nabla f_{i}(\theta)+(\varphi-\tilde{\varphi})\sum_{i=1}^{n}b_{i,\varphi}^{\star}\nabla f_{i}(\theta).
    \end{equation*}
    By letting $\theta=\theta_{\varphi}^{\star}$ and noting the optimality of $\theta_{\varphi}^{\star}\in\mathrm{int}(\Theta)$, we conclude $\sum_{i=1}^{n}b_{i,\varphi}^{\star}\nabla f_{i}(\theta)=\mathbf{0}$ and thus $\sum_{i=1}^{n}a_{i,\varphi}^{\star}\nabla f_{i}(\theta)=\mathbf{0}$.
    This contradicts the assumption.

    (ii) In the next place, we show that for any $(\theta_{\varphi}^{\star},a_{\varphi}^{\star},b_{\varphi}^{\star})$ and $(\theta_{\tilde{\varphi}}^{\star},a_{\tilde{\varphi}}^{\star},b_{\tilde{\varphi}}^{\star})$, either
    \begin{align*}
        \max_{a\in A}\sum_{i=1}^{n}a_{i}f_{i}(\theta_{\tilde{\varphi}}^{\star})&> \max_{a\in A}\sum_{i=1}^{n}a_{i}f_{i}(\theta_{\varphi}^{\star}),\\
        \max_{a\in A}\sum_{i=1}^{n}a_{i}f_{i}(\theta_{\varphi}^{\star})-\min_{b\in B}\sum_{i=1}^{n}b_{i}f_{i}(\theta_{\varphi}^{\star})&\ge \max_{a\in A}\sum_{i=1}^{n}a_{i}f_{i}(\theta_{\tilde{\varphi}}^{\star})-\min_{b\in B}\sum_{i=1}^{n}b_{i}f_{i}(\theta_{\tilde{\varphi}}^{\star})
    \end{align*}
    or
    \begin{align*}
        \max_{a\in A}\sum_{i=1}^{n}a_{i}f_{i}(\theta_{\tilde{\varphi}}^{\star})&\ge  \max_{a\in A}\sum_{i=1}^{n}a_{i}f_{i}(\theta_{\varphi}^{\star}),\\
        \max_{a\in A}\sum_{i=1}^{n}a_{i}f_{i}(\theta_{\varphi}^{\star})-\min_{b\in B}\sum_{i=1}^{n}b_{i}f_{i}(\theta_{\varphi}^{\star})&> \max_{a\in A}\sum_{i=1}^{n}a_{i}f_{i}(\theta_{\tilde{\varphi}}^{\star})-\min_{b\in B}\sum_{i=1}^{n}b_{i}f_{i}(\theta_{\tilde{\varphi}}^{\star})
    \end{align*}
    holds.
    We prove it by Eqs.~\eqref{eq:mono:proof:1} and \eqref{eq:mono:proof:2} and contradiction.
    We assume both of the following equalities for some $(\theta_{\varphi}^{\star},a_{\varphi}^{\star},b_{\varphi}^{\star})$ and $(\theta_{\tilde{\varphi}}^{\star},a_{\tilde{\varphi}}^{\star},b_{\tilde{\varphi}}^{\star})$:
    \begin{align*}
        \max_{a\in A}\sum_{i=1}^{n}a_{i}f_{i}(\theta_{\tilde{\varphi}}^{\star})&= \max_{a\in A}\sum_{i=1}^{n}a_{i}f_{i}(\theta_{\varphi}^{\star}),\\
        \max_{a\in A}\sum_{i=1}^{n}a_{i}f_{i}(\theta_{\varphi}^{\star})-\min_{b\in B}\sum_{i=1}^{n}b_{i}f_{i}(\theta_{\varphi}^{\star})&= \max_{a\in A}\sum_{i=1}^{n}a_{i}f_{i}(\theta_{\tilde{\varphi}}^{\star})-\min_{b\in B}\sum_{i=1}^{n}b_{i}f_{i}(\theta_{\tilde{\varphi}}^{\star}).
    \end{align*}
    Then, it concludes that (a)
    \begin{align*}
        \sum_{i=1}^{n}a_{i,\varphi}^{\star}f_{i}(\theta_{\varphi}^{\star})-\tilde{\varphi}\sum_{i=1}^{n}b_{i,\varphi}^{\star}f_{i}(\theta_{\varphi}^{\star})&=\sum_{i=1}^{n}a_{i,\tilde{\varphi}}^{\star}f_{i}(\theta_{\tilde{\varphi}}^{\star})-\tilde{\varphi}\sum_{i=1}^{n}b_{i,\tilde{\varphi}}^{\star}f_{i}(\theta_{\tilde{\varphi}}^{\star})\\
        &=\min_{\theta\in\Theta}\max_{a\in A,b\in B}\left(\sum_{i=1}^{n}a_{i}f_{i}(\theta)-\tilde{\varphi}\sum_{i=1}^{n}b_{i}f_{i}(\theta)\right),
    \end{align*}
    and (b) for all $\theta\in\Theta$ and $(a,b)\in A\times B$,
    \begin{equation*}
        \sum_{i=1}^{n}a_{i}f_{i}(\theta_{\varphi}^{\star})-\tilde{\varphi}\sum_{i=1}^{n}b_{i}f_{i}(\theta_{\varphi}^{\star})
        \le \sum_{i=1}^{n}a_{i,\varphi}^{\star}f_{i}(\theta_{\varphi}^{\star})-\tilde{\varphi}\sum_{i=1}^{n}b_{i,\varphi}^{\star}f_{i}(\theta_{\varphi}^{\star})
        \le \sum_{i=1}^{n}a_{i,\varphi}^{\star}f_{i}(\theta)-\tilde{\varphi}\sum_{i=1}^{n}b_{i,\varphi}^{\star}f_{i}(\theta)
    \end{equation*}
    because of the following equality holding true and (a):
    \begin{equation*}
        \sum_{i=1}^{n}a_{i,\varphi}^{\star}f_{i}(\theta_{\varphi}^{\star})-\tilde{\varphi}\sum_{i=1}^{n}b_{i,\varphi}^{\star}f_{i}(\theta_{\varphi}^{\star})=\max_{a\in A, b\in B}\sum_{i=1}^{n}a_{i}f_{i}(\theta_{\varphi}^{\star})-\tilde{\varphi}\sum_{i=1}^{n}b_{i}f_{i}(\theta_{\varphi}^{\star}).
    \end{equation*}
    However, both (a) and (b) never hold simultaneously by (i). Hence Eqs.~\eqref{eq:mono:proof:1} and \eqref{eq:mono:proof:2} give the desired result.

    (iii) Finally, (ii) yield
    \begin{equation*}
        \frac{\max_{a\in A}\sum_{i=1}^{n}a_{i}f_{i}(\theta_{\varphi}^{\star})-\min_{b\in B}\sum_{i=1}^{n}b_{i}f_{i}(\theta_{\varphi}^{\star})}{\max_{a\in A}\sum_{i=1}^{n}a_{i}f_{i}(\theta_{\varphi}^{\star})}
        > \frac{\max_{a\in A}\sum_{i=1}^{n}a_{i}f_{i}(\theta_{\tilde{\varphi}}^{\star})-\min_{b\in B}\sum_{i=1}^{n}b_{i}f_{i}(\theta_{\tilde{\varphi}}^{\star})}{\max_{a\in A}\sum_{i=1}^{n}a_{i}f_{i}(\theta_{\tilde{\varphi}}^{\star})}.
    \end{equation*}
    Therefore,
    \begin{equation*}
        \frac{\max_{a\in A}\sum_{i=1}^{n}a_{i}f_{i}(\theta_{\varphi}^{\star})}{\min_{b\in B}\sum_{i=1}^{n}b_{i}f_{i}(\theta_{\varphi}^{\star})}
        > \frac{\max_{a\in A}\sum_{i=1}^{n}a_{i}f_{i}(\theta_{\tilde{\varphi}}^{\star})}{\min_{b\in B}\sum_{i=1}^{n}b_{i}f_{i}(\theta_{\tilde{\varphi}}^{\star})}.
    \end{equation*}
    We obtain the conclusion.
\end{proof}

\section{Uniqueness of Solutions}\label{appendix:uniqune}
We give the following proposition, which shows that $\Theta^{\star}$ is a singleton under mild conditions. 
\begin{proposition}\label{prop:unique:general}
    Suppose that Assumptions \ref{assump:sets} and \ref{assump:risk} hold. 
    Let $\psi:\LASET\to\R$ be a lower semicontinuous function.
    If (i) $M_{f}>m_{f}$ and $\|\LASET\|_{1}<1+2m_{f}/(M_{f}-m_{f})$ or (ii) $M_{f}=m_{f}$, then there exists a unique $\theta^{\star}\in\Theta$ such that for some $\lambda^{\star}\in\LASET$,
    \begin{align*}
        \sum_{i=1}^{n}\lambda_{i}^{\star}f_{i}(\theta^{\star})-\psi(\lambda^{\star})=\min_{\theta\in\Theta}\max_{\lambda\in\LASET}\left(\sum_{i=1}^{n}\lambda_{i}f_{i}(\theta)-\psi(\lambda)\right),
    \end{align*}
    and for any $\theta\in\Theta$ and $\lambda\in\LASET$, $\sum_{i=1}^{n}\lambda_{i}f_{i}(\theta^{\star})-\psi(\lambda)\le \sum_{i=1}^{n}\lambda_{i}^{\star}f_{i}(\theta^{\star})-\psi(\lambda^{\star})\le \sum_{i=1}^{n}\lambda_{i}^{\star}f_{i}(\theta)-\psi(\lambda^{\star})$.
\end{proposition}
\begin{proof}
    (a) We first show that $\sum_{i=1}^{n}\lambda_{i}f_{i}(\theta)-\psi(\lambda)$ is strongly convex in $\theta\in\Theta$ for any $\lambda\in\LASET$.
    We only consider the case (i) since (ii) is parallel.
    Lemma \ref{lem:min:weightedMean:general} along with Assumption \ref{assump:risk} verifies that for any $\theta,\theta'\in\Theta$ and $\lambda\in\LASET$,
    \begin{align*}
        \sum_{i=1}^{n}\lambda_{i}f_{i}\left(\theta\right)-\sum_{i=1}^{n}\lambda_{i}f_{i}\left(\theta'\right)-\sum_{i=1}^{n}\lambda_{i}\left\langle\nabla f_{i}\left(\theta'\right),\theta-\theta'\right\rangle\ge \frac{1}{2}\left(m_{f}-\frac{M_{f}-m_{f}}{2}\left(\left\|\LASET\right\|_{1}-1\right)\right)\left\|\theta-\theta'\right\|_{2}^{2}.
    \end{align*}
    The condition $\|\LASET\|_{1}<1+2m_{f}/(M_{f}-m_{f})$ leads to the desired strong convexity.

    (b) Assumption \ref{assump:sets} and (a) guarantee that for all $\lambda\in\LASET$, there exists a unique $\theta^{\star}(\lambda)\in\Theta$ such that $\sum_{i=1}^{n}\lambda_{i}f_{i}(\theta^{\star}(\lambda))-\psi(\lambda)=\min_{\theta\in\Theta}\sum_{i=1}^{n}\lambda_{i}f_{i}(\theta)-\psi(\lambda)$.
    Besides, consider arbitrary $(\theta^{\star,1},\lambda^{\star,1}),(\theta^{\star,2},\lambda^{\star,2})\in\Theta\times\LASET$ such that for $\ell\in\{1,2\}$,
    \begin{align*}
        \sum_{i=1}^{n}\lambda_{i}^{\star,\ell}f_{i}(\theta^{\star,\ell})-\psi(\lambda^{\star,\ell})=\min_{\theta\in\Theta}\max_{\lambda\in\LASET}\left(\sum_{i=1}^{n}\lambda_{i}f_{i}(\theta)-\psi(\lambda)\right)
    \end{align*}
    and for any $\theta\in\Theta$ and $\lambda\in\LASET$, $\sum_{i=1}^{n}\lambda_{i}f_{i}(\theta^{\star,\ell})-\psi(\lambda)\le \sum_{i=1}^{n}\lambda_{i}^{\star,\ell}f_{i}(\theta^{\star,\ell})-\psi(\lambda^{\star,\ell})\le \sum_{i=1}^{n}\lambda_{i}^{\star,\ell}f_{i}(\theta)-\psi(\lambda^{\star,\ell})$.
    We first assume that $\theta^{\star,1}\neq\theta^{\star,2}$.
    Since $\theta^{\star,1}$ is the unique minimizer of $\sum_{i=1}^{n}\lambda_{i}^{\star,1}f_{i}(\theta)-\psi(\lambda^{\star,1})$, we see that $\sum_{i=1}^{n}\lambda_{i}^{\star,1}f_{i}(\theta^{\star,2})-\psi(\lambda^{\star,1})>\sum_{i=1}^{n}\lambda_{i}^{\star,2}f_{i}(\theta^{\star,2})-\psi(\lambda^{\star,2})$.
    This, however, contradicts the definition of $\lambda^{\star,2}$, the maximizer of $\sum_{i=1}^{n}\lambda_{i}f_{i}(\theta^{\star,2})-\psi(\lambda)$.
    Therefore, $\theta^{\star,1}=\theta^{\star,2}$ and this gives the desired conclusion.
\end{proof}

\section{Convergence Guarantee of \textsc{Scaff-PD-IA}}

We consider the following problem more general than \eqref{eq:problem:scalarized}:
\begin{equation}\label{eq:problem:lambda:general}
    \min_{\theta\in\Theta}\max_{\lambda\in\LASET}\left(\sum_{i=1}^{n}\lambda_{i}f_{i}(\theta)-\psi(\lambda)\right),
\end{equation}
where $\psi$ is a real-valued convex function on $\LASET$ with $M_{\psi}$-smoothness ($M_{\psi}>0$) and $m_{\psi}$-strong-convexity ($m_{\psi}\ge0$; possibly $m_{\psi}=0$), that is, $(m_{\psi}/2)\|\lambda-\lambda'\|_{2}^{2}\le \psi(\lambda)-\psi(\lambda')-\langle \nabla \psi(\lambda'),\lambda-\lambda'\rangle\le (M_{\psi}/2)\|\lambda-\lambda'\|_{2}^{2}$ for all $\lambda,\lambda'\in\LASET$.
The \textsc{Scaff-PD-IA} algorithm with the regularization as \citet{yu2023scaff} is then given as Algorithm \ref{alg:scaffpd:reg}.
Though the representation of the update of $\lambda^{r}$ is different to that of Algorithm \ref{alg:scaffpd}, they are equivalent; the representation of Algorithm \ref{alg:scaffpd} is practical for implementation, and that of Algorithm \ref{alg:scaffpd:reg} is convenient for theoretical analysis.
Let $\Phi(\theta,\lambda):=\sum_{i=1}^{n}\lambda_{i}f_{i}(\theta)$ and $F(\theta,\lambda):=\Phi(\theta,\lambda)-\psi(\lambda)$ respectively.

\begin{algorithm}
\caption{\textsc{Scaff-PD-IA} with regularization}\label{alg:scaffpd:reg}
        \begin{algorithmic}
        \Require Set model parameters $(\theta^{1},\lambda^{1})=(\theta^{0},\lambda^{0})$
        \For{$r \gets 1$ to $R$}
            \State Set hyperparameters $\{\tau_{r},\sigma_{r},\varsigma_{r},\eta_{r}\}$
            \For{$i \gets 1$ to $n$}
                \State $L_{i}^{r}\gets f_{i}(\theta^{r})$, $c_{i}^{r}\gets\nabla f_{i}(\theta^{r},\xi_{i,0})$, send $(L_{i}^{r},c_{i}^{r})$ to the server
            \EndFor
            \State $s^{r}\gets(1+\varsigma_{r})(L_{1}^{r},\ldots,L_{n}^{r})-\varsigma_{r}(L_{1}^{r-1},\ldots,L_{n}^{r-1})$
            \State $\lambda^{r+1}\gets\argmin_{\lambda\in\LASET}\{\psi(\lambda)-\langle s^{r},\lambda\rangle +\frac{1}{2\sigma_{r}}\|\lambda-\lambda^{r}\|_{2}^{2}\}$
            \State $c^{r}\gets\sum_{i=1}^{n}\lambda_{i}^{r+1}c_{i}^{r}$, send $c^{r}$ to each client
            \For{$i \gets 1$ to $n$}
                \State $\Delta u_{i}^{r}\gets\textsc{Local-stochastic-update}(\theta^{r},c_{i}^{r},c^{r},r)$, send $\Delta u_{i}^{r}$ to the server
            \EndFor
            \State $\theta^{r+1}\gets\argmin_{\theta\in\Theta}\{\langle\sum_{i=1}^{N}\lambda_{i}^{r+1}\Delta u_{i}^{r},\theta\rangle+\frac{1}{2\tau_{r}}\|\theta-\theta^{r}\|_{2}^{2}\}$
        \EndFor
        \Ensure $(\theta^{R+1},\lambda^{R+1})$
    \end{algorithmic}
    \end{algorithm}

\subsection{Additional Assumption}
We set the following assumption on $\|\LASET\|_{1}$ and hyperparameters in Algorithm \ref{alg:scaffpd:reg}.

\begin{assumption}\label{assump:hyperparameters}
    Under Assumptions \ref{assump:risk} and \ref{assump:smoothInBoth},
    $\|\LASET\|_{1}<1+\frac{m_{f}}{2M_{f}}$, $\gamma_{0},\tau_{0}>0$, and the hyperparameters $\{\sigma_{r}:r\in[R]_{0}\}, \{\gamma_{r}:r\in[R]_{0}\}, \{\tau_{r}:r\in[R]_{0}\}, \{\varsigma_{r}:r\in[R]\}$, and $\{\eta_{r}:r\in[R]\}$ satisfies the conditions that for some $\beta\ge \max\{1,2\tau_{0}M_{f}\}$ for all $r\in[R]_{0}$,
    \begin{align}
        \sigma_{r}&=\gamma_{r}\tau_{r},\\
        \varsigma_{r+1}&=\sigma_{r}/\sigma_{r+1},\\
        \gamma_{r+1}&=\gamma_{r}\left(1+\left(\frac{m_{f}}{4}-\frac{M_{f}(\|\LASET\|_{1}-1)}{2}\right)\tau_{r}\right),\\
        \tau_{r+1}&=\tau_{r}\sqrt{\gamma_{r}/\gamma_{r+1}},\\
        \eta_{r}&=\tau_{r}/J\beta.
    \end{align}
    Moreover, $\tau_{0},\gamma_{0}$ are sufficiently small so that the following inequalities hold for some $c_{\alpha}\in(0,1)$:
    \begin{align}
        \tau_{0}\le  \min\left\{\frac{1}{1200M_{f}\|\LASET\|_{1}},\frac{18(1+\sqrt{3})}{m_{f}-2M_{f}(\|\LASET\|_{1}-1)}\right\},\ 
        \frac{4L_{\lambda\theta}^{2}\tau_{0}^{2}\gamma_{0}}{c_{\alpha}}+27M_{f}\tau_{0}\le 1.
    \end{align}
\end{assumption}

\subsection{Technical Lemmas}

The following lemma is an extension to Lemma A.1 of \citet{yu2023scaff}.
\begin{lemma}[perturbed strong convexity and smoothness]\label{lem:a1}
For any $L$-smooth $\mu$-strongly-convex $f:\R^{d}\to\R$, for any $x,y,z\in\R^{d}$,
\begin{align*}
     f(z)-f(y)+\frac{\mu}{4}\left\|y-z\right\|_{2}^{2}-L\left\|z-x\right\|_{2}^{2}&\le \left\langle \nabla f(x),z-y\right\rangle \\
     &\le  f(z)-f(y)+\left(L-\frac{\mu}{2}\right)\left\|x-z\right\|_{2}^{2}+L\left\|y-z\right\|_{2}^{2}
\end{align*}
    
\end{lemma}
\begin{proof}
    The first inequality follows from Lemma 5 of \citet{karimireddy2020scaffold}.

    On the second inequality, note that
    \begin{align*}
        \left\langle \nabla f(x),z-x\right\rangle&\le f(z)-f(x)-\frac{\mu}{2}\left\|x-z\right\|_{2}^{2}\\
        \left\langle \nabla f(x),x-y\right\rangle&\le f(x)-f(y)+\frac{L}{2}\left\|x-y\right\|_{2}^{2}
    \end{align*}
    and thus
    \begin{align*}
        \left\langle \nabla f(x),z-y\right\rangle &\le f(z)-f(y)-\frac{\mu}{2}\left\|x-z\right\|_{2}^{2}+\frac{L}{2}\left\|x-y\right\|_{2}^{2}\\
        &\le f(z)-f(y)-\frac{\mu}{2}\left\|x-z\right\|_{2}^{2}+L\left\|x-z\right\|_{2}^{2}+L\left\|y-z\right\|_{2}^{2}.
    \end{align*}
    Hence the statement holds true.
\end{proof}

Let us define $\mathcal{E}^{r}$ as follows:
\begin{equation*}
    \mathcal{E}^{r}:=\Ep\left[\frac{1}{J}\sum_{i=1}^{N}\sum_{j=1}^{J}\left|\lambda_{i}^{r+1}\right|\left\|u_{i,j-1}^{r}-\theta^{r}\right\|_{2}^{2}\right].
\end{equation*}
We also give notation $g_{i,j}^{r}(\theta)=\nabla_{\theta}f_{i}(\theta,\xi_{i,j}^{r})$ for all $i\in[n]$, $j\in [J]$, $r\in\N$, and $\theta\in\R^{d}$.

The following lemma is an extension to Lemma A.2 of \citet{yu2023scaff}.
\begin{lemma}\label{lem:a2}
    Fix $J\ge 2$.
    If for some $\beta\ge1$, $\tau_{r}=J\eta_{r}\beta$, $4\tau_{r}^{2}M_{f}^{2}\le \beta^{2}$, then
    \begin{align*}
        \mathcal{E}^{r}\le \left\|\LASET\right\|_{1}\left(\frac{24\tau_{r}^{2}}{\beta^{2}}\Ep\left[\left\|\nabla_{\theta}\Phi\left(\theta^{r},\lambda^{r+1}\right)\right\|_{2}^{2}\right]+\frac{24\tau_{r}^{2}\delta_{g}^{2}}{\beta^{2}}\left(1+\left\|\LASET\right\|_{2}^{2}\right)+6J\eta_{r}^{2}\delta_{g}^{2}\right).
    \end{align*}
\end{lemma}

\begin{proof}
    In this proof, we let $\tilde{\Ep}[\cdot]=\Ep[\cdot|\theta^{r},\lambda^{r+1}]$.
    We abbreviate the superscript $r$ denoting the round.
    Recall the following definition of $u_{i,j}$ ($j\ge1$):
    \begin{equation*}
        u_{i,j}=u_{i,j-1}-\eta_{r}\left(g_{i,j}\left(u_{i,j-1}\right)-\hat{c}_{i}+\hat{c}\right),
    \end{equation*}
    where
    \begin{align*}
        \tilde{\Ep}\left[g_{i,j}\left(u_{i,j-1}\right)\right]&=\nabla f_{i}\left(u_{i,j-1}\right),\\
        \tilde{\Ep}\left[\hat{c}_{i}\right]&=\nabla f_{i}(\theta)=c_{i},\\
        \tilde{\Ep}\left[\hat{c}\right]&=\sum_{i=1}^{n}\lambda_{i}\nabla f_{i}\left(\theta\right)=c.
    \end{align*}
    Therefore, we obtain
    \begin{align*}
        \tilde{\Ep}\left[\left\|u_{i,j}-\theta\right\|_{2}^{2}\right]&=\tilde{\Ep}\left[\left\|u_{i,j-1}-\theta-\eta_{r}\left(g_{i,j}\left(u_{i,j-1}\right)-\hat{c}_{i}+\hat{c}\right)\right\|_{2}^{2}\right]\\
        &\le \tilde{\Ep}\left[\left\|u_{i,j-1}-\theta-\eta_{r}\left(\nabla f_{i}\left(u_{i,j-1}\right)-\hat{c}_{i}+\hat{c}\right)\right\|_{2}^{2}\right]+\eta_{r}^{2}\delta_{g}^{2}\\
        &\le \left(1+\frac{1}{J-1}\right)\tilde{\Ep}\left[\left\|u_{i,j-1}-\theta\right\|_{2}^{2}\right]+J\eta_{r}^{2}\tilde{\Ep}\left[\left\|\nabla f_{i}\left(u_{i,j-1}\right)-\hat{c}_{i}+\hat{c}\right\|_{2}^{2}\right]+\eta_{r}^{2}\delta_{g}^{2}\\
        &= \left(1+\frac{1}{J-1}\right)\tilde{\Ep}\left[\left\|u_{i,j-1}-\theta\right\|_{2}^{2}\right]+\frac{\tau_{r}^{2}}{\beta^{2}J}\tilde{\Ep}\left[\left\|\nabla f_{i}\left(u_{i,j-1}\right)-\hat{c}_{i}+\hat{c}\right\|_{2}^{2}\right]+\eta_{r}^{2}\delta_{g}^{2}\\
        &\le \left(1+\frac{1}{J-1}\right)\tilde{\Ep}\left[\left\|u_{i,j-1}-\theta\right\|_{2}^{2}\right]+\frac{2\tau_{r}^{2}}{\beta^{2}J}\tilde{\Ep}\left[\left\|\nabla f_{i}\left(u_{i,j-1}\right)-c_{i}+c\right\|_{2}^{2}\right]\\
        &\quad+\frac{2\tau_{r}^{2}}{\beta^{2}J}\tilde{\Ep}\left[\left\|\hat{c}_{i}-\hat{c}-c_{i}+c\right\|_{2}^{2}\right]+\eta_{r}^{2}\delta_{g}^{2}\\
        &\le \left(1+\frac{1}{J-1}\right)\tilde{\Ep}\left[\left\|u_{i,j-1}-\theta\right\|_{2}^{2}\right]+\frac{4\tau_{r}^{2}}{\beta^{2}J}\tilde{\Ep}\left[\left\|\nabla f_{i}\left(u_{i,j-1}\right)-c_{i}\right\|_{2}^{2}\right]+\frac{4\tau_{r}^{2}}{\beta^{2}J}\tilde{\Ep}\left[\left\|c\right\|_{2}^{2}\right]\\
        &\quad+\underbrace{\frac{4\tau_{r}^{2}}{\beta^{2}J}\tilde{\Ep}\left[\left\|\hat{c}_{i}-c_{i}\right\|_{2}^{2}\right]+\frac{4\tau_{r}^{2}}{\beta^{2}J}\tilde{\Ep}\left[\left\|\hat{c}-c\right\|_{2}^{2}\right]+\eta_{r}^{2}\delta_{g}^{2}}_{=:\Gamma}\\
        &\le \left(1+\frac{1}{J-1}+\frac{4\tau_{r}^{2}M_{f}^{2}}{\beta^{2}J}\right)\tilde{\Ep}\left[\left\|u_{i,j-1}-\theta\right\|_{2}^{2}\right]+\frac{4\tau_{r}^{2}}{\beta^{2}J}\tilde{\Ep}\left[\left\|c\right\|_{2}^{2}\right]+\Gamma\\
        &\le \left(1+\frac{2}{J-1}\right)\tilde{\Ep}\left[\left\|u_{i,j-1}-\theta\right\|_{2}^{2}\right]+\frac{4\tau_{r}^{2}}{\beta^{2}J}\tilde{\Ep}\left[\left\|c\right\|_{2}^{2}\right]+\Gamma.
    \end{align*}
    Therefore, it holds that 
    \begin{align*}
        \tilde{\Ep}\left[\left\|u_{i,j}-\theta\right\|_{2}^{2}\right]&\le \sum_{i=0}^{j-1}\left(1+\frac{2}{J-1}\right)^{i}\left(\frac{4\tau_{r}^{2}}{\beta^{2}J}\tilde{\Ep}\left[\left\|c\right\|_{2}^{2}\right]+\Gamma\right)\\
        &\le \frac{(1+2/(J-1))^{j}-1}{2/(J-1)}\left(\frac{4\tau_{r}^{2}}{\beta^{2}J}\tilde{\Ep}\left[\left\|c\right\|_{2}^{2}\right]+\Gamma\right)\\
        &\le \frac{J-1}{2}\left(\left(1+\frac{2}{J-1}\right)\left(1+\frac{2}{J-1}\right)^{J-1}-1\right)\left(\frac{4\tau_{r}^{2}}{\beta^{2}J}\tilde{\Ep}\left[\left\|c\right\|_{2}^{2}\right]+\Gamma\right)\\
        &\le \frac{J-1}{2}\left(\left(1+\frac{2}{J-1}\right)e^{2}-1\right)\left(\frac{4\tau_{r}^{2}}{\beta^{2}J}\tilde{\Ep}\left[\left\|c\right\|_{2}^{2}\right]+\Gamma\right)\\
        &= \left(\left(\frac{J+1}{2}\right)e^{2}-\frac{J-1}{2}\right)\left(\frac{4\tau_{r}^{2}}{\beta^{2}J}\tilde{\Ep}\left[\left\|c\right\|_{2}^{2}\right]+\Gamma\right)\\
        &\le  \frac{24\tau_{r}^{2}}{\beta^{2}}\tilde{\Ep}\left[\left\|c\right\|_{2}^{2}\right]+6J\Gamma
    \end{align*}
    since
    \begin{align*}
        \left(\frac{J+1}{2}\right)e^{2}-\frac{J-1}{2}=\frac{J}{2}\left(e^{2}-1+\frac{1}{J}\left(e^{2}+1\right)\right)\le \frac{J}{4}\left(3e^{2}-1\right)\le 6J.
    \end{align*}
    Moreover, we obtain the following upper bound for $\Gamma$:
    \begin{align*}
        \Gamma&=\frac{4\tau_{r}^{2}}{\beta^{2}J}\tilde{\Ep}\left[\left\|\hat{c}_{i}-c_{i}\right\|_{2}^{2}\right]+\frac{4\tau_{r}^{2}}{\beta^{2}J}\tilde{\Ep}\left[\left\|\hat{c}-c\right\|_{2}^{2}\right]+\eta_{r}^{2}\delta_{g}^{2}\le \frac{4\tau_{r}^{2}\delta_{g}^{2}}{\beta^{2}J}\left(1+\left\|\LASET\right\|_{2}^{2}\right)+\eta_{r}^{2}\delta_{g}^{2}.
    \end{align*}
    Therefore, we obtain
    \begin{align*}
        &\tilde{\Ep}\left[\sum_{i=1}^{N}\left|\lambda_{i}\right|\sum_{j=1}^{J}\frac{1}{J}\left\|u_{i,j}-\theta\right\|_{2}^{2}\right]\le \sum_{i=1}^{N}\left|\lambda_{i}\right|\sum_{j=1}^{J}\frac{1}{J}\left(\frac{24\tau_{r}^{2}}{\beta^{2}}\tilde{\Ep}\left[\left\|c\right\|_{2}^{2}\right]+6J\Gamma\right)\\
        &\le \sum_{i=1}^{N}\left|\lambda_{i}\right|\sum_{j=1}^{J}\frac{1}{J}\left(\frac{24\tau_{r}^{2}}{\beta^{2}}\tilde{\Ep}\left[\left\|\nabla_{\theta}\Phi\left(\theta,\lambda\right)\right\|_{2}^{2}\right]+\frac{24\tau_{r}^{2}\delta_{g}^{2}}{\beta^{2}}\left(1+\left\|\LASET\right\|_{2}^{2}\right)+6J\eta_{r}^{2}\delta_{g}^{2}\right).
    \end{align*}
    The tower property of conditional expectations yields the conclusion.
\end{proof}

The following two lemmas are Lemmas A.3 and A.4 of \citet{yu2023scaff}.
\begin{lemma}[Lemma A.3 of \citealp{yu2023scaff}]\label{lem:a3}
    Under the same assumption as Lemma \ref{lem:a2} and Assumption \ref{assump:noprojection},
    it holds that
    \begin{align*}
        \frac{1}{\tau_{r}}\Ep\left[\left\|\theta^{r+1}-\theta^{r}\right\|_{2}^{2}\right]\ge -\tau_{r}M_{f}^{2}\mathcal{E}^{r}+\frac{\tau_{r}}{2}\Ep\left[\left\|\nabla_{\theta}\Phi\left(\theta^{r},\lambda^{r+1}\right)\right\|_{2}^{2}\right].
    \end{align*}
\end{lemma}

\begin{lemma}[Lemma A.4 of \citealp{yu2023scaff}]\label{lem:a4}
    If $\tau_{r}=J\eta_{r}\beta$ with $\beta\ge1$, then for all $\lambda\in\LASET$, 
    \begin{align}
        &\psi(\lambda^{r+1})-\left\langle s^{r+1},\lambda^{r+1}\right\rangle \notag\\
        &\le \psi\left(\lambda\right)-\left\langle s^{r+1},\lambda\right\rangle +\frac{1}{\sigma_{r}}\left(D\left(\lambda,\lambda^{r}\right)-D\left(\lambda,\lambda^{r+1}\right)-D\left(\lambda^{r+1},\lambda^{r}\right)\right)-\frac{m_{\psi}}{2}D\left(\lambda^{r+1},\lambda\right).
    \end{align}
\end{lemma}
Let us set the following notation.
\begin{equation}
    \Delta\theta^{r}=\frac{1}{J}\sum_{i=1}^{n}\sum_{j=1}^{J}\lambda_{i}^{r+1}g_{i,j}\left(u_{i,j-1}^{r}\right),\ 
    \tilde{\Delta}\theta^{r}=\Ep\left[\Delta\theta^{r}\right]=\frac{1}{J}\sum_{i=1}^{n}\sum_{j=1}^{J}\Ep\left[\lambda_{i}^{r+1}\nabla f_{i}\left(u_{i,j-1}^{r}\right)\right].
\end{equation}

We extend Lemma A.5 of \citet{yu2023scaff}.
\begin{lemma}\label{lem:a5}
    Under the same assumption as Lemma \ref{lem:a2}, for all $\theta\in\Theta$,
    \begin{align*}
        \Ep\left[\left\langle \Delta \theta^{r},\theta^{r+1}-\theta\right\rangle\right] \le \frac{1}{\tau_{r}}\Ep\left[D\left(\theta,\theta^{r}\right)-D\left(\theta,\theta^{r+1}\right)-D\left(\theta^{r+1},\theta^{r}\right)\right].
    \end{align*}
    In addition, let us consider the following decomposition:
    \begin{align*}
         \Ep\left[\left\langle \Delta \theta^{r},\theta^{r+1}-\theta\right\rangle\right]&= \underbrace{\Ep\left[\left\langle \Delta \theta^{r},\theta^{r}-\theta\right\rangle\right]}_{\mathcal{T}_{1}}+ \underbrace{\Ep\left[\left\langle \tilde{\Delta} \theta^{r},\theta^{r+1}-\theta^{r}\right\rangle\right]}_{\mathcal{T}_{2}}\\
         &\quad+ \underbrace{\Ep\left[\left\langle \Delta \theta^{r}-\tilde{\Delta}\theta^{r},\theta^{r+1}-\theta^{r}\right\rangle\right]}_{\mathcal{T}_{3}}
    \end{align*}
    Then the following inequalities hold:
    \begin{align*}
        \mathcal{T}_{1}&\ge \Ep\left[\Phi\left(\theta^{r},\lambda^{r+1}\right)-\Phi\left(\theta,\lambda^{r+1}\right)\right]+\left(\frac{m_{f}}{4}-\frac{M_{f}(\|\LASET\|_{1}-1)}{2}\right)\Ep\left[\left\|\theta^{r}-\theta\right\|_{2}^{2}\right]-M_{f}\mathcal{E}^{r},\\
        \mathcal{T}_{2}&\ge \Ep\left[\Phi\left(\theta^{r+1},\lambda^{r+1}\right)-\Phi\left(\theta^{r},\lambda^{r+1}\right)\right]-\frac{M_{f}(5\left\|\LASET\right\|_{1}-1)}{2}\Ep\left[\left\|\theta^{r+1}-\theta^{r}\right\|_{2}^{2}\right]-2M_{f}\mathcal{E}^{r}\\
        \mathcal{T}_{3}&\ge -\frac{2\left\|\LASET\right\|_{2}^{2}\tau_{r}\delta_{g}^{2}}{J}-\frac{1}{ 2\tau_{r}}\Ep\left[D\left(\theta^{r+1},\theta^{r}\right)\right].
    \end{align*}
\end{lemma}

\begin{proof}
    The first inequality follows from \citet{tseng2008accelerated}.
    The bound for $\mathcal{T}_{3}$ is same as \citet{yu2023scaff}:
    \begin{align*}
        |\mathcal{T}_{3}|&\le \Ep\left[\left|\left\langle\Delta \theta^{r}-\tilde{\Delta}\theta^{r},\theta^{r+1}-\theta^{r}\right\rangle\right|\right]\le \tau_{r}\Ep\left[\left\|\Delta \theta^{r}-\tilde{\Delta}\theta^{r}\right\|_{2}^{2}\right]+\frac{1}{4\tau_{r}}\Ep\left[\left\|\theta^{r+1}-\theta^{r}\right\|_{2}^{2}\right]\\
        &\le \frac{2\|\LASET\|_{2}^{2}\tau_{r}\delta_{g}^{2}}{J}+\frac{1}{2\tau_{r}}\Ep\left[D\left(\theta^{r+1},\theta^{r}\right)\right].
    \end{align*}
    In the next place, we can evaluate $\mathcal{T}_{1}$ as follows:
    \begin{align*}
        \mathcal{T}_{1}&=\Ep\left[\left\langle \Delta \theta^{r},\theta^{r}-\theta\right\rangle \right]\\
        &=\Ep\left[\sum_{i=1}^{N}\lambda_{i}^{r+1}\sum_{j=1}^{J}\frac{1}{J}\left\langle \nabla f_{i}\left(u_{i,j-1}^{r}\right),\theta^{r}-\theta\right\rangle\right]\\
        &\ge \Ep\left[\sum_{i:\lambda_{i}^{r+1}\ge0}\lambda_{i}^{r+1}\sum_{j=1}^{J}\frac{1}{J}\left(f_{i}\left(\theta^{r}\right)-f_{i}\left(\theta\right)+\frac{m_{f}}{4}\left\|\theta^{r}-\theta\right\|_{2}^{2}-M_{f}\left\|u_{i,j-1}^{r}-\theta^{r}\right\|_{2}^{2}\right)\right]\\
        &\quad+\Ep\left[\sum_{i:\lambda_{i}^{r+1}<0}\lambda_{i}^{r+1}\sum_{j=1}^{J}\frac{1}{J}\left(f_{i}\left(\theta^{r}\right)-f_{i}\left(\theta\right)+\left(M_{f}-\frac{m_{f}}{2}\right)\left\|\theta^{r}-\theta\right\|_{2}^{2}+M_{f}\left\|u_{i,j-1}^{r}-\theta^{r}\right\|_{2}^{2}\right)\right]\\
        &= \Ep\left[\Phi\left(\theta^{r},\lambda^{r+1}\right)-\Phi\left(\theta,\lambda^{r+1}\right)\right]\\
        &\quad+\Ep\left[\left(\frac{m_{f}}{4}\sum_{i=1}^{n}\lambda_{i}^{r+1}\mathbb{I}_{[0,\infty)}\left(\lambda_{i}^{r+1}\right)+\left(M_{f}-\frac{m_{f}}{2}\right)\sum_{i=1}^{n}\lambda_{i}^{r+1}\mathbb{I}_{(-\infty,0)}\left(\lambda_{i}^{r+1}\right)\right)\left\|\theta^{r}-\theta\right\|_{2}^{2}\right]\\
        &\quad-\Ep\left[\sum_{i=1}^{N}\left|\lambda_{i}^{r+1}\right|\sum_{j=1}^{J}\frac{1}{J}M_{f}\left\|u_{i,j-1}^{r}-\theta^{r}\right\|_{2}^{2}\right]\\
        &= \Ep\left[\Phi\left(\theta^{r},\lambda^{r+1}\right)-\Phi\left(\theta,\lambda^{r+1}\right)\right]+\Ep\left[\left(\frac{m_{f}}{4}+\left(M_{f}-\frac{3m_{f}}{4}\right)\sum_{i=1}^{n}\lambda_{i}^{r+1}\mathbb{I}_{(-\infty,0)}\left(\lambda_{i}^{r+1}\right)\right)\left\|\theta^{r}-\theta\right\|_{2}^{2}\right]\\
        &\quad-M_{f}\mathcal{E}^{r}\\
        &\ge \Ep\left[\Phi\left(\theta^{r},\lambda^{r+1}\right)-\Phi\left(\theta,\lambda^{r+1}\right)\right]+\left(\frac{m_{f}}{4}-\left(M_{f}-\frac{3m_{f}}{4}\right)\frac{\|\LASET\|_{1}-1}{2}\right)\Ep\left[\left\|\theta^{r}-\theta\right\|_{2}^{2}\right]-M_{f}\mathcal{E}^{r}.
    \end{align*}
    Note that we here used $\sum_{i=1}^{n}\lambda_{i}\mathbb{I}_{[0,\infty)}(\lambda_{i})= 1-\sum_{i=1}^{n}\lambda_{i}\mathbb{I}_{(-\infty,0)}(\lambda_{i})$ and Lemma \ref{lem:max:sign:general}.
    What is more, for $\mathcal{T}_{2}$, we have
    \begin{align*}
        \mathcal{T}_{2}&=\Ep\left[\left\langle \tilde{\Delta} \theta^{r},\theta^{r+1}-\theta^{r}\right\rangle\right]\\
        &=\Ep\left[\sum_{i=1}^{N}\lambda_{i}^{r+1}\sum_{j=1}^{J}\frac{1}{J}\left\langle \nabla f_{i}\left(u_{i,j-1}^{r}\right),\theta^{r+1}-\theta^{r}\right\rangle\right]\\
        &\ge \Ep\left[\sum_{i:\lambda_{i}^{r+1}\ge0}\lambda_{i}^{r+1}\sum_{j=1}^{J}\frac{1}{J}\left(f_{i}\left(\theta^{r+1}\right)-f_{i}\left(\theta^{r}\right)+\frac{m_{f}}{4}\left\|\theta^{r+1}-\theta^{r}\right\|_{2}^{2}-M_{f}\left\|u_{i,j-1}^{r}-\theta^{r+1}\right\|_{2}^{2}\right)\right]\\
        &\quad+\Ep\left[\sum_{i:\lambda_{i}^{r+1}<0}\lambda_{i}^{r+1}\sum_{j=1}^{J}\frac{1}{J}\left(f_{i}\left(\theta^{r+1}\right)-f_{i}\left(\theta^{r}\right)+\left(M_{f}-\frac{m_{f}}{2}\right)\left\|\theta^{r+1}-\theta^{r}\right\|_{2}^{2}+M_{f}\left\|u_{i,j-1}^{r}-\theta^{r+1}\right\|_{2}^{2}\right)\right]\\
        &\ge \Ep\left[\Phi\left(\theta^{r+1},\lambda^{r+1}\right)-\Phi\left(\theta^{r},\lambda^{r+1}\right)\right]\\
        &\quad+ \Ep\left[\sum_{i:\lambda_{i}^{r+1}\ge0}\lambda_{i}^{r+1}\left(\frac{m_{f}}{4}\left\|\theta^{r+1}-\theta^{r}\right\|_{2}^{2}-2M_{f}\left\|\theta^{r+1}-\theta^{r}\right\|_{2}^{2}\right)\right]\\
        &\quad+\Ep\left[\sum_{i:\lambda_{i}^{r+1}<0}\lambda_{i}^{r+1}\left(\left(M_{f}-\frac{m_{f}}{2}\right)\left\|\theta^{r+1}-\theta^{r}\right\|_{2}^{2}+2M_{f}\left\|\theta^{r+1}-\theta^{r}\right\|_{2}^{2}\right)\right]-2M_{f}\mathcal{E}^{r}\\
        &\ge \Ep\left[\Phi\left(\theta^{r+1},\lambda^{r+1}\right)-\Phi\left(\theta^{r},\lambda^{r+1}\right)\right]\\
        &\quad+\left(\frac{m_{f}}{4}-2M_{f}\left\|\LASET\right\|_{1}-\left(M_{f}-\frac{3m_{f}}{4}\right)\frac{\|\LASET\|_{1}-1}{2}\right)\Ep\left[\left\|\theta^{r+1}-\theta^{r}\right\|_{2}^{2}\right]-2M_{f}\mathcal{E}^{r}.
    \end{align*}
    Therefore the conclusion holds true.
\end{proof}

We extend Lemma B.1 of \citet{yu2023scaff} as follows:
\begin{lemma}\label{lem:b1}
    Under Assumptions \ref{assump:sets}--\ref{assump:hyperparameters}, for all $\theta\in\Theta,\lambda\in\LASET,r\in[R]$,
    \begin{align*}
        \Ep\left[F\left(\theta^{r+1},\lambda\right)-F\left(\theta,\lambda^{r+1}\right)\right]\le -Z_{r+1}+V_{r}+\Delta_{r}+\left(2\left\|\LASET\right\|_{2}^{2}+\frac{1+4\left(1+\left\|\LASET\right\|_{2}^{2}\right)}{50}\right)\tau_{r}\delta_{g}^{2},
    \end{align*}
    where 
    \begin{align*}
        Z_{r+1}&=\Ep\left[\left\langle q^{r+1},\lambda^{r+1}-\lambda\right\rangle +\frac{1}{2\sigma_{r}}\left\|\lambda^{r+1}-\lambda\right\|_{2}^{2}+\left(\frac{1}{2\tau_{r}}+\frac{m_{f}}{8}-\frac{M_{f}(\|\LASET\|_{1}-1)}{4}\right)\left\|\theta^{r+1}-\theta\right\|_{2}^{2}+\frac{1}{2\alpha_{r+1}}\left\|q^{r+1}\right\|_{2}^{2}\right],\\
        V_{r}&=\Ep\left[\varsigma_{r}\left\langle q^{r},\lambda^{r}-\lambda\right\rangle +\frac{1}{2\sigma_{r}}\left\|\lambda^{r}-\lambda\right\|_{2}^{2}+\frac{1}{2\tau_{r}}\left\|\theta-\theta^{r}\right\|_{2}^{2}+\frac{\varsigma_{r}}{2\alpha_{r}}\left\|q^{r}\right\|_{2}^{2}\right],\\
        \Delta_{r}&=\Ep\left[\left(\frac{\alpha_{r}\varsigma_{r}}{2}-\frac{1}{2\sigma_{r}}\right)\left\|\lambda^{r+1}-\lambda^{r}\right\|_{2}^{2}+\left(\frac{L_{\lambda\theta}^{2}}{2\alpha_{r+1}}+\frac{9M_{f}\|\LASET\|_{1}}{4}-\frac{1}{8\tau_{r}}\right)\left\|\theta^{r+1}-\theta^{r}\right\|_{2}^{2}\right].
    \end{align*}
    Here, $\alpha_{r}>0,r\in[R]$ are arbitrary positive numbers and $q^r$ is defined as
    \begin{align*}
        q^{r}=\nabla_{\lambda}\Phi\left(\theta^{r},\lambda^{r}\right)-\nabla_{\lambda}\Phi\left(\theta^{r-1},\lambda^{r-1}\right).
    \end{align*}
\end{lemma}

\begin{proof}
    We obtain the following inequalities:
    \begin{align*}
        &\Ep\left[\Phi\left(\theta^{r+1},\lambda\right)-\Phi\left(\theta,\lambda^{r+1}\right)\right]\\
        &\le \Ep\left[\Phi\left(\theta^{r+1},\lambda\right)-\Phi\left(\theta^{r},\lambda^{r+1}\right)\right]-\mathcal{T}_{2}-\mathcal{T}_{3}+M_{f}\mathcal{E}^{r}\\
        &\quad+\Ep\left[\frac{1}{\tau_{r}}\left(D(\theta,\theta^{r})-D(\theta,\theta^{r+1})-D(\theta^{r},\theta^{r+1})\right)\right]-\Ep\left[\left(\frac{m_{f}}{4}-\frac{M_{f}(\|\LASET\|_{1}-1)}{2}\right)\left\|\theta^{r}-\theta\right\|_{2}^{2}\right]\\
        &\le  \Ep\left[\Phi\left(\theta^{r+1},\lambda\right)-\Phi\left(\theta^{r+1},\lambda^{r+1}\right)+\Phi\left(\theta^{r+1},\lambda^{r+1}\right)-\Phi\left(\theta^{r},\lambda^{r+1}\right)\right]-\mathcal{T}_{2}+M_{f}\mathcal{E}^{r}+\frac{2\left\|\LASET\right\|_{2}^{2}\tau_{r}\delta_{g}^{2}}{J}\\
        &\quad+\Ep\left[\frac{1}{\tau_{r}}\left(D(\theta,\theta^{r})-D(\theta,\theta^{r+1})-\frac{1}{2}D(\theta^{r},\theta^{r+1})\right)\right]-\Ep\left[\left(\frac{m_{f}}{4}-\frac{M_{f}(\|\LASET\|_{1}-1)}{2}\right)\left\|\theta^{r}-\theta\right\|_{2}^{2}\right]\\
        &\le  \Ep\left[\Phi\left(\theta^{r+1},\lambda\right)-\Phi\left(\theta^{r+1},\lambda^{r+1}\right)+\Phi\left(\theta^{r+1},\lambda^{r+1}\right)-\Phi\left(\theta^{r},\lambda^{r+1}\right)\right]-\mathcal{T}_{2}+M_{f}\mathcal{E}^{r}+\frac{2\left\|\LASET\right\|_{2}^{2}\tau_{r}\delta_{g}^{2}}{J}\\
        &\quad+\underbrace{\Ep\left[\frac{1}{\tau_{r}}\left(D(\theta,\theta^{r})-D(\theta,\theta^{r+1})-\frac{1}{2}D(\theta^{r},\theta^{r+1})\right)\right]-\Ep\left[\left(\frac{m_{f}}{8}-\frac{M_{f}(\|\LASET\|_{1}-1)}{4}\right)\left\|\theta^{r+1}-\theta\right\|_{2}^{2}\right]}_{=:A^{r}}\\
        &\quad+\Ep\left[\left(\frac{m_{f}}{4}-\frac{M_{f}(\|\LASET\|_{1}-1)}{2}\right)\left\|\theta^{r+1}-\theta^{r}\right\|_{2}^{2}\right]\\
        &\le \Ep\left[\Phi\left(\theta^{r+1},\lambda\right)-\Phi\left(\theta^{r+1},\lambda^{r+1}\right)\right]+\frac{9M_{f}\|\LASET\|_{1}}{4}\Ep\left[\left\|\theta^{r+1}-\theta^{r}\right\|_{2}^{2}\right]\\
        & \quad +\frac{2\left\|\LASET\right\|_{2}^{2}\tau_{r}\delta_{g}^{2}}{J}+3M_{f}\mathcal{E}^{r}+A^{r}
    \end{align*}
    by Lemma \ref{lem:a5} in the first and second inequalities, and $-\|x-y\|_{2}^{2}\le -(1/2)\|x-z\|_{2}^{2}+\|y-z\|_{2}^{2}$ for all $x,y,z\in\R^{d}$ in the third inequality.
    The concavity of $\Phi$ with respect to $\lambda$ yields
    \begin{align*}
        \Phi\left(\theta^{r+1},\lambda\right)-\Phi\left(\theta^{r+1},\lambda^{r+1}\right)\le \left\langle \nabla_{\lambda}\Phi\left(\theta^{r+1},\lambda^{r+1}\right),\lambda-\lambda^{r+1}\right\rangle\text{ a.s.}
    \end{align*}
    Therefore, it holds
    \begin{align*}
        \Ep\left[\Phi\left(\theta^{r+1},\lambda\right)-\Phi\left(\theta,\lambda^{r+1}\right)\right]
        &\le \Ep\left[\left\langle \nabla_{\lambda}\Phi\left(\theta^{r+1},\lambda^{r+1}\right),\lambda-\lambda^{r+1}\right\rangle\right]\\
        &\quad+\frac{9M_{f}\|\LASET\|_{1}}{4}\Ep\left[\left\|\theta^{r+1}-\theta^{r}\right\|_{2}^{2}\right]\\
        &\quad+\frac{2\left\|\LASET\right\|_{2}^{2}\tau_{r}\delta_{g}^{2}}{J}+3M_{f}\mathcal{E}^{r}+A^{r}.
    \end{align*}
    Noting the following equalities
    \begin{align*}
        &\left\langle \nabla_{\lambda}\Phi\left(\theta^{r+1},\lambda^{r+1}\right),\lambda-\lambda^{r+1}\right\rangle+\left\langle s^{r},\lambda^{r+1}-\lambda\right\rangle\\
        &=-\left\langle \nabla_{\lambda}\Phi\left(\theta^{r+1},\lambda^{r+1}\right),\lambda^{r+1}-\lambda\right\rangle+\left\langle \left(1+\varsigma_{r}\right)\nabla_{\lambda}\Phi\left(\theta^{r},\lambda^{r}\right)-\varsigma_{r}\nabla_{\lambda}\Phi\left(\theta^{r-1},\lambda^{r-1}\right),\lambda^{r+1}-\lambda\right\rangle\\
        &=-\left\langle \nabla_{\lambda}\Phi\left(\theta^{r+1},\lambda^{r+1}\right)-\nabla_{\lambda}\Phi\left(\theta^{r},\lambda^{r}\right),\lambda^{r+1}-\lambda\right\rangle\\
        &\quad+\varsigma_{r}\left\langle \nabla_{\lambda}\Phi\left(\theta^{r},\lambda^{r}\right)-\nabla_{\lambda}\Phi\left(\theta^{r-1},\lambda^{r-1}\right),\lambda^{r+1}-\lambda\right\rangle\\
        &=-\left\langle q^{r+1},\lambda^{r+1}-\lambda\right\rangle+\varsigma_{r}\left\langle q^{r},\lambda^{r+1}-\lambda\right\rangle,
    \end{align*}
    and the inequality by Lemma \ref{lem:a4}
    \begin{align*}
        \Ep\left[\psi(\lambda^{r+1})-\left\langle s^{r+1},\lambda^{r+1}\right\rangle 
        \right]&\le \Ep\left[\psi\left(\lambda\right)-\left\langle s^{r+1},\lambda\right\rangle\right]\\
        &\quad+\underbrace{\frac{1}{\sigma_{r}}\Ep\left[D\left(\lambda,\lambda^{r}\right)-D\left(\lambda,\lambda^{r+1}\right)-D\left(\lambda^{r+1},\lambda^{r}\right)\right]}_{=:B^{r}},
    \end{align*}
    we have
    \begin{align*}
        &\Ep\left[F\left(\theta^{r+1},\lambda\right)-F\left(\theta,\lambda^{r+1}\right)\right]\\
        &=\Ep\left[\Phi\left(\theta^{r+1},\lambda\right)-\psi\left(\lambda\right)-\Phi\left(\theta,\lambda^{r+1}\right)+\psi\left(\lambda^{r+1}\right)\right]\\
        &\le \Ep\left[\left\langle \nabla_{\lambda}\Phi\left(\theta^{r+1},\lambda^{r+1}\right),\lambda-\lambda^{r+1}\right\rangle+\left\langle s^{r},\lambda^{r+1}-\lambda\right\rangle\right]\\
        &\quad+\frac{9M_{f}\|\LASET\|_{1}}{4}\Ep\left[\left\|\theta^{r+1}-\theta^{r}\right\|_{2}^{2}\right]+3M_{f}\mathcal{E}^{r}+A^{r}+B^{r}+\frac{2\left\|\LASET\right\|_{2}^{2}\tau_{r}\delta_{g}^{2}}{J}\\
        &= \Ep\left[-\left\langle q^{r+1},\lambda^{r+1}-\lambda\right\rangle+\varsigma_{r}\left\langle q^{r},\lambda^{r+1}-\lambda\right\rangle\right]\\
        &\quad+\frac{9M_{f}\|\LASET\|_{1}}{4}\Ep\left[\left\|\theta^{r+1}-\theta^{r}\right\|_{2}^{2}\right]+3M_{f}\mathcal{E}^{r}+A^{r}+B^{r}+\frac{2\left\|\LASET\right\|_{2}^{2}\tau_{r}\delta_{g}^{2}}{J}\\
        &= \Ep\left[-\left\langle q^{r+1},\lambda^{r+1}-\lambda\right\rangle\right]+\varsigma_{r}\Ep\left[\left\langle q^{r},\lambda^{r}-\lambda\right\rangle\right]+\varsigma_{r}\Ep\left[\left\langle q^{r},\lambda^{r+1}-\lambda^{r}\right\rangle\right]\\
        &\quad+\frac{9M_{f}\|\LASET\|_{1}}{4}\Ep\left[\left\|\theta^{r+1}-\theta^{r}\right\|_{2}^{2}\right]+3M_{f}\mathcal{E}^{r}+A^{r}+B^{r}+\frac{2\left\|\LASET\right\|_{2}^{2}\tau_{r}\delta_{g}^{2}}{J}.
    \end{align*}
    Moreover, the following evaluation holds almost surely:
    \begin{align*}
        \varsigma_{r}\left\langle q^{r},\lambda^{r+1}-\lambda^{r}\right\rangle 
        &\le\frac{\varsigma_{r}}{2\alpha_{r}}\left\|q^{r}\right\|_{2}^{2}+\frac{\alpha_{r}\varsigma_{r}}{2}\left\|\lambda^{r+1}-\lambda^{r}\right\|_{2}^{2}
    \end{align*}
    Hence, we obtain
    \begin{align*}
        &\Ep\left[F\left(\theta^{r+1},\lambda\right)-F\left(\theta,\lambda^{r+1}\right)\right]\\
        &\le \Ep\left[-\left\langle q^{r+1},\lambda^{r+1}-\lambda\right\rangle\right]+\varsigma_{r}\Ep\left[\left\langle q^{r},\lambda^{r}-\lambda\right\rangle\right]+\Ep\left[\frac{\varsigma_{r}}{2\alpha_{r}}\left\|q^{r}\right\|_{2}^{2}+\frac{\alpha_{r}\varsigma_{r}}{2}\left\|\lambda^{r+1}-\lambda^{r}\right\|_{2}^{2}\right]\\
        &\quad+\frac{9M_{f}\|\LASET\|_{1}}{4}\Ep\left[\left\|\theta^{r+1}-\theta^{r}\right\|_{2}^{2}\right]+3M_{f}\mathcal{E}^{r}+A^{r}+B^{r}+\frac{2\left\|\LASET\right\|_{2}^{2}\tau_{r}\delta_{g}^{2}}{J}\\
        &\le \Ep\left[-\left\langle q^{r+1},\lambda^{r+1}-\lambda\right\rangle\right]+\varsigma_{r}\Ep\left[\left\langle q^{r},\lambda^{r}-\lambda\right\rangle\right]+\Ep\left[\frac{\varsigma_{r}}{2\alpha_{r}}\left\|q^{r}\right\|_{2}^{2}+\frac{\alpha_{r}\varsigma_{r}}{2}\left\|\lambda^{r+1}-\lambda^{r}\right\|_{2}^{2}\right]\\
        &\quad+\frac{9M_{f}\|\LASET\|_{1}}{4}\Ep\left[\left\|\theta^{r+1}-\theta^{r}\right\|_{2}^{2}\right]+3M_{f}\mathcal{E}^{r}+\frac{2\left\|\LASET\right\|_{2}^{2}\tau_{r}\delta_{g}^{2}}{J}\\
        &\quad+\frac{1}{\tau_{r}}\Ep\left[D(\theta,\theta^{r})-D(\theta,\theta^{r+1})-\frac{1}{2}D(\theta^{r},\theta^{r+1})\right]-\left(\frac{m_{f}}{8}-\frac{M_{f}(\|\LASET\|_{1}-1)}{4}\right)\Ep\left[\left\|\theta^{r+1}-\theta\right\|_{2}^{2}\right]\\
        &\quad+\frac{1}{\sigma_{r}}\Ep\left[D\left(\lambda,\lambda^{r}\right)-D\left(\lambda,\lambda^{r+1}\right)-D\left(\lambda^{r+1},\lambda^{r}\right)\right]\\
        &\le -\underbrace{\Ep\left[\left\langle q^{r+1},\lambda^{r+1}-\lambda\right\rangle+\frac{1}{\sigma_{r}}D\left(\lambda,\lambda^{r+1}\right)+\frac{1}{\tau_{r}}D(\theta,\theta^{r+1})+\left(\frac{m_{f}}{8}-\frac{M_{f}(\|\LASET\|_{1}-1)}{4}\right)\left\|\theta^{r+1}-\theta\right\|_{2}^{2}+\frac{\left\|q^{r+1}\right\|_{2}^{2}}{2\alpha_{r+1}}\right]}_{=Z_{r+1}}\\
        &\quad+\underbrace{\Ep\left[\varsigma_{r}\left\langle q^{r},\lambda^{r}-\lambda\right\rangle+\frac{1}{\sigma_{r}}D\left(\lambda,\lambda^{r}\right)+\frac{1}{\tau_{r}}D(\theta,\theta^{r})+\frac{\varsigma_{r}}{2\alpha_{r}}\left\|q^{r}\right\|_{2}^{2}\right]}_{=V_{r}}\\
        &\quad+\frac{\alpha_{r}\varsigma_{r}}{2}\Ep\left[\left\|\lambda^{r+1}-\lambda^{r}\right\|_{2}^{2}\right]+\frac{2\left\|\LASET\right\|_{2}^{2}\tau_{r}\delta_{g}^{2}}{J}\\
        &\quad+\Ep\left[\frac{\left\|q^{r+1}\right\|_{2}^{2}}{2\alpha_{r+1}}+\frac{9M_{f}\|\LASET\|_{1}}{4}\left\|\theta^{r+1}-\theta^{r}\right\|_{2}^{2}+3M_{f}\mathcal{E}^{r}-\frac{1}{2\tau_{r}}D(\theta^{r},\theta^{r+1})-\frac{1}{\sigma_{r}}D\left(\lambda^{r+1},\lambda^{r}\right)\right]\\
        &\le -Z_{r+1}+V_{r}+\frac{\alpha_{r}\varsigma_{r}}{2}\Ep\left[\left\|\lambda^{r+1}-\lambda^{r}\right\|_{2}^{2}\right]+\frac{L_{\lambda\theta}^{2}}{2\alpha_{r+1}}\Ep\left[\left\|\theta^{r+1}-\theta^{r}\right\|_{2}^{2}\right]+\frac{2\left\|\LASET\right\|_{2}^{2}\tau_{r}\delta_{g}^{2}}{J}\\
        &\quad+\frac{9M_{f}\|\LASET\|_{1}}{4}\Ep\left[\left\|\theta^{r+1}-\theta^{r}\right\|_{2}^{2}\right]+3M_{f}\mathcal{E}^{r}+\Ep\left[-\frac{1}{2\tau_{r}}D(\theta^{r},\theta^{r+1})-\frac{1}{\sigma_{r}}D\left(\lambda^{r+1},\lambda^{r}\right)\right]\\
        &= -Z_{r+1}+V_{r}+\left(\frac{\alpha_{r}\varsigma_{r}}{2}-\frac{1}{2\sigma_{r}}\right)\Ep\left[\left\|\lambda^{r+1}-\lambda^{r}\right\|_{2}^{2}\right]\\
        &\quad+\left(\frac{L_{\lambda\theta}^{2}}{2\alpha_{r+1}}+\frac{9M_{f}\|\LASET\|_{1}}{4}-\frac{1}{8\tau_{r}}\right)\Ep\left[\left\|\theta^{r+1}-\theta^{r}\right\|_{2}^{2}\right]+2\left\|\LASET\right\|_{2}^{2}\tau_{r}\delta_{g}^{2}\\
        &\quad+\underbrace{3M_{f}\mathcal{E}^{r}-\frac{1}{8\tau_{r}}\Ep\left[\left\|\theta^{r+1}-\theta^{r}\right\|_{2}^{2}\right]}_{=\mathcal{T}_{4}}.
    \end{align*}

    We give an upper bound for $\mathcal{T}_{4}$.
    Lemmas \ref{lem:a2} and \ref{lem:a3} yields the following bound (note that $\beta\ge1$, $\tau_{r}\le (1200\|\LASET\|_{1}M_{f})^{-1}$):
    \begin{align*}
        \mathcal{T}_{4}&\le 3M_{f}\mathcal{E}^{r}+\frac{\tau_{r}M_{f}^{2}}{8}\mathcal{E}^{r}-\frac{\tau_{r}}{16}\Ep\left[\left\|\nabla_{\theta}\Phi\left(\theta^{r},\lambda^{r+1}\right)\right\|_{2}^{2}\right]\\
        &\le \left(3M_{f}+\frac{\tau_{r}M_{f}^{2}}{8}\right)\left\|\LASET\right\|_{1}\left(\frac{24\tau_{r}^{2}}{\beta^{2}}\Ep\left[\left\|\nabla_{\theta}\Phi\left(\theta^{r},\lambda^{r+1}\right)\right\|_{2}^{2}\right]+\frac{24\tau_{r}^{2}\delta_{g}^{2}}{\beta^{2}}\left(1+\left\|\LASET\right\|_{2}^{2}\right)+6J\eta_{r}^{2}\delta_{g}^{2}\right)\\
        &\quad-\frac{\tau_{r}}{16}\Ep\left[\left\|\nabla_{\theta}\Phi\left(\theta^{r},\lambda^{r+1}\right)\right\|_{2}^{2}\right]\\
        &=\left(\left\|\LASET\right\|_{1}\left(24+\tau_{r}M_{f}\right)\frac{3M_{f}\tau_{r}^{2}}{\beta^{2}}-\frac{\tau_{r}}{16}\right)\Ep\left[\left\|\nabla_{\theta}\Phi\left(\theta^{r},\lambda^{r+1}\right)\right\|_{2}^{2}\right]\\
        &\quad+\left\|\LASET\right\|_{1}\left(3M_{f}+\frac{\tau_{r}M_{f}^{2}}{8}\right)\left(\frac{24\tau_{r}^{2}\delta_{g}^{2}}{\beta^{2}}\left(1+\left\|\LASET\right\|_{2}^{2}\right)+6J\eta_{r}^{2}\delta_{g}^{2}\right)\\
        &\le \frac{\tau_{r}}{16}\left(1200\left\|\LASET\right\|_{1}M_{f}\tau_{r}-1\right)\Ep\left[\left\|\nabla_{\theta}\Phi\left(\theta^{r},\lambda^{r+1}\right)\right\|_{2}^{2}\right]\\
        &\quad+6\left\|\LASET\right\|_{1}\left(3+\frac{\tau_{r}M_{f}}{8}\right)\left(4\tau_{r}^{2}\left(1+\left\|\LASET\right\|_{2}^{2}\right)+J\eta_{r}^{2}\right)M_{f}\delta_{g}^{2}\\
        &\le 24\left\|\LASET\right\|_{1}\left(4\tau_{r}^{2}\left(1+\left\|\LASET\right\|_{2}^{2}\right)+J\left(\frac{\tau_{r}}{J}\right)^{2}\right)M_{f}\delta_{g}^{2}\\
        &\le 24\left\|\LASET\right\|_{1}\left(4\left(1+\left\|\LASET\right\|_{2}^{2}\right)+1\right)M_{f}\tau_{r}^{2}\delta_{g}^{2}\\
        &\le \frac{1+4\left(1+\left\|\LASET\right\|_{2}^{2}\right)}{50}\tau_{r}\delta_{g}^{2};
    \end{align*}
    Hence we obtain the conclusion.
\end{proof}

\subsection{Hyperparameter Setting}

We use the notation
$\{t_{r}:r\in[R]_{0}\}$ such that
    \begin{align}
        t_{r}=\sigma_{r}/\sigma_{0}.
    \end{align}

\begin{lemma}\label{lem:b2}
    Under Assumption \ref{assump:hyperparameters}, for all $r\in[R]_{0}$,
    \begin{align*}
        t_{r}\left(\frac{1}{\tau_{r}}+\frac{m_{f}}{4}-\frac{M_{f}(\|\LASET\|_{1}-1)}{2}\right)\ge \frac{t_{r+1}}{\tau_{r+1}},\ \frac{t_{r}}{\sigma_{r}}\ge\frac{t_{r+1}}{\sigma_{r+1}},\ \frac{t_{r}}{t_{r+1}}=\varsigma_{r+1}.
    \end{align*}
\end{lemma}

\begin{proof}
The second and third statements are trivial by the definitions.
Due to the updates of $\gamma_{r}$ and $\sigma_{r}$, for all $r\in[R]_{0}$,
    \begin{align*}
        1+\frac{m_{f}\tau_{r}}{4}-\frac{M_{f}(\|\LASET\|_{1}-1)\tau_{r}}{2}=\frac{\gamma_{r+1}}{\gamma_{r}}=\frac{\sigma_{r+1}}{\tau_{r+1}}\frac{\tau_{r}}{\sigma_{r}}.
    \end{align*}
    The first statement is equivalent to the following inequality, which holds for all $r\in[R]_{0}$ by the assumption:
    \begin{align*}
        1+\frac{m_{f}\tau_{r}}{4}-\frac{M_{f}(\|\LASET\|_{1}-1)\tau_{r}}{2}\ge\frac{\tau_{r}}{\tau_{r+1}}\frac{t_{r+1}}{t_{r}}=\frac{\tau_{r}}{\tau_{r+1}}\frac{\sigma_{r+1}}{\sigma_{r}}.
    \end{align*}
    Hence we obtain the conclusion.
\end{proof}

\begin{lemma}\label{lem:b3}
    Under Assumption \ref{assump:hyperparameters},
    we have
    \begin{align*}
        \frac{\tau_{r}}{\sigma_{r}}=\frac{1}{\gamma_{r}}=\mathcal{O}\left(\frac{1}{r^{2}}\right),\ \gamma_{r}=\Omega\left(r^{2}\right),\ \sigma_{r}=\Omega\left(r\right),\ \tau_{r}\sigma_{r}=\tau_{0}^{2}\gamma_{0}.
    \end{align*}
\end{lemma}

\begin{proof}
    As $\tau_{r+1}=\tau_{r}\sqrt{\gamma_{r}/\gamma_{r+1}}$, it holds that $\tau_{r}=\tau_{0}\sqrt{\gamma_{0}/\gamma_{r}}$.
    The update of $\gamma_{r}$ gives
    \begin{align}
        \gamma_{r+1}=\gamma_{r}\left(1+\frac{m_{f}\tau_{r}}{4}-\frac{M_{f}(\|\LASET\|_{1}-1)\tau_{r}}{2}\right)=\gamma_{r}+\left(\frac{m_{f}}{4}-\frac{M_{f}(\|\LASET\|_{1}-1)}{2}\right)\tau_{0}\sqrt{\gamma_{0}\gamma_{r}}.
    \end{align}
    
    We show that the following inequality holds for all $r\in\N_{0}$:
    \begin{align}
        \gamma_{r}\ge \left(\frac{r}{3}\left(\frac{m_{f}}{4}-\frac{M_{f}(\|\LASET\|_{1}-1)}{2}\right)\tau_{0}\sqrt{\gamma_{0}}\right)^{2}.
    \end{align}
    $\gamma_{0}\ge0$ is obvious.
    To use the deduction, we show the bound for $\gamma_{1}$.
    The assumption of $\tau_{0}$ yields
    \begin{align}
        \gamma_{1}=\gamma_{0}\left(1+\left(\frac{m_{f}}{4}-\frac{M_{f}(\|\LASET\|_{1}-1)}{2}\right)\tau_{0}\right)\ge \gamma_{0}\left(\frac{1}{3}\left(\frac{m_{f}}{4}-\frac{M_{f}(\|\LASET\|_{1}-1)}{2}\right)\tau_{0}\right)^{2}.
    \end{align}
    We assume that the inequality holds for $r=r'\in\N$.
    Then,
    \begin{align*}
        \gamma_{r'+1}&=\gamma_{r'}+\left(\frac{m_{f}}{4}-\frac{M_{f}(\|\LASET\|_{1}-1)}{2}\right)\tau_{0}\sqrt{\gamma_{0}\gamma_{r'}}\\
        &\ge \left(\frac{r'}{3}\right)^{2}\left(\left(\frac{m_{f}}{4}-\frac{M_{f}(\|\LASET\|_{1}-1)}{2}\right)\tau_{0}\sqrt{\gamma_{0}}\right)^{2}+\frac{r'}{3}\left(\left(\frac{m_{f}}{4}-\frac{M_{f}(\|\LASET\|_{1}-1)}{2}\right)\tau_{0}\sqrt{\gamma_{0}}\right)^{2}\\
        &\ge \left[\left(\frac{r'}{3}\right)^{2}+\frac{r'}{3}\right]\left(\left(\frac{m_{f}}{4}-\frac{M_{f}(\|\LASET\|_{1}-1)}{2}\right)\tau_{0}\sqrt{\gamma_{0}}\right)^{2}\\
        &\ge \left(\frac{r'+1}{3}\right)^{2}\left(\left(\frac{m_{f}}{4}-\frac{M_{f}(\|\LASET\|_{1}-1)}{2}\right)\tau_{0}\sqrt{\gamma_{0}}\right)^{2}.
    \end{align*}
    It gives
    \begin{align*}
        \sigma_{r}=\gamma_{r}\tau_{r}=\left(\frac{\gamma_{r+1}-\gamma_{r}}{m_{f}/4-M_{f}(\|\LASET\|_{1}-1)/2}\right)^{2}\ge \tau_{0}\sqrt{\gamma_{0}\gamma_{r}}\ge \frac{r}{3}\left(\frac{m_{f}}{4}-\frac{M_{f}(\|\LASET\|_{1}-1)}{2}\right)\tau_{0}^{2}\gamma_{0}
    \end{align*}
    and
    \begin{align*}
        \tau_{r}\sigma_{r}=\frac{\sigma_{r}^{2}}{\gamma_{r}}=\frac{\left(\gamma_{r+1}-\gamma_{r}\right)^{2}}{\left(m_{f}/4-M_{f}(\|\LASET\|_{1}-1)/2\right)^{2}\gamma_{r}}=\tau_{0}^{2}\gamma_{0}.
    \end{align*}
    Therefore, we have
    \begin{align*}
        \frac{\tau_{r}}{\sigma_{r}}=\mathcal{O}\left(\frac{1}{r^{2}}\right),\ \frac{1}{\gamma_{r}}=\mathcal{O}\left(\frac{1}{r^{2}}\right)
    \end{align*}
    and the conclusion holds.
\end{proof}

\begin{proposition}\label{prop:b5}
    Under Assumption \ref{assump:hyperparameters}, for all $r\in[R]_{0}$, $\Delta_{r}\le 0$.
\end{proposition}

\begin{proof}
    Let $\alpha_{r}=c_{\alpha}/\sigma_{r-1}$.
    We just need to show
    \begin{align*}
        \alpha_{r}\varsigma_{r}-\frac{1}{\sigma_{r}}\le 0,\ \frac{L_{\lambda\theta}^{2}}{2\alpha_{r+1}}+\frac{9M_{f}\|\LASET\|_{1}}{4}-\frac{1}{8\tau_{r}}\le 0.
    \end{align*}
    The first inequality holds immediately by the definitions of $\alpha_{r},\varsigma_{r},\sigma_{r}$ and the following inequality:
    \begin{align*}
        &\alpha_{r}\varsigma_{r}-\frac{1}{\sigma_{r}}=\frac{c_{\alpha}\varsigma_{r}}{\sigma_{r-1}}-\frac{1}{\sigma_{r}}=\frac{1}{\sigma_{r}}\left(c_{\alpha}-1\right)<0.
    \end{align*}
    The second inequality can be derived by Lemma \ref{lem:b3} and $\tau_{r}=\tau_{0}\sqrt{\gamma_{0}/\gamma_{r}}=\tau_{0}\prod_{s=0}^{r-1}\sqrt{\gamma_{s}/\gamma_{s+1}}\le \tau_{0}$ as follows:
    \begin{align*}
        \frac{L_{\lambda\theta}^{2}}{2\alpha_{r+1}}+\frac{9M_{f}\|\LASET\|_{1}}{4}-\frac{1}{8\tau_{r}}
        &\le \frac{1}{8\tau_{r}}\left(8\tau_{r}\left(\frac{L_{\lambda\theta}^{2}\sigma_{r}}{2c_{\alpha}}+\frac{9M_{f}}{4}\left(1+\frac{m_{f}}{2M_{f}}\right)\right)-1\right)\\
        &\le \frac{1}{8\tau_{r}}\left(8\left(\frac{L_{\lambda\theta}^{2}\tau_{0}^{2}\gamma_{0}}{2c_{\alpha}}+\frac{27}{8}M_{f}\tau_{r}\right)-1\right)\\
        &\le \frac{1}{8\tau_{r}}\left(\left(\frac{4L_{\lambda\theta}^{2}\tau_{0}^{2}\gamma_{0}}{c_{\alpha}}+27M_{f}\tau_{0}\right)-1\right)\\
        &\le 0.
    \end{align*}
    Hence, we obtain the conclusion.
\end{proof}
\subsection{Proof of Theorem \ref{thm:scaffpd}}

\begin{proof}[Proof of Theorem \ref{thm:scaffpd}]
    The uniqueness of $\theta^{\star}$ holds for the problem \eqref{eq:problem:lambda:general} by Proposition \ref{prop:unique:general} and the fact that $1+2m_{f}/M_{f}<1+2m_{f}/(M_{f}-m_{f})$ if $M_{f}>m_{f}$ (the case $M_{f}=m_{f}$ is trivial).
    
    $\varsigma_{r}$ satisfies
    \begin{equation}
        \varsigma_{r+1}=\frac{\sigma_{r}}{\sigma_{r+1}}=\frac{\tau_{r}\gamma_{r}}{\tau_{r+1}\gamma_{r+1}}=\sqrt{\frac{\gamma_{r}}{\gamma_{r+1}}}=\frac{1}{\sqrt{1+m_{f}\tau_{r}}},\ \tau_{r+1}=\tau_{r}\sqrt{\frac{\gamma_{r}}{\gamma_{r+1}}}=\varsigma_{r+1}\tau_{r};
    \end{equation}
    here, we used $\tau_{r+1}=\tau_{r}\sqrt{\gamma_{r}/\gamma_{r+1}}$.
    In the next place, we set $t_{r}$ and $\alpha_{r}$ as follows:
    \begin{align*}
        t_{r}=\sigma_{r}/\sigma_{0},\ \alpha_{r}=c_{\alpha}/\sigma_{r-1},
    \end{align*}
    where $c_{\alpha}\in(0,1)$ is an arbitrary constant.
    Then, Proposition \ref{prop:b5} yields $\Delta_{r}\le 0$.
    Moreover, Lemma \ref{lem:b2} gives
    \begin{align*}
        t_{r+1}V_{r+1}\le t_{r}Z_{r+1}
    \end{align*}
    Furthermore, the Cauchy--Schwarz inequality and $\gamma_{R+1}/\gamma_{R}\ge1$ lead to
    \begin{align*}
        Z_{R+1}&\ge -\frac{1}{2\alpha_{R+1}}\Ep\left[\left\|q^{R+1}\right\|_{2}^{2}\right]-\frac{\alpha_{R+1}}{2}\Ep\left[\left\|\lambda^{R+1}-\lambda\right\|_{2}^{2}\right] +\frac{1}{2\sigma_{R}}\Ep\left[\left\|\lambda^{R+1}-\lambda\right\|_{2}^{2}\right]\\
        &\quad+\left(\frac{1}{2\tau_{R}}+\frac{m_{f}}{8}-\frac{M_{f}(\|\LASET\|_{1}-1)}{4}\right)\Ep\left[\left\|\theta^{R+1}-\theta\right\|_{2}^{2}\right]+\frac{1}{2\alpha_{R+1}}\Ep\left[\left\|q^{R+1}\right\|_{2}^{2}\right]\\
        &=\left(\frac{1}{2\sigma_{R}}-\frac{c_{\alpha}}{2\sigma_{R}}\right)\Ep\left[\left\|\lambda^{R+1}-\lambda\right\|_{2}^{2}\right]+\frac{\gamma_{R+1}}{2\gamma_{R}\tau_{R}}\Ep\left[\left\|\theta^{R+1}-\theta\right\|_{2}^{2}\right]\\
        &\ge \frac{1}{2\tau_{R}}\Ep\left[\left\|\theta^{R+1}-\theta\right\|_{2}^{2}\right]
    \end{align*}
    Therefore, using $\theta^{1}=\theta^{0},\lambda^{1}=\lambda^{0}$, we obtain that for all $\theta\in\Theta$ and $\lambda\in\LASET$,
    \begin{align*}
        &\left(\sum_{r=0}^{R}t_{r}\right)\Ep\left[F\left(\Bar{\theta}^{R+1},\lambda\right)-F\left(\theta,\Bar{\lambda}^{R+1}\right)\right]+\frac{t_{R}}{\tau_{R}}\Ep\left[D\left(\theta^{R+1},\theta\right)\right]\\
        &\le\sum_{r=0}^{R}t_{r}\Ep\left[F\left(\theta^{r+1},\lambda\right)-F\left(\theta,\lambda^{r+1}\right)\right]+\frac{t_{R}}{\tau_{R}}\Ep\left[D\left(\theta^{R+1},\theta\right)\right]\\
        &= \sum_{r=1}^{R}t_{r}\Ep\left[F\left(\theta^{r+1},\lambda\right)-F\left(\theta,\lambda^{r+1}\right)\right]+t_{0}\left(F\left(\theta^{1},\lambda\right)-F\left(\theta,\lambda^{1}\right)\right)+\frac{t_{R}}{\tau_{R}}\Ep\left[D\left(\theta^{R+1},\theta\right)\right]\\
        &\le \sum_{r=1}^{R}t_{r}\left(-Z_{r+1}+V_{r}+\Delta_{r}+C\tau_{r}\delta_{g}^{2}\right)+t_{0}\left(F\left(\theta^{0},\lambda\right)-F\left(\theta,\lambda^{0}\right)\right)+\frac{t_{R}}{\tau_{R}}\Ep\left[D\left(\theta^{R+1},\theta\right)\right]\\
        &\le \sum_{r=1}^{R}t_{r}\left(-Z_{r+1}+V_{r}+C\tau_{r}\delta_{g}^{2}\right)+t_{0}\left(F\left(\theta^{0},\lambda\right)-F\left(\theta,\lambda^{0}\right)\right)+\frac{t_{R}}{\tau_{R}}\Ep\left[D\left(\theta^{R+1},\theta\right)\right]\\
        &\le \sum_{r=1}^{R}t_{r}\left(-Z_{r+1}+V_{r}\right)+\sum_{r=1}^{R}t_{r}C\tau_{r}\delta_{g}^{2}+t_{0}\left(F\left(\theta^{0},\lambda\right)-F\left(\theta,\lambda^{0}\right)\right)+\frac{t_{R}}{\tau_{R}}\Ep\left[D\left(\theta^{R+1},\theta\right)\right]\\
        &\le-t_{R}Z_{R+1}+t_{0}V_{1}+\sum_{r=1}^{R}t_{r}C\tau_{r}\delta_{g}^{2}+t_{0}\left(F\left(\theta^{0},\lambda\right)-F\left(\theta,\lambda^{0}\right)\right)+\frac{t_{R}}{\tau_{R}}\Ep\left[D\left(\theta^{R+1},\theta\right)\right]\\
        &\le -\frac{t_{R}}{\tau_{R}}\Ep\left[D\left(\theta^{R+1},\theta\right)\right]+t_{0}V_{1}+\sum_{r=1}^{R}t_{r}C\tau_{r}\delta_{g}^{2}+t_{0}\left(F\left(\theta^{0},\lambda\right)-F\left(\theta,\lambda^{0}\right)\right)+\frac{t_{R}}{\tau_{R}}\Ep\left[D\left(\theta^{R+1},\theta\right)\right]\\
        &\le t_{0}V_{1}+\sum_{r=1}^{R}t_{r}C\tau_{r}\delta_{g}^{2}+t_{0}\left(F\left(\theta^{0},\lambda\right)-F\left(\theta,\lambda^{0}\right)\right)\\
        &\le C_{0} D\left(\theta,\theta^{0}\right)+C_{0}D\left(\lambda,\lambda^{0}\right)+t_{0}\left(F\left(\theta^{0},\lambda\right)-F\left(\theta,\lambda^{0}\right)\right)+\sum_{r=1}^{R}\left(t_{r}C\tau_{r}\delta_{g}^{2}\right)
    \end{align*}
    where $\Bar{\theta}^{R+1}$ and $\Bar{\lambda}^{R+1}$ are defined as
    \begin{align*}
        \Bar{\theta}^{R+1}=\frac{1}{\sum_{r=0}^{R}t_{r}}\sum_{r=0}^{R}t_{r}\theta^{r+1},\ \Bar{\lambda}^{R+1}=\frac{1}{\sum_{r=0}^{R}t_{r}}\sum_{r=0}^{R}t_{r}\lambda^{r+1}.
    \end{align*}
    Lemma \ref{lem:b3} yields $\sigma_{R}/\tau_{R}=\mathcal{O}(R^{2}),\ \sum_{r=0}^{R}t_{r}\tau_{r}=\mathcal{O}\left(R\right),\ t_{R}=\sigma_{R}/\sigma_{0},\ t_{r}\tau_{r}=\tau_{0}^{2}\gamma_{0}$.
    Hence we have that for any $\lambda^{\star}\in\LASET^{\star}$,
    \begin{align*}
        \Ep\left[D\left(\theta^{R+1},\theta^{\star}\right)\right]&\le \frac{\tau_{R}}{t_{R}}\left(C_{0}D\left(\theta^{\star},\theta^{0}\right)+C_{0}D\left(\lambda^{\star},\lambda^{0}\right)+t_{0}\left(F\left(\theta^{0},\lambda^{\star}\right)-F\left(\theta^{\star},\lambda^{0}\right)\right)+\sum_{r=0}^{R}\left(t_{r}C\tau_{r}\delta_{g}^{2}\right)\right)\\
        &\le \frac{C_{1}\delta_{g}^{2}}{R}+\frac{C_{2}}{R^{2}}\left(D\left(\theta^{\star},\theta^{0}\right)+D\left(\lambda^{\star},\lambda^{0}\right)+F\left(\theta^{0},\lambda^{\star}\right)-F\left(\theta^{\star},\lambda^{0}\right)\right).
    \end{align*}
    Here, we used the fact that $F(\Bar{\theta}^{R+1},\lambda^{\star})-F(\theta^{\star},\Bar{\lambda}^{R+1})\ge 0$.
    Hence we obtain the conclusion.
\end{proof}

\section{On the Almost Sure Boundedness of \textsc{Scaff-PD-IA}}\label{appendix:asbounded}
We discuss a concise sufficient condition for Assumption \ref{assump:noprojection} of \textsc{Scaff-PD-IA} (and obviously \textsc{Scaff-PD} of \citealp{yu2023scaff}) to justify Lemma \ref{lem:a3}.
In this section, we replace all updates $\theta^{r+1}= \argmin_{\theta\in\Theta}\{\langle \sum_{i=1}^{N}\lambda_{i}^{r+1}\Delta u_{i}^{r},\theta\rangle +(2\tau_{r})^{-1}\|\theta-\theta^{r}\|_{2}^{2}\}$ in Algorithm \ref{alg:scaffpd:reg} with projection-free updates
\begin{equation*}
    \theta^{r+1}=\theta^{r}-\tau_{r}\sum_{i=1}^{N}\lambda_{i}^{r+1}\Delta u_{i}^{r}.
\end{equation*}
If for some $\theta_{0}\in\R^{d}$ and $C>0$, for all $r\in\N$ and $\theta^{0}\in\R^{d}$, $\|\theta^{r}-\theta_{0}\|_{2}\le C+\|\theta^{0}-\theta_{0}\|_{2}$ almost surely, then Assumption \ref{assump:noprojection} gets satisfied for sufficiently large $\Theta$ and appropriate $\theta^{0}$.

Set the following assumptions.
\begin{assumption}[almost sure dissipativity]\label{assump:asbounded}
    For some $\theta_{0}\in\R^{d}$, the following hold:
    \begin{enumerate} 
        \item[(i)] For some $m,b>0$, for all $\theta\in\R^{d}$, $\lambda\in\LASET$, and $\xi\in\mathbb{X}$, 
\begin{align}
    \left\langle \sum_{i=1}^{n}\lambda_{i}\nabla_{\theta} f_{i}(\theta,\xi),\theta-\theta_{0}\right\rangle&\ge m\|\theta-\theta_{0}\|_{2}^{2}-b.
\end{align}
\item[(ii)]For some $B_{g},M_{g}>0$, for all $i\in[n]$, $\theta,\theta'\in\R^{d}$, and $\xi,\xi'\in\mathbb{X}$,
    \begin{align}
        \left\|\nabla_{\theta} f_{i}(\theta,\xi)-\nabla_{\theta} f_{i}(\theta,\xi')\right\|_{2}&\le B_{g}\left(1+\left\|\theta-\theta_{0}\right\|_{2}\right),\\
        \left\|\nabla_{\theta} f_{i}(\theta,\xi)-\nabla_{\theta} f_{i}(\theta',\xi)\right\|_{2}&\le M_{g}\left\|\theta-\theta'\right\|_{2}.
    \end{align}
    \end{enumerate}
\end{assumption}

We obtain the following result on the almost sure boundedness of \textsc{Scaff-PD-IA}.
\begin{proposition}\label{prop:asbounded}
    Under Assumptions \ref{assump:hyperparameters} and \ref{assump:asbounded}, there exists some constant $C,\beta\ge1$ such that if $\tau_{0}\le 1$, $\tau_{r}\le \min\{m/C,1/m\}$, $\tau_{r}=J\eta_{r}\beta$, and $4\tau_{r}^{2}M_{g}^{2}\le \beta^{2}$, then $\|\theta^{r}-\theta_{0}\|_{2}^{2}\le C+\|\theta^{0}-\theta_{0}\|_{2}^{2}$ for all $r\in\N$ almost surely.
\end{proposition}

\begin{proof}
We first decompose $\left\|\theta^{r+1}-\theta_{0}\right\|_{2}^{2}$ as
\begin{align*}
    \left\|\theta^{r+1}-\theta_{0}\right\|_{2}^{2}&=\left\|\theta^{r}-\theta_{0}-\tau_{r}\left(\frac{1}{J}\sum_{i=1}^{n}\sum_{j=1}^{J}\lambda_{i}^{r+1}g_{i,j}(u_{i,j-1}^{r})\right)\right\|_{2}^{2}\\
    &=\left\|\theta^{r}-\theta_{0}\right\|_{2}^{2}\underbrace{-2\tau_{r}\left\langle\frac{1}{J}\sum_{i=1}^{n}\sum_{j=1}^{J}\lambda_{i}^{r+1}g_{i,j}(u_{i,j-1}^{r}),\theta^{r}-\theta_{0}\right\rangle}_{=T_{1}}\\
    &\quad+\tau_{r}^{2}\left\|\frac{1}{J}\sum_{i=1}^{n}\sum_{j=1}^{J}\lambda_{i}^{r+1}g_{i,j}(u_{i,j-1}^{r})\right\|_{2}^{2}\\
    &\le \left\|\theta^{r}-\theta_{0}\right\|_{2}^{2}+T_{1}+3\tau_{r}^{2}\underbrace{\left\|\frac{1}{J}\sum_{i=1}^{n}\sum_{j=1}^{J}\lambda_{i}^{r+1}(g_{i,j}(u_{i,j-1}^{r})-g_{i,j}(\theta^{r}))\right\|_{2}^{2}}_{=T_{2}}\\
    &\quad+3\tau_{r}^{2}\left\|\frac{1}{J}\sum_{i=1}^{n}\sum_{j=1}^{J}\lambda_{i}^{r+1}(g_{i,j}(\theta^{r})-g_{i,j}(\theta_{0}))\right\|_{2}^{2}+3\tau_{r}^{2}\left\|\frac{1}{J}\sum_{i=1}^{n}\sum_{j=1}^{J}\lambda_{i}^{r+1}g_{i,j}(\theta_{0})\right\|_{2}^{2}\\
    &\le \left\|\theta^{r}-\theta_{0}\right\|_{2}^{2}+T_{1}+3\tau_{r}^{2}T_{2}+3\tau_{r}^{2}\left\|\lambda^{r+1}\right\|_{1}\left\|\theta^{r}-\theta_{0}\right\|_{2}^{2}\\
    &\quad+3\tau_{r}^{2}\underbrace{\sup_{\{\lambda\in\LASET,\xi_{i,j}\in\mathbb{X};i\in[n],j\in[J]\}}\left\|\frac{1}{J}\sum_{i=1}^{n}\sum_{j=1}^{J}\lambda_{i}\nabla f_{i}(\theta_{0},\xi_{i,j})\right\|_{2}^{2}}_{=C_{1}}
\end{align*}
Assumptions \ref{assump:asbounded} yields
\begin{align*}
    T_{1}
    &= -\frac{2\tau_{r}}{J}\left\langle\sum_{i=1}^{n}\sum_{j=1}^{J}\lambda_{i}^{r+1}g_{i,j}(\theta^{r}),\theta^{r}-\theta_{0}\right\rangle+\frac{2\tau_{r}}{J}\left\langle\sum_{i=1}^{n}\sum_{j=1}^{J}\lambda_{i}^{r+1}(g_{i,j}(\theta^{r})-g_{i,j}(u_{i,j-1}^{r})),\theta^{r}-\theta_{0}\right\rangle\\
    &\le -2\tau_{r}m\left\|\theta^{r}-\theta_{0}\right\|_{2}^{2}+2\tau_{r}b+\tau_{r}^{2}\left\|\theta^{r}-\theta_{0}\right\|_{2}^{2}+\frac{1}{J}\left\|\sum_{j=1}^{J}\sum_{i=1}^{n}\lambda_{i}^{r+1}\left(g_{i,j}^{r}(\theta^{r})-g_{i,j}^{r}(u_{i,j-1}^{r})\right)\right\|_{2}^{2}\\
    &= -2\tau_{r}m\left\|\theta^{r}-\theta_{0}\right\|_{2}^{2}+2\tau_{r}b+\tau_{r}^{2}\left\|\theta^{r}-\theta_{0}\right\|_{2}^{2}+T_{2}.
\end{align*}
Regarding $T_{2}$, we obtain
\begin{align*}
    T_{2}&\le \frac{\left\|\lambda^{r+1}\right\|_{1}}{J}\sum_{j=1}^{J}\sum_{i=1}^{n}\left|\lambda_{i}^{r+1}\right|^{2}\left\|g_{i,j}^{r}(\theta^{r})-g_{i,j}^{r}(u_{i,j-1}^{r})\right\|_{2}^{2}\\
    &\le \frac{\left\|\lambda^{r+1}\right\|_{1}M_{g}^{2}}{J}\sum_{j=1}^{J}\sum_{i=1}^{n}\left|\lambda_{i}^{r+1}\right|^{2}\left\|u_{i,j-1}^{r}-\theta^{r}\right\|_{2}^{2}.
\end{align*}
Similarly to Lemma \ref{lem:a1} above or Lemma A.1 of \citet{yu2023scaff}, it holds that
\begin{align*}
    \left\|u_{i,j}^{r}-\theta^{r}\right\|_{2}^{2}&\le \left(1+\frac{1}{J-1}\right)\left\|u_{i,j-1}^{r}-\theta^{r}\right\|_{2}^{2}+2J\eta_{r}^{2}\left\|g_{i,j}^{r}(u_{i,j-1}^{r})-g_{i,0}^{r}(\theta^{r})\right\|_{2}^{2}+2J\eta_{r}^{2}\left\|\sum_{i=1}^{n}\lambda_{i}^{r+1}g_{i,0}^{r}(\theta^{r})\right\|_{2}^{2}\\
    &\le\left(1+\frac{1}{J-1}\right)\left\|u_{i,j-1}^{r}-\theta^{r}\right\|_{2}^{2}+4J\eta_{r}^{2}\left\|g_{i,j}^{r}(u_{i,j-1}^{r})-g_{i,j}^{r}(\theta^{r})\right\|_{2}^{2}\\
    &\quad+4J\eta_{r}^{2}\left\|g_{i,j}^{r}(\theta^{r})-g_{i,0}^{r}(\theta^{r})\right\|_{2}^{2}+2J\eta_{r}^{2}\left\|\sum_{i=1}^{n}\lambda_{i}^{r+1}g_{i,0}^{r}(\theta^{r})\right\|_{2}^{2}\\
    &\le \left(1+\frac{1}{J-1}+\frac{4\tau_{r}^{2}M_{g}^{2}}{\beta^{2}J}\right)\left\|u_{i,j-1}^{r}-\theta^{r}\right\|_{2}^{2}+\frac{8\tau_{r}^{2}B_{g}^{2}}{\beta^{2}J}\left(1+\left\|\theta^{r}-\theta_{0}\right\|_{2}^{2}\right)\\
    &\quad+2J\eta_{r}^{2}\left\|\sum_{i=1}^{n}\lambda_{i}^{r+1}g_{i,0}^{r}(\theta^{r})\right\|_{2}^{2}\\
    &\le \left(1+\frac{2}{J-1}\right)\left\|u_{i,j-1}^{r}-\theta^{r}\right\|_{2}^{2}+\frac{8\tau_{r}^{2}B_{g}^{2}}{\beta^{2}J}\left(1+\left\|\theta^{r}-\theta_{0}\right\|_{2}^{2}\right)+2J\eta_{r}^{2}\left\|\sum_{i=1}^{n}\lambda_{i}^{r+1}g_{i,0}^{r}(\theta^{r})\right\|_{2}^{2}
\end{align*}
and
\begin{align*}
    \left\|u_{i,j}^{r}-\theta^{r}\right\|_{2}^{2}&\le\sum_{k=1}^{j-1}\left(1+\frac{2}{J-1}\right)^{k}\left(\frac{8\tau_{r}^{2}B_{g}^{2}}{\beta^{2}J}\left(1+\left\|\theta^{r}-\theta_{0}\right\|_{2}^{2}\right)+2J\eta_{r}^{2}\left\|\sum_{i=1}^{n}\lambda_{i}^{r+1}g_{i,0}^{r}(\theta^{r})\right\|_{2}^{2}\right)\\
    &\le \frac{48\tau_{r}^{2}B_{g}^{2}}{\beta^{2}}\left\|\theta^{r}-\theta_{0}\right\|_{2}^{2}+\frac{12\tau_{r}^{2}}{\beta^{2}}\left\|\sum_{i=1}^{n}\lambda_{i}^{r+1}g_{i,0}^{r}(\theta^{r})\right\|_{2}^{2}.
\end{align*}
Therefore,
\begin{align*}
    T_{2}
    &\le \left\|\lambda^{r+1}\right\|_{1}\left\|\lambda^{r+1}\right\|_{2}^{2}M_{g}^{2} \tau_{r}^{2}\left(\frac{48B_{g}^{2}}{\beta^{2}}\left(1+\left\|\theta^{r}-\theta_{0}\right\|_{2}^{2}\right)+\frac{12}{\beta^{2}}\left\|\sum_{i=1}^{n}\lambda_{i}^{r+1}g_{i,0}^{r}(\theta^{r})\right\|_{2}^{2}\right)\\
    &\le \left\|\lambda^{r+1}\right\|_{1}\left\|\lambda^{r+1}\right\|_{2}^{2}M_{g}^{2} \tau_{r}^{2}\left(\frac{24}{\beta^{2}}\left(2B_{g}^{2}+\left\|\lambda^{r+1}\right\|_{1}M_{g}^{2}\right)\left\|\theta^{r}-\theta_{0}\right\|_{2}^{2}+\frac{48B_{g}^{2}}{\beta^{2}}+\frac{24}{\beta^{2}}C_{1}^{2}\right)\\
    &\le C_{2}\tau_{r}^{2}\left\|\theta^{r}-\theta_{0}\right\|_{2}^{2}+\tau_{r}^{2}C_{3},
\end{align*}
where $C_{2}:=24\|\LASET\|_{1}\|\LASET\|_{2}^{2}(2B_{g}^{2}+\|\LASET\|_{1}M_{g}^{2})M_{g}^{2}\beta^{-2}$ and $C_{3}=24\|\lambda^{r+1}\|_{1}\|\lambda^{r+1}\|_{2}^{2}M_{g}^{2}\beta^{-2}(2B_{g}^{2}+C_{1}^{2})$.
Then we obtain
\begin{align*}
    \left\|\theta^{r+1}-\theta_{0}\right\|_{2}^{2}
    &\le \left(1-2\tau_{r}m+\tau_{r}^{2}+3\tau_{r}^{2}\left\|\lambda^{r+1}\right\|_{1}\right)\left\|\theta^{r}-\theta_{0}\right\|_{2}^{2}+\left(1+3\tau_{r}^{2}\right)T_{2}+2\tau_{r}b+3\tau_{r}^{2}C_{1}\\
    &\le \left(1-2\tau_{r}m+\tau_{r}^{2}\left(1+3\tau_{r}^{2}\left\|\lambda^{r+1}\right\|_{1}+C_{2}\left(1+3\tau_{r}^{2}\right)\right)\right)\left\|\theta^{r}-\theta_{0}\right\|_{2}^{2}\\
    &\quad+\tau_{r}\left(\left(1+3\tau_{r}^{2}+2b\right)C_{3}+3\tau_{r}C_{1}\right)\\
    &\le  \left(1-2\tau_{r}m+C_{4}\tau_{r}^{2}\right)\left\|\theta^{r}-\theta_{0}\right\|_{2}^{2}+\tau_{r}C_{5},
\end{align*}
where $C_{4}=\max\{m,1+3\tau_{0}^{2}\|\LASET\|_{1}+C_{2}(1+3\tau_{0}^{2})\}$ and $C_{5}=(1+3\tau_{0}^{2}+2b)C_{3}+3\tau_{0}C_{1}$.
If $\tau_{r}\le \min\{m/C_{4},1/m\}$, we have
\begin{align}
    \left\|\theta^{r+1}-\theta_{0}\right\|_{2}^{2}&\le \left\|\theta^{0}-\theta_{0}\right\|_{2}^{2}+C_{5}\left(\tau_{r}+\sum_{i=0}^{r-1}\tau_{i}\prod_{j=i+1}^{r}\left(1-\tau_{j}m\right)\right)\notag\\
    &\le \left\|\theta^{0}-\theta_{0}\right\|_{2}^{2}+C_{5}\left(\tau_{0}+\frac{1}{m}\sum_{i=0}^{r-1}m\tau_{i}\prod_{j=i+1}^{r}\left(1-\tau_{j}m\right)\right)\label{eq:bounded:preconclusion}.
\end{align}

We finally show that for $\varpi_{i}=m\tau_{i}$, under Assumption \ref{assump:hyperparameters},
\begin{equation}\label{eq:bounded:induction}
    \sum_{i=0}^{r-1}\varpi_{i}\prod_{j=i+1}^{r}\left(1-\varpi_{j}\right)\le \sqrt{1+\frac{m_{f}\tau_{0}}{4}}+1.
\end{equation}
Let us set $S_{r}:=\sum_{i=0}^{r-1}\varpi_{i}\prod_{j=i+1}^{r}(1-\varpi_{j})$.
$S_{1}=m\tau_{0}(1-\tau_{1}m)\le m\tau_{0}(1-\tau_{0}m)$ and thus it obviously holds.
Assume that Eq.~\eqref{eq:bounded:induction} holds for $r=r'\in\N$.
If $S_{r'}\in [\sqrt{1+m_{f}\tau_{0}/4},\sqrt{1+m_{f}\tau_{0}/4}+1]$, then
\begin{align*}
    S_{r'+1}&=(1-\varpi_{r'+1})(S_{r'}+\varpi_{r'})=S_{r'}(1-\varpi_{r'+1})\left(1+\frac{\varpi_{r'}}{S_{r'}}\right)\le S_{r'}(1-\varpi_{r'+1})\left(1+\frac{\varpi_{r'}}{\sqrt{1+m_{f}\tau_{0}/4}}\right)\\
    &\le S_{r'}(1-\varpi_{r'+1})\left(1+\frac{\varpi_{r'+1}\sqrt{\gamma_{r'+1}/\gamma_{r'}}}{\sqrt{1+m_{f}\tau_{0}/4}}\right)\le S_{r'}(1-\varpi_{r'+1})(1+\varpi_{r'+1})\le S_{r'}.
\end{align*}
Otherwise, 
\begin{align*}
    S_{r+1}=(1-\varpi_{r+1})(S_{r}+\varpi_{r})\le S_{r}+1<\sqrt{1+m_{f}\tau_{0}/4}+1.
\end{align*}
Hence Eq.~\eqref{eq:bounded:induction} holds true.

Eqs.~\eqref{eq:bounded:preconclusion} and \eqref{eq:bounded:induction} yield the desired result.
\end{proof}

\section{Proofs for Section \ref{sec:generalization}}

We present a more formal result of the generalization error bound in Section \ref{sec:generalization} and its proof.
To this aim, we need to define an additional notion.
We define the weighted Rademacher complexity $\mathfrak{R}_{Nn}(\mathcal{G},\lambda)$ with fixed $\lambda\in\LASET$ and the minimax Rademacher complexity $\mathfrak{R}_{Nn}(\mathcal{G},\LASET)$ for this $\mathcal{G}$ as follows:
\begin{equation*}
    \mathfrak{R}_{Nn}(\mathcal{G},\lambda)=\Ep\left[\sup_{h\in\mathcal{H}}\sum_{i=1}^{n}\frac{\lambda_{i}}{N}\sum_{j=1}^{N}\sigma_{i,j}\ell(h(x_{i,j}),y_{i,j})\right],\quad\mathfrak{R}_{Nn}(\mathcal{G},\LASET)=\sup_{\lambda\in\LASET}\mathfrak{R}_{Nn}(\mathcal{G},\lambda),
\end{equation*}
where $\{\sigma_{i,j};i\in[n],j\in[N]\}$ is an array of i.i.d.~Rademacher variables independent of $\{(x_{i,j},y_{i,j})\}$.
Using the above definition, we derive an upper bound for the generalization error. This upper bound is provided for a general setting.
\begin{theorem}[formal version of Theorem \ref{thm:generalization}]\label{thm:generalization_appendix}
    The following holds:
    \begin{enumerate}
      \setlength{\parskip}{0cm}
  \setlength{\itemsep}{0cm}
        \item[(i)] For fixed $\epsilon>0$ and $N\in\N$, for any $\delta>0$, with probability at least $1-\delta$, for all $h\in\mathcal{H}$ and $\lambda\in\LASET$,
    \begin{equation*}
        L_{\lambda}(h)\le L_{\lambda,Nn}(h)+2\mathfrak{R}_{Nn}(\mathcal{G},\lambda)+M\epsilon+M\sqrt{\frac{\mathfrak{s}(\lambda)}{2Nn}\log\frac{\mathcal{N}(\LASET,\|\cdot\|_{1},\epsilon)}{\delta}}.
    \end{equation*}
    \item[(ii)] If $\ell$ is $\{-1,+1\}$-valued and $\mathcal{G}$ admits its VC-dimension $\mathrm{VC}(\mathcal{G})<Nn$, then it holds that
    \begin{equation*}
        \mathfrak{R}_{Nn}(\mathcal{G},\LASET)\le \sqrt{2\mathfrak{s}(\LASET)\frac{\mathrm{VC}(\mathcal{G})}{Nn}\left(1+\log\left(\frac{Nn}{\mathrm{VC}(\mathcal{G})}\right)\right)}.
    \end{equation*}
    \end{enumerate}
\end{theorem}
From this theorem, we immediately obtain Theorem \ref{thm:generalization}.
Moreover, we note that $\mathcal{N}(\LASET,\|\cdot\|_{1},\epsilon)\le \mathcal{N}(\Delta^{n-1},\|\cdot\|_{1},\frac{2(n-1)\epsilon}{n\max_{\lambda\in\LASET}\|\lambda-n^{-1}\mathbf{1}\|_{1}})$
since $\max_{a\in\Delta^{n-1}}\|a-u\|_{1}=(n-1)/n+(n-1)/n=2(n-1)/n$.

\begin{proof}[Proof of Theorem \ref{thm:generalization_appendix}]
    We only give the proof of the second statement since the proof of the first statement is the same as that of Theorem 1 of \citet{mohri2019agnostic}.
    Note that we show a proof similar to that of Lemma 1 of \citet{mohri2019agnostic}.

    We define a family of sets $A_{\lambda}\subset\R^{Nn}$ for each $\lambda\in\LASET$ as
    \begin{equation*}
        A_{\lambda}:=\left\{\left[\frac{\lambda_{i}}{N}\ell(h(x_{i,j}),y_{i,j})\right]_{i\in [n],j\in[N]}:x_{i,j}\in\mathbb{X},y_{i,j}\in\mathbb{Y}\right\}.
    \end{equation*}
    For any $\lambda\in\LASET$ and $a\in A_{\lambda}$, $\|a\|_{2}=\sqrt{n\lambda_{i}^{2}/N}\le \sqrt{\mathfrak{s}(\LASET)/N}$.
    Massart's lemma yields that for any $\lambda\in\LASET$,
    \begin{equation*}
        \mathfrak{R}_{Nn}(\mathcal{G},\lambda)\le \Ep_{\{\sigma_{i,j}\}}\left[\sup_{a\in A}\sum_{i=1}^{n}\sum_{j=1}^{N}\sigma_{i,j}a_{i,j}\right]\le \sqrt{\frac{2\mathfrak{s}(\LASET)\log|A_{\lambda}|}{Nn}}.
    \end{equation*}
    Sauer's lemma gives that for $Nn>d$, $|A_{\lambda}|\le (eNn/d)^{d}$.
    Therefore, we yield the conclusion.
\end{proof}

\begin{proof}[Proof of Proposition \ref{prop:skewness}]
    By convexity, we have
    \begin{align*}
        \mathfrak{s}(\LASET)&\le  n\sum_{i=1}^{\lfloor n\alpha\rfloor+1}\left(\left|\frac{1}{(1-\varphi)\alpha n}-\frac{1}{n}\right|^{2}+\left|-\frac{\varphi}{(1-\varphi)\alpha n}-\frac{1}{n}\right|^{2}\right)+1\\
        &=n(\lfloor n\alpha\rfloor+1)\frac{(1/((1-\varphi)\alpha)-1)^{2}+(\varphi/((1-\varphi)\alpha)+1)^{2}}{n^{2}}+1\\
        &=\alpha\left[\left(\frac{1}{(1-\varphi)\alpha}-1\right)^{2}+\left(\frac{\varphi}{(1-\varphi)\alpha}+1\right)^{2}\right]+1+\mathcal{O}(1/n),
    \end{align*}
    and 
    \begin{align*}
        \mathfrak{s}(\LASET)&\ge  n\sum_{i=1}^{\lfloor n\alpha\rfloor}\left(\left|\frac{1}{(1-\varphi)\alpha n}-\frac{1}{n}\right|^{2}+\left|-\frac{\varphi}{(1-\varphi)\alpha n}-\frac{1}{n}\right|^{2}\right)+1\\
        &=n\lfloor n\alpha\rfloor\frac{(1/((1-\varphi)\alpha)-1)^{2}+(\varphi/((1-\varphi)\alpha)+1)^{2}}{n^{2}}+1\\
        &=\alpha\left[\left(\frac{1}{(1-\varphi)\alpha}-1\right)^{2}+\left(\frac{\varphi}{(1-\varphi)\alpha}+1\right)^{2}\right]+1+\mathcal{O}(1/n).
    \end{align*}
    This is the desired result.
\end{proof}

\section{Implementation Details}\label{appendix:experiments}

\subsection{Details of Section \ref{sec:mono:sim}}\label{sec:mono:sim:detail}
We give a detailed explanation of the experiment in Section \ref{sec:mono:sim}.
We consider the \texttt{penguins} data: we let $x_{j}$ and $y_{j}$ with $j\in [30]$ denote the vector (\texttt{bill\_depth\_mm} and \texttt{flipper\_length\_mm}) and \texttt{bill\_length\_mm} data in the $j$-th sample, and choose first 10 samples from each species.
Hence, $(x_{j},y_{j})$ with $j=1,\ldots,10$, those with $j=11,\ldots,20$, and those with $j=21,\ldots,30$ are data of \texttt{Adelie}, \texttt{Chinstrap}, and \texttt{Gentoo} respectively. 
We analyze a simple regression model $\texttt{(bill\_length\_mm)}=\theta_{1}+\theta_{2}\texttt{(bill\_depth\_mm)}+\theta_{3}\texttt{(flipper\_length\_mm)}+\texttt{(error)}$ with $\theta\in\{-35,-34.99,-34,98,\ldots,-25\}\times\{0,0.001,\ldots,1\}^{2}$.
We evaluate three quadratic loss functions of $\theta$ for datasets separated according to the three classes:
$f_{i}(\theta)=\sum_{j=10(i-1)+1}^{10i}(y_{j}-\theta_{1}-\theta_{2}(x_{j})_{(1)}+\theta_{3}(x_{j})_{2})^{2}$ with $i=1,2,3$ ($f_{1}$, $f_{2}$, and $f_{3}$ represent the losses of \texttt{Adelie}, \texttt{Chinstrap}, and \texttt{Gentoo}).
Setting $\Delta^{3}=A=B$, we exactly obtain the solutions of \eqref{eq:problem:scalarized} for all $\varphi\in\{0,0.01,\ldots,0.05\}$ and Figure \ref{fig:penguin:fairness} represents the unfairness measure \eqref{eq:fairness} of those solutions.

\subsection{Details of Section \ref{sec:experiments}}
We give details in our experiments in Section \ref{sec:experiments}.

We conduct the classification experiment of the MNIST dataset and that of the CIFAR10 dataset
2-layer fully-connected neural nets with hidden 50 units, ReLU activation, and 50\% dropout.
In both the tasks, the number of clients $n$ is fixed as $100$.
We generate the heterogeneity of the datasets of clients in the manner same as \citet{yurochkin2019bayesian} (see also \citealp{li2022federated}, \citealp{yu2022tct}, and \citealp{yu2023scaff}): we first generate 10 random vectors from the $n$-dimensional Dirichlet distribution with parameter $(0.5,\ldots,0.5)$, where 10 is the number of classes to be predicted; we then distribute datasets to clients with the proportion according to the generated vectors.
We set the numbers of global epochs $R$ of all the algorithms to be $100$ and the numbers of local epochs $J$ of all the algorithms except for \textsc{SAFL} to be $5$.
We let the ambiguity sets of \textsc{SAFL}, \textsc{DRFA} and \textsc{Scaff-PD} be $\Delta_{0.20}^{n-1}$.
Moreover, we set the ambiguity sets of \textsc{Scaff-PD-IA} as $A=B=\Delta_{0.20}^{n-1}$ and $\varphi=0.20$.

For the primal updates of \textsc{FedAvg}, \textsc{SAFL}, and \textsc{DRFA}, the (local) learning rate is fixed as $10^{-1}$.
The learning rate of the primal updates in \textsc{Scaffold} (i.e., \textsc{Scaff-PD-IA} with $\Lambda_{\varphi}=(1/n,\ldots,1/n)$), \textsc{Scaff-PD}, and \textsc{Scaff-PD-IA}, we set the local learning rate ($\eta_{r}$) and global learning rate ($\tau_{r}$) as $10^{-2}$ and $2.5\times 10^{-2}$ respectively.
The learning rate for the dual updates of the algorithms ($\sigma_{r}$) is set to be $10^{-3}$.

\bibliographystyle{apecon}
\bibliography{main}

\end{document}